\documentclass{article}
\usepackage[utf8]{inputenc}
\usepackage{amsmath,amssymb, amsthm}
\usepackage{algorithm,algorithmic}
\usepackage{makecell}
\usepackage{hyperref}
\usepackage{graphicx}
\usepackage{lineno}
\usepackage[inline]{enumitem}
\usepackage[dvipsnames]{xcolor}
\usepackage{comment}
 
\usepackage{cleveref}
\usepackage[backend=biber,style=numeric-comp,maxcitenames=2,maxbibnames=99,natbib=true]{biblatex}
\usepackage{longtable}

\usepackage{caption}

\newtheorem{assumption}{Assumption}
\newtheorem{theorem}{Theorem}
\newtheorem{definition}{Definition}

\usepackage{xparse} 
\NewDocumentCommand{\LT}{O{t_1} O{x_1}}{%
  \ensuremath{\Psi_\theta^{\scriptstyle (#1, #2)}}%
}
\NewDocumentCommand{\LTN}{O{t_1} O{x_1}}{%
  \ensuremath{\Psi_{\theta,N}^{\scriptstyle (#1, #2)}}%
}
\NewDocumentCommand{\tildeLTN}{O{t_1} O{x_1}}{%
  \ensuremath{\widetilde\Psi_{\theta,N}^{\scriptstyle (#1, #2)}}%
}

\usepackage[normalem]{ulem}
\usepackage{epigraph}
\usepackage{booktabs}

\newcommand{\R}{\mathbb{R}}

\newcommand{\e}{{\text{e}}}

\definecolor{matlab1}{RGB}{0,    114,  189} 
\definecolor{matlab2}{RGB}{217,   83,   25}
\definecolor{matlab3}{RGB}{237,  177,   32}
\definecolor{matlab4}{RGB}{126,   47,  142}
\definecolor{matlab5}{RGB}{119,  172,   48}
\definecolor{matlab6}{RGB}{77,   190,  238}
\definecolor{matlab7}{RGB}{162,   20,   47}

\hypersetup{
  colorlinks=true,breaklinks,
  linktoc=section,
  linkcolor=matlab1,
  linkbordercolor=matlab2,
  citecolor=matlab1,
  urlcolor=matlab1,
  pdfborder = {0 0 1}
}

\definecolor{python1}{HTML}{1f77b4}
\definecolor{python2}{HTML}{ff7f0e}
\definecolor{python3}{HTML}{2ca02c}
\definecolor{python4}{HTML}{d62728}
\definecolor{python5}{HTML}{9467bd}
\definecolor{python6}{HTML}{8c564b}
\definecolor{python7}{HTML}{e377c2}
\definecolor{python8}{HTML}{7f7f7f}
\definecolor{python9}{HTML}{bcbd22}
\definecolor{python10}{HTML}{17becf}
\definecolor{designcolor}{HTML}{1f77b4}

\usepackage[textsize=small]{todonotes}
\usepackage{tikz}
\usetikzlibrary{positioning, calc}

\usepackage{mathrsfs}

\DeclareMathOperator*{\argmin}{\arg\min}

\usepackage{fullpage}

\usepackage[group-separator={,}, group-minimum-digits=4]{siunitx}
\newcommand{\keywords}[1]{\vspace{2mm}\noindent\textbf{Keywords: } #1}

\title{Latent Twins}

\author{
  Matthias Chung\thanks{Department of Mathematics, Emory University; \texttt{\{matthias.chung, max.collins2\}@emory.edu}}%
  \and Deepanshu Verma\thanks{School of Mathematical and Statistical Sciences, Clemson University; \texttt{dverma@clemson.edu}}%
  \and Max Collins\footnotemark[1]
  \and Amit N. Subrahmanya \thanks{Argonne National Laboratory \texttt{\{vsastry, vhebbur, famitnagesh\}@anl.gov}}  %
  \and Varuni Katti Sastry \footnotemark[3] %
  \and Vishwas Rao  \footnotemark[3]
}
\date{}
\begin{document}

\maketitle

\epigraph{``Nature uses only the longest threads to weave her patterns, so each small piece of her fabric reveals the organization of the entire tapestry.''}{\textit{Richard P. Feynman}}

\begin{abstract}
Over the past decade, scientific machine learning has transformed the development of mathematical and computational frameworks for analyzing, modeling, and predicting complex systems. From inverse problems to numerical PDEs, dynamical systems, and model reduction, these advances have pushed the boundaries of what can be simulated. Yet they have often progressed in parallel, with representation learning and algorithmic solution methods evolving largely as separate pipelines. With \emph{Latent Twins}, we propose a unifying mathematical framework that creates a hidden surrogate in latent space for the underlying equations. Whereas digital twins mirror physical systems in the digital world, Latent Twins mirror mathematical systems in a learned latent space governed by operators. Through this lens, classical modeling, inversion, model reduction, and operator approximation all emerge as special cases of a single principle. We establish the fundamental approximation properties of Latent Twins for both ODEs and PDEs and demonstrate the framework across three representative settings: (i) canonical ODEs, capturing diverse dynamical regimes; (ii) a PDE benchmark using the shallow-water equations, contrasting Latent Twin simulations with DeepONet and forecasts with a 4D-Var baseline; and (iii) a challenging real-data geopotential reanalysis dataset, reconstructing and forecasting from sparse, noisy observations. Latent Twins provide a compact, interpretable surrogate for solution operators that evaluate across arbitrary time gaps in a single-shot, while remaining compatible with scientific pipelines such as assimilation, control, and uncertainty quantification. Looking forward, this framework offers scalable, theory-grounded surrogates that bridge data-driven representation learning and classical scientific modeling across disciplines.
\end{abstract}


\section{Introduction}

Neural networks and generative AI are transforming a broad spectrum of domains, including science and engineering research, by introducing data-driven methodologies that redefine traditional approaches including model-based simulation, inference, prediction and forecasting. This integration is opening new pathways for scientific approaches to complex systems, advancing computationally tractable and scientifically rigorous methodologies.

A key observation in scientific computing is that intrinsic low-dimensional structures are present in high-dimensional scientific data \cite{chung_good_2025}. For instance, acoustic and image data with their compact representation in the Fourier domain or within their image spectrum gave rise to compression techniques such as JPEG and MP3 \cite{wallace1991jpeg, brandenburg1999mp3,oppenheim1989dtsp}. In the same manner as data, mathematical models inherently have low-dimensional representation by construction. For instance, the Poisson equation $-\nabla \cdot (\mu \nabla u) = f$ with some given boundary conditions may result in complex solution, while its resulting dynamics are fully encoded by the single parameter $\mu$. In general, the characteristic of low-dimensionality is shared by both mathematical models and scientific data, and is often expressed through a few dominant modes or features, enabling significant dimensionality reduction and giving rise to more efficient data processing, analysis, and computation. Classical techniques like Principal Component Analysis (PCA) and Proper Orthogonal Decomposition (POD) have long exploited this property to simplify complex data or systems, particularly in dynamical systems \cite{brunton2022data}. However, these methods are inherently linear and may struggle to capture the nonlinear intricacies of modern scientific data. 

The emergence of autoencoders provides a powerful nonlinear alternative to these classical approaches. Autoencoder architectures aim to learn low-dimensional representations from data or from model simulations. By training a neural network that is forced to exhibit a small dimensional layer, typically referred to as the latent space, the network learns to compress the input into a lower-dimensional latent representation. This process is akin to identifying the underlying structure of the model or data, allowing for efficient storage and manipulation. Accordingly, autoencoders may be seen as data-driven variants of manifold learning in mathematics, uncovering the intrinsic low-dimensional geometry hidden in complex datasets.

This work introduces a general approach that leverages learnable compact representations of low-dimensional mathematical structures. By operating within dual hidden latent space, we can perform trainable manipulations while still recovering the behavior of the original system. We demonstrate the versatility of this framework beyond standard methods, extending its utility across scientific domains where data and models are intrinsically linked. In doing so, we provide a common foundation that connects data-driven findings with the interpretability and rigor of model-based science.

\begin{definition}[Latent twins]\label{def:latenttwin}
Let $(X,\|\cdot\|_X)$ and $(Y,\|\cdot\|_Y)$ be normed vector spaces, and let $F^{\to}\!:X\to Y$ and $F^{\gets}\!:Y\to X$ denote the (possibly unknown) forward and inverse operators of interest. We define the \emph{Latent Twins} of $F^{\to}$ and $F^{\gets}$ as parameterized surrogate operators
\[
f^{\to} := d_y \circ m^{\to} \circ e_x, 
\qquad \text{and} \qquad
f^{\gets} := d_x \circ m^{\gets} \circ e_y,
\]
where
\begin{itemize}
    \item $Z_x$ and $Z_y$ are latent spaces of dimensions $n_x \ll \dim(X)$ and $n_y \ll \dim(Y)$ respectively,
  \item $e_x\!:X\to Z_x$ and $e_y\!:Y\to Z_y$ are encoders into these
  latent spaces 
  ,
  \item $d_x\!:Z_x\to X$ and $d_y\!:Z_y\to Y$ are the corresponding decoders,
  \item $m^{\to}\!:Z_x\to Z_y$ and $m^{\gets}\!:Z_y\to Z_x$ are trainable latent mappings.
\end{itemize}
While the autoencoders are approximately identity mappings on $x$ and $y$, $
d_x\!\circ e_x \approx \operatorname{id}_{X}$ and $d_y\!\circ e_y \approx \operatorname{id}_{Y}$ the Latent Twins approximate $F^{\to}$ and $F^{\gets}$ on data manifolds in $X$ and $Y$.
\end{definition}

We emphasize that the encoders $e_x, e_y$, decoders $d_x, d_y$, and latent mappings $m^{\to}, m^{\gets}$ are all parameterized by neural network weights, collectively denoted by $\theta$. Consequently, the Latent Twin $f^{\to}$ and $f^{\gets}$ depend on $\theta$. For clarity of presentation, we suppress this dependence in the notation and only make it explicit when required for understanding. In contrast to many common data-driven surrogates, the specific autoencoder design of the Latent Twin ensures that the latent spaces serve only to represent the corresponding variables $x$ and $y$, while the latent mappings $m^{\to}$ and $m^{\gets}$ can be directly interpreted as latent-space counterparts of the original operators $F^{\to}$ and $F^{\gets}$.

An instructive analogy can be drawn with the concept of \emph{digital twins}—virtual representations of physical systems that integrate physics-based models and data-driven updates to enable predictive maintenance, optimization, and decision-making under uncertainty \cite{tao_digital_2018}. In a similar vein, we introduce the notion of Latent Twins (schematically illustrated in \Cref{fig:latentTwin}), where complex physical or mathematical processes are mirrored in an embedded latent space. By projecting high-dimensional models into such latent domains, operations can be performed more efficiently while retaining essential structure, thereby providing interpretable and predictive surrogates with significant computational and analytical advantages.

\begin{figure}
    \centering
      \begin{tikzpicture}

	\node[fill=python1!50, minimum width=2.5cm, minimum height=0.5cm] (x) at (0,0) {$x$};
    \draw[fill=python9!50,draw=none]
             ([yshift=-0.2cm,xshift=0.0cm]x.south west) -- ([yshift=-1.5cm,xshift=0.5cm]x.south west) -- ([yshift=-1.5cm,xshift=-0.5cm]x.south east) -- ([yshift=-0.2cm,xshift=0.0cm]x.south east) -- cycle node (e) at ([yshift=-0.75cm,xshift=0.0cm]x.south)
             {$\begin{matrix}\text{encoder}\\[-0.2ex] e_x\end{matrix}$};              
               
    \node[fill=python3!50, minimum width=1.5cm, minimum height=0.5cm] (zx) at (0,-2.3) {\footnotesize{latent $\tilde{z}_x$}};    
    \draw[fill=python4!50,draw=none]
             ([yshift=-0.2cm,xshift=0.0cm]zx.south west) -- ([yshift=-1.5cm,xshift=-0.5cm]zx.south west) -- ([yshift=-1.5cm,xshift=0.5cm]zx.south east) -- ([yshift=-0.2cm,xshift=0.0cm]zx.south east) -- cycle node (d) at ([yshift=-0.75cm,xshift=0.0cm]zx.south) 
             {$\begin{matrix}\text{decoder}\\[-0.2ex] d_x\end{matrix}$};               
    \node[fill=python1!50, minimum width=2.5cm, minimum height=0.5cm] (xt) at (0,-4.5) {$\tilde x$};      

    \node[fill=python2!50, minimum width=2.5cm, minimum height=0.5cm] (b) at (6,0) {$ y$};
    \draw[fill=python5!50,draw=none]
             ([yshift=-0.2cm,xshift=0.0cm]b.south west) -- ([yshift=-1.5cm,xshift=0.5cm]b.south west) -- ([yshift=-1.5cm,xshift=-0.5cm]b.south east) -- ([yshift=-0.2cm,xshift=0.0cm]b.south east) -- cycle node (e) at ([yshift=-0.75cm,xshift=0.0cm]b.south) 
             {$\begin{matrix}\text{encoder}\\[-0.2ex] e_y\end{matrix}$}; 
               
    \node[fill=python6!50, minimum width=1.5cm, minimum height=0.5cm] (zb) at (6,-2.3) {\footnotesize{latent $\tilde{z}_y$}};    
    \draw[fill=python9!50,draw=none]
             ([yshift=-0.2cm,xshift=0.0cm]zb.south west) -- ([yshift=-1.5cm,xshift=-0.5cm]zb.south west) -- ([yshift=-1.5cm,xshift=0.5cm]zb.south east) -- ([yshift=-0.2cm,xshift=0.0cm]zb.south east) -- cycle node (d) at ([yshift=-0.75cm,xshift=0.0cm]zb.south) 
             {$\begin{matrix}\text{decoder}\\[-0.2ex] d_y\end{matrix}$};  
    \node[fill=python2!50, minimum width=2.5cm, minimum height=0.5cm] (bt) at (6,-4.5) {$\tilde y$}; 

\begin{scope}[-latex,shorten >=9pt,shorten <=9pt,line width=5pt]

    \draw[matlab2!50!matlab3]  ([yshift=-0.2cm]zb.west) to ([yshift=-0.2cm]zx.east);
    \draw[matlab1!50!matlab2] ([yshift=0.2cm]zx.east) to ([yshift=0.2cm]zb.west);    
\end{scope}
\node[matlab1!50!matlab2] (m_forward) at (3.0,-1.7) {forward latent map $m^{\to}$};
\node[matlab2!50!matlab3] (m_inverse) at (3.0,-2.9) {inverse latent map $m^{\gets}$};

\begin{scope}[-latex,shorten >=9pt,shorten <=9pt,line width=5pt]
    \draw[matlab1!50!matlab2] ([yshift=-0.0cm]x.east) to ([yshift=-0.0cm]b.west);
    \draw[matlab2!50!matlab3] ([yshift=0.0cm]bt.west) to ([yshift=0.0cm]xt.east);
\end{scope}

\node[matlab1!50!matlab2] (f_forward) at (3.0,0.3) {forward $F^{\to}$};
\node[matlab2!50!matlab3] (f_inverse) at (3.5,-4.8) {inverse $F^{\gets}$};
\node[rotate=90,black] at (-2,-2.3) {$x$-autoencoder};
\node[rotate=270,black] at (7.7,-2.3) {$y$-autoencoder};
      
\end{tikzpicture}
    \caption{Schematic of the Latent Twin framework. Two autoencoders (for the $x$- and $y$-spaces) provide compact latent representations $\tilde z_x$ and $\tilde z_y$. Within this latent space, the forward mapping $F^{\to}$ and inverse mapping $F^{\gets}$ are approximated by surrogate models constructed as compositions of encoders, latent-to-latent mappings, and decoders. Concretely, the forward surrogate is defined as $f^{\to} = d_y \circ m^{\to} \circ e_x$, mapping inputs $x$ to outputs $\tilde y$ through the forward latent map $m^{\to}$, while the inverse surrogate is given by $f^{\gets} = d_x \circ m^{\gets} \circ e_y$, recovering inputs from outputs via the inverse latent map $m^{\gets}$. This structure highlights how Latent Twins emulate bidirectional operators in a low-dimensional space, combining the efficiency of latent representations with the expressiveness of trainable mappings.}
    \label{fig:latentTwin}
\end{figure}
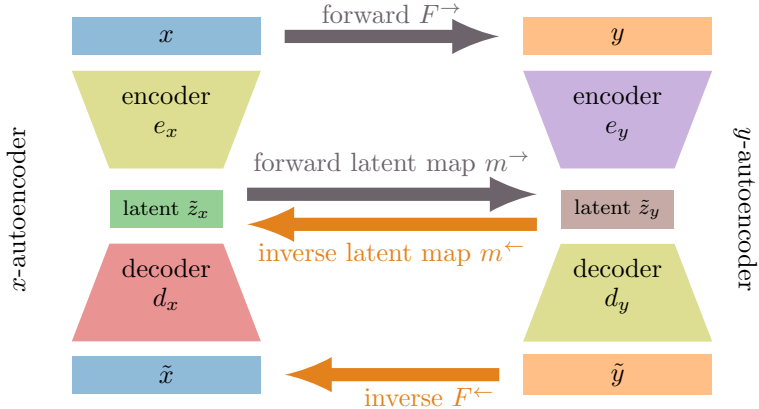

\paragraph{Our Contributions.} We introduce the Latent Twin framework, a unifying approach that bridges representation learning, dynamical systems, inverse problems, and operator learning. On the theoretical side, we establish approximation guarantees for Latent Twins in both ODE and PDE settings, carefully analyzing reconstruction and flow-map errors. This positions Latent Twins not as yet another specialized method, but as a general principle for learning solution operators in latent space. On the numerical side, we demonstrate the framework’s versatility across a range of settings: (i) ODE benchmarks (harmonic oscillator, SIR, Lotka--Volterra, Lorenz--63), where Latent Twins are compared directly against established sequence models such as LSTMs; (ii) PDE-based prediction tasks, where Latent Twins are evaluated on the shallow-water equations and contrasted with standard data assimilation methods such as 4D-Var; and (iii) real-world geopotential height data, where the framework shows promising predictive capabilities. Importantly, Latent Twins unify ideas that have emerged in distinct communities---reduced-order modeling, Koopman operators, neural operators, data assimilation, and paired autoencoders---into a single coherent formulation. While related methods exist in each of these domains, our overarching contribution lies in articulating and analyzing the Latent Twin principle as the common thread that connects them. To acknowledge prior work and avoid overstating novelty, we provide an extended review of the relevant literature, situating Latent Twins within and across these fields. Although our focus here is on dynamical systems, the framework naturally extends to operator inference and inverse problems; related work on paired autoencoders for inversion illustrates this broader scope, and further extensions to operator learning will be addressed in forthcoming work.

\paragraph{Structure of our Work.} \Cref{sec:background} reviews related ideas across model reduction, operator learning, Koopman operators, optimization, inverse problems, and data assimilation. \Cref{sec:latenttwins} formalizes Latent Twins for dynamical systems and develops approximation guarantees in \Cref{sub:theory} for ODEs (\Cref{sub:odes}) and PDEs (\Cref{sub:pdes}). \Cref{sec:numerics} reports numerical studies: (i) canonical ODEs, illustrating performance across different dynamical regimes; (ii) a PDE prediction study on the shallow water equations with a 4D-Var comparisons; and (iii) challenging forecasting geopotential height from  noisy observations. We conclude with an outlook on promising extensions of the Latent Twin approach in \Cref{sec:conclusion}. Additional material is provided in \Cref{app:swe-details}, while \Cref{sec:notation} summarizes the notation used throughout this work.

We conclude with an outlook on extensions in \Cref{sec:conclusion}, integration with scientific pipelines, and open questions, and provide various details in \Cref{app:swe-details}. For a quick reference \Cref{sec:notation} displays the notation used in this work.

\section{Background}\label{sec:background}

A recurring theme across scientific computing and data-driven modeling is the observation that complex physical and mathematical systems often admit effective low-dimensional structure. Classical methods such as Principal Component Analysis (PCA) \cite{jolliffe2002pca} have long exploited this property to project high-dimensional dynamics onto reduced subspaces, providing tractable surrogates while retaining interpretability. 

Proper Orthogonal Decomposition (POD) \cite{benner2017model,benner_survey_2015} plays a central role here, systematically extracting dominant modes from solution data to form reduced-order models (ROMs). Building on this philosophy, projection-based ROM formulations—including Galerkin, Petrov--Galerkin, and Reduced Basis methods \cite{carlberg2017galerkin,quarteroni2015reduced}—extend these ideas to PDEs and parametric systems, enabling substantial computational savings. While these linear approaches have been remarkably successful, they struggle when the underlying dynamics lie on nonlinear manifolds that cannot be well-approximated by a fixed linear subspace. This limitation has motivated nonlinear, data-driven approaches.

The advent of autoencoders provides a nonlinear and adaptive alternative to classical projection methods. By learning expressive encoders and decoders, autoencoders construct compact latent spaces tailored to the underlying structure of data or model outputs. This process not only supports efficient storage and reconstruction but also enables the discovery of hidden variables or features that may be inaccessible to handcrafted linear reductions. As a result, data-driven approaches based on autoencoder frameworks have increasingly been adopted across diverse fields, including model reduction, optimization, prior learning, forecasting and prediction, and inverse problems. In particular, the model reduction community has embraced autoencoders as efficient nonlinear extensions of classical reduction techniques, offering a fresh perspective on reduced-order data-driven modeling and opening new opportunities for scientific computing \cite{hesthaven2018non,lee2020model}.

\emph{Projection Approaches.} Projection methods also play a central role in optimization, where constraints or separable structures must be exploited to reduce complexity. Classical approaches such as projected gradient descent enforce feasibility by mapping iterates back onto the constraint set $\mathcal{C}$, $x_{k+1} = \Pi_{\mathcal{C}}(x_k - \eta \nabla f(x_k))$, while variable projection eliminates a set of variables analytically to focus computation on the remaining ones \cite{golub1973differentiation, chung2023variable}. More recently, data-driven methods have extended this projection philosophy: learned proximal operators approximate constraint projections \cite{amos2017optnet, meinhardt2017learning}, and algorithm unrolling recasts optimization routines as trainable neural networks \cite{gregor2010learning, mardani2018neural}. These ideas highlight a recurring pattern—optimization efficiency is achieved by alternating between compression into a reduced space and structured mappings within that space. Traditional projections are hand-crafted and problem-specific, whereas modern data-driven approaches learn both the representation and the mappings jointly. This same principle underlies the Latent Twin framework, where autoencoders establish the reduced latent space and the trainable maps $m^{\to}$ and $m^{\gets}$ serve as interpretable surrogates of the original operators.

\emph{Inverse Problems.} Modern inverse problems increasingly utilize learned priors or data-driven end-to-end approaches to solve ill-posed mappings from measurements to states \cite{arridge2019solving, jin2017deep, stuart_inverseproblems}. Among these, \emph{paired autoencoders} have emerged as a simple and effective paradigm: one trains \emph{parallel} encoders/decoders $(e_x,d_x)$ for the state space and $(e_y,d_y)$ for the measurement space, together with a learnable mapping between their latent variables \cite{chung2024paired, hart2025paired, piening2024paired, chung2025good}. This design separates \emph{representation} from \emph{inference}: the autoencoders learn low-dimensional manifolds that capture the geometry of states and observations, while a latent-to-latent map handles the cross-space correspondence. This perspective naturally extends to the latent-twin framework: parallel encoders and decoders anchor the latent manifolds, while trainable latent mappings act as interpretable surrogates of the forward and inverse operators. The result is a unified, end-to-end approach that not only supports inversion but also accommodates temporal evolution, thus broadening the scope beyond the standard inverse problem setting.

\emph{Operator Learning.} Operator learning has recently emerged as a powerful paradigm in scientific machine learning, focusing on approximating solution operators that map between infinite-dimensional function spaces. Landmark contributions include the Deep Operator Network (DeepONet) \cite{lu2021learning} and the Fourier Neural Operator (FNO) \cite{li2020fourier}, both of which demonstrated that neural networks can efficiently approximate nonlinear operators arising in PDEs. The framework has since been extended to physics-informed, multiscale, and probabilistic settings \cite{kovachki2023neural}. 

A closely related idea is the recently proposed Latent DeepONet (L-DeepONet)~\cite{Kontolati2024}, which combines autoencoders with DeepONet to perform operator learning in a reduced latent space. In L-DeepONet, an autoencoder is first trained to compress high-dimensional solution fields, after which a DeepONet maps inputs (such as parameters, initial, or boundary conditions) to latent codes that are then decoded back to the physical solution. This two-stage design has demonstrated improved efficiency and accuracy over vanilla DeepONet and FNO on large-scale PDE problems. Building on this approach, PI-Latent-NO \cite{karumuri2025physicsinformedlatentneuraloperator} addresses the data-hungry nature of L-DeepONet by incorporating physics-informed constraints directly into the training process. This framework employs two coupled DeepONets—a Latent-DeepONet for dimensionality reduction and a Reconstruction-DeepONet for mapping back to physical space—trained end-to-end with governing physics integrated throughout. While this represents a significant advance in reducing data requirements for latent operator learning, it remains focused on parametric PDE problems rather than temporal evolution.

Latent Twins naturally embed into this operator-learning perspective while offering distinctive features. In contrast to L-DeepONet's two-stage design and PI-Latent-NO's parametric focus, Latent Twins integrate representation learning and operator approximation into a single framework targeting temporal evolution. They (i) use autoencoders tailored to underlying dynamics, (ii) incorporate paired forward and inverse mappings, and (iii) approximate global solution operators directly, enabling transitions between arbitrary time points without recursive integration. This provides a structured, representation-driven approach to learning dynamical systems with theoretical guarantees and integration into scientific pipelines.

\emph{Data Assimilation.} Data assimilation (DA) fuses prior knowledge, model dynamics, and noisy, partial observations to estimate the evolving state of a physical system—classically through variational (3D/4D-Var) and sequential (Kalman/ensemble) methods \cite{kalman1960new, ghil1991data, kalnay2003atmospheric}. In large physical models, however, the curse of dimensionality forces compromises on data-driven approaches such as Ensemble Kalman Filters \cite{evensen2009data}. A recent trend mitigates this by learning low-dimensional representations in which assimilation and forecasting are cheaper and often more stable. Variational schemes have been adapted to operate in learned latent spaces—e.g., ``bi-reduced'' 3D-Var with attention-based CAEs \cite{mack2020attention} or VAE-based differentiable 3D-Var objectives whose static latent covariances induce flow-dependent structure in state space \cite{melinc_3d-var_2024}. Sequential approaches have similarly been adapted to latent spaces: Latent Assimilation executes filtering directly in the latent domain but relies on a separate surrogate (e.g., LSTM/analogs) to propagate dynamics \cite{paszynski_data_2021, amendola2021data, grooms2021analog}, with extensions to EnKF/ensemble smoothers using VAE pairs \cite{pasmans_ensemble_2025, canchumuni2019towards} and score-based latent priors for trajectory inference \cite{rozet2023score, si_latent-ensf_2024}. Our latent-twin viewpoint is compatible with—and complementary to—this line of work: autoencoders furnish task-adapted latent manifolds for the observation and state spaces, while the learned Latent Twin $f^{\to}$ serves as a lightweight, end-to-end forecast operator that can be embedded in either variational objectives (latent-space 4D-Var) or sequential updates (latent-space filters).

\emph{Koopman Operators.} The Latent Twin approach carries connections to the rich literature on approximating the Koopman operator. The Koopman operator provides a linear framework for analyzing nonlinear dynamical systems. Given a discrete-time dynamical system $x_{t+1} = \Phi(x_t)$, the Koopman operator $\mathscr{K}$ acts on scalar-valued observables $g\colon \mathbb{R}^n \to \mathbb{C}$ by composition with the dynamics, $\mathscr{K} g(x) = (g \circ \Phi)(x)$. While the state-space dynamics may be nonlinear, the Koopman operator is linear (though infinite-dimensional) in function space, enabling the use of spectral analysis and linear decomposition techniques for nonlinear systems~\cite{koopman1931hamiltonian, mezic2005spectral}. Practical methods approximate this operator in finite dimensions, with Dynamic Mode Decomposition (DMD) among the most widely used approaches \cite{tu2014dynamic, williams2015data, proctor2016dynamic}. Extensions of DMD have been successful in fluid dynamics, control, and signal processing, though their expressiveness is limited by fixed dictionaries of observables.

Recent work combines Koopman theory with deep learning. \emph{Koopman autoencoders} employ an encoder $e$, a linear latent evolution operator $\mathfrak{K}$, and a decoder $d$. The encoder maps system states $x_t$ to latent coordinates $z_t = e(x_t)$, which evolve approximately linearly as $z_{t+1} \approx \mathfrak{K} z_t$, and are decoded as $\hat{x}_{t+1} = d(z_{t+1})$. These architectures enforce both accurate reconstruction and approximate linearity in the latent space~\cite{lusch2018deep, takeishi2017learning, morton2018deep, si_latent-ensf_2024}. While Koopman autoencoders have demonstrated strong performance for nonlinear dynamics, they remain tied to discrete time-evolution, requiring recursive applications for long-term predictions. Latent Twins provide a complementary perspective: instead of approximating one-step dynamics, they directly learn global solution operators, mapping states between arbitrary time points $t_1$ and $t_2$. This formulation naturally accommodates continuous-time dynamics and PDE-governed systems, avoiding the need for recursive integration. In this sense, Latent Twins can be viewed as lightweight neural operators that retain the interpretability of paired latent spaces while providing a flexible surrogate for temporal evolution.

\emph{Latent-Space Dynamics.} Complementing these nonlinear reductions, another line of work learns the governing dynamics directly in a latent space. The Latent Space Dynamics Identification (LaSDI) framework exemplifies this idea: an autoencoder compresses high-dimensional states, a system-identification stage discovers latent ODEs, and the latent model is used for prediction~\cite{bonneville2024comprehensive}. Variants adapt this pipeline for robustness to noise, thermodynamic consistency, or active learning strategies to reduce data costs \cite{tran2024weak, park2024tlasdi, he2023glasdi, bonneville2024gplasdi}. These contributions sit naturally between classical ROM and operator learning: they retain the efficiency of reduced coordinates while shifting the modeling burden from full state space to latent dynamics.

\medskip
\noindent
The literature on autoencoder-based scientific machine learning is extensive. The references provided here are representative rather than exhaustive, chosen to highlight key threads across model reduction, optimization, inverse problems, operator learning, data assimilation, and latent-space dynamics, and to illustrate the broad range of connections to the Latent Twin framework.

\section{Latent Twins for Dynamical Systems} \label{sec:latenttwins}

As discussed earlier, the premise behind the Latent Twin framework is that latent spaces and their corresponding latent mappings can serve as efficient surrogates—or ``Twins''—of complex operators. These include forward processes, inverse problems, but also time–evolution solution operators (flow-maps) governed by ODEs or PDEs as we will consider here. By operating within suitably constructed latent spaces, the Latent Twin framework may seek to emulate the behavior of high-dimensional, unstable, or ill-posed evolution equations in a tractable form. This strategy is particularly valuable when direct computation in the original state dynamics is infeasible due to high dimensionality, computational complexity, or ill-posedness \cite{stuart_inverseproblems, arridge2019solving, jin2017deep}.

To be more precise, here we consider a state trajectory $x:\mathcal{T}\to X$ governed by an operator equation  $S(x,t)=0$. This abstract viewpoint encompasses a wide class of evolution models, including continuous-time ODEs, PDEs, and integro-differential equations \cite{hairer1993solving1, hairer1996solving2, evans2022partial, robinson2012introduction}. In many applications—such as weather prediction, fluid dynamics, or time-series analysis of audio and video—the governing dynamics may be available only through data or simulations rather than an explicit model, making surrogate representations particularly valuable.

For ODEs and PDEs, a Latent Twin as defined in \Cref{def:latenttwin} takes a simplified form: it consists of a \emph{single} autoencoder $\Lambda_\theta = d \circ e$ with encoder $e\colon X \to Z$ and decoder $d\colon Z \to X$, together with a latent evolution map $m\colon Z \times \mathcal{T} \times \mathcal{T} \to Z$. 
In this setting, we identify $Y=X$, since the operator maps the state space into itself and the dynamics evolve entirely within $X$. In this setting, separate forward and inverse Latent Twins are unnecessary, as the latent map itself accommodates bi-directional time-evolution. \Cref{fig:latenttwinDS} illustrates the Latent Twin construction in the setting of time-dependent systems. Although we do not incorporate noisy inputs into the autoencoders in the present work, we note this as a natural extension and particularly interesting for inverse problems.  We stress that for time points $t_1, t_2$, no assumptions on ordering ($t_2 > t_1$) or equidistant spacing ($t_2 = t_1 + h$) are required.

\begin{figure}
    \centering
    \begin{tikzpicture}

	\node[fill=python1!50, minimum width=3.0cm, minimum height=0.5cm] (x) at (0,0) {$x(t_1)+\varepsilon(t_1)$};
    
    \draw[fill=python9!50,draw=none]
             ([yshift=-0.2cm,xshift=0.0cm]x.south west) -- ([yshift=-1.5cm,xshift=0.75cm]x.south west) -- ([yshift=-1.5cm,xshift=-0.75cm]x.south east) -- ([yshift=-0.2cm,xshift=0.0cm]x.south east) -- cycle node (e) at ([yshift=-0.75cm,xshift=0.0cm]x.south) {$\begin{matrix}\mbox{encoder} \\[-0.0ex] e \end{matrix}$}; 
               
    \node[fill=python3!50, minimum width=1.5cm, minimum height=0.5cm] (zx) at (0,-2.3) {$z(t_1)$};    
    \draw[fill=python4!50,draw=none]
             ([yshift=-0.2cm,xshift=0.0cm]zx.south west) -- ([yshift=-1.5cm,xshift=-0.75cm]zx.south west) -- ([yshift=-1.5cm,xshift=0.75cm]zx.south east) -- ([yshift=-0.2cm,xshift=0.0cm]zx.south east) -- cycle node (d) at ([yshift=-0.75cm,xshift=0.0cm]zx.south) {$\begin{matrix}\mbox{decoder} \\[-0.0ex] d \end{matrix}$};  
    \node[fill=python1!50, minimum width=3.0cm, minimum height=0.5cm] (xt) at (0,-4.5) {$\tilde x(t_1)$};

    \node[fill=python2!50, minimum width=3.0cm, minimum height=0.5cm] (b) at (8.5,0) {$x(t_2)+\varepsilon(t_2)$};
    \draw[fill=python9!50,draw=none]
             ([yshift=-0.2cm,xshift=0.0cm]b.south west) -- ([yshift=-1.5cm,xshift=0.75cm]b.south west) -- ([yshift=-1.5cm,xshift=-0.75cm]b.south east) -- ([yshift=-0.2cm,xshift=0.0cm]b.south east) -- cycle node (e) at ([yshift=-0.75cm,xshift=0.0cm]b.south) {$\begin{matrix}\mbox{encoder} \\[-0.0ex] e \end{matrix}$}; 
               
    \node[fill=python6!50, minimum width=1.5cm, minimum height=0.5cm] (zb) at (8.5,-2.3) {$z(t_2)$};    
    \draw[fill=python4!50,draw=none]
             ([yshift=-0.2cm,xshift=0.0cm]zb.south west) -- ([yshift=-1.5cm,xshift=-0.75cm]zb.south west) -- ([yshift=-1.5cm,xshift=0.75cm]zb.south east) -- ([yshift=-0.2cm,xshift=0.0cm]zb.south east) -- cycle node (d) at ([yshift=-0.75cm,xshift=0.0cm]zb.south) {$\begin{matrix}\mbox{decoder} \\[-0.0ex] d \end{matrix}$};  
    \node[fill=python2!50, minimum width=3.0cm, minimum height=0.5cm] (bt) at (8.5,-4.5) {$\tilde x(t_2)$};

    \node[fill=python1!50, minimum width=1.0cm, minimum height=0.5cm] (t1) at (3.5,-1.3) {$t_1$};
    \node[fill=python3!50, minimum width=1.0cm, minimum height=1.cm] (zt1) at (3.5,-2.3) {$z(t_1)$};
    \node[fill=python2!50, minimum width=1.0cm, minimum height=0.5cm] (t2) at (3.5,-3.3) {$t_2$};

    \draw[fill=python8!50,draw=none] ([xshift=0.5cm]t1.north east) -- ([xshift=-0.5cm,yshift=0.3cm]zb.west) -- ([xshift=-0.5cm,yshift=-0.3cm]zb.west) -- ([xshift=0.5cm]t2.south east) -- cycle; 
    	\node at (5.8,-2.25) {$\begin{matrix}\mbox{latent evolution} \\ m \end{matrix}$};

    \begin{scope}[-latex,shorten >=9pt,shorten <=9pt,line width=2pt]
    \draw[black!75!white] ([yshift=-0.0cm]zx.east) to ([yshift=-0.0cm]zt1.west);

    \end{scope}        

\end{tikzpicture}
    \caption{\emph{Schematic of a Latent Twin for time-evolution systems.} The autoencoder (on the left) provides a latent representation of the system state. Given a potentially noisy observation $x(t_1) + \varepsilon(t_1)$, the goal is to recover the underlying state $x(t_1)$ by passing it through an encoder $e$ and decoder $d$, resulting in the latent representation $z(t_1)$. The autoencoder operates  independently of the specific time point and can be applied to any state $x(t)$, such as at time $t_2$ (on the right). The latent map $m$ (middle) takes the encoded state $z(t_1)$ along with the time pair $(t_1, t_2)$ and predicts the latent representation at the future time, $z(t_2)$. A  complete surrogate prediction--referred to as a Latent Twin--of the state at an arbitrary time $t_2$, given the state at $t_1$, is expressed as 
    $x(t_2) \approx (d \circ m(\,\cdot\,, t_1, t_2) \circ e)(x(t_1))$.}
    \label{fig:latenttwinDS}
\end{figure}
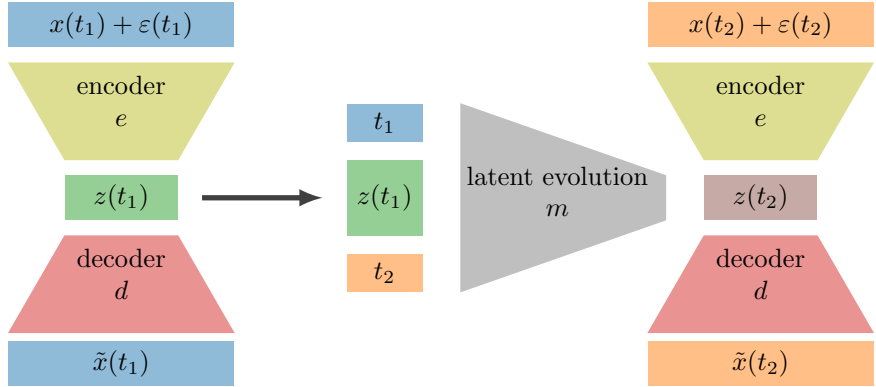

Given a state $x(t_1)$ at time $t_1$, the encoder produces a latent representation $z(t_1) = e(x(t_1))$. The latent evolution map then advances this state to any other time $t_2$,
\[
z(t_2) = m(z(t_1), t_1, t_2),
\]
and with the definition of the Latent Twin 
\begin{equation}\label{eq:latenttwin}
    \LT(t_2) = (d \circ m(\,\cdot\,,t_1,t_2)\circ e)(x(t_1))
\end{equation}
decoding yields the surrogate trajectory,
\[
x(t_2) \;\approx\; \LT(t_2).
\]
With this notation, we explicitly indicate the dependence on the network parameters $\theta$, highlighting the trainable nature of the Latent Twin.

The key distinction between the Latent Twin approach and purely pipeline-based architectures that sequentially process from state $x(t_1)$ to state $x(t_2)$ lies in their generality and structural design. Rather than simply passing data through a latent space while potentially performing operations within the latent space, the Latent Twin explicitly organizes the data so that the latent spaces are meaningfully aligned. Here, the objective is to ensure that the latent spaces represent corresponding encoded states $x$. This is achieved through autoencoders that are trained to accurately represent the underlying data distributions, ensuring that the latent spaces capture essential structure for downstream mappings. 

To train the Latent Twin, we optimize the autoencoder and the latent evolution map jointly, so that the state representation and its temporal dynamics are learned simultaneously. The training objective can be written as a Bayes risk minimization problem
\begin{equation}\label{eq:bayesrisk}
    \min_\theta \; \mathbb{E}\; \left(
         \mathcal{L}\Big(\Lambda_\theta(x_1),\, x_1\Big) 
       + \mathcal{L}\Big(\LT(t_2),\, x_2\Big)\right),
\end{equation}
where the expectation is taken over pairs of states $x_1, x_2$ sampled from the system dynamics at random time points $t_1$ and $t_2$, with both the states and time points regarded as random variables. The loss function $\mathcal{L}$ is typically chosen as a mean squared error, though additional regularization terms may be incorporated depending on the application. For a given representative dataset $\{(t_1^j, x_1^j), (t_2^j, x_2^j)\}_{j = 1}^J$ we may rephrase \Cref{eq:bayesrisk} in its empirical form
\begin{equation}\label{eq:empirical}
    \min_\theta \; \tfrac{1}{J} \sum_{j = 1}^J \left( 
         \mathcal{L}\Big(\Lambda_\theta(x_1^j),\, x_1^j\Big) 
       + \mathcal{L}\Big(\LT[t_1^j][x_1^j](t_2^j),\, x_2^j\Big) \right).
\end{equation}

\subsection{Approximation Properties of Latent Twins}\label{sub:theory}

A central theoretical question is whether Latent Twins can act as universal approximators of solution operators for dynamical systems. In particular: given an initial state and time, under what conditions can the Latent Twin $\LT$ reliably approximate the underlying true flow-map $\Phi$, and how do its error characteristics compare with those of traditional recursive timemarching schemes? Addressing this requires a closer look at the theoretical properties of solution operators and the approximation capabilities of neural networks.

\subsubsection{ODEs}\label{sub:odes}
Consider the ODE
\[
x' = G(x,t), \qquad x(t_0) = x_0, \quad t\in[0,T],
\]
with solution operator $\Phi(t_2,t_1,x_1) = x(t_2)$ where $x(t_1)=x_1$. 
As a simple example, for the linear ODE $x' = Mx$ with $M \in \mathbb{R}^{n\times n}$, the exact flow is given by the matrix exponential, $\Phi(\,\cdot\,, t_1,x_1) = \e^{M (\,\cdot\, - t_1)}x_1$.
The Latent Twin $\LT(\,\cdot\,)$ is trained to approximate $\Phi$ by learning a mapping
\[
 \LT[t_1][x_1](t_2) \approx \Phi(t_2,t_1,x_1),
\]
from a data set $\{(t_1^j,x_1^j), (t_2^j,x_2^j)\}_{j=1}^J$.

To establish that the Latent Twin $\LT$ is a valid approximator of the true solution operator, two ingredients are required: (i) the underlying ODE must admit a well-posed flow-map (ensuring existence, uniqueness, and continuity of solutions), and (ii) the neural network architecture realizing $\LT$ must be expressive enough to approximate such flows on compact domains.  The Latent Twin  $\LT = d \circ m \circ e$ is realized as the composition of three components: 
(a) an \emph{encoder} $e(\,\cdot\,)$ that maps states to a latent representation, 
(b) a \emph{latent map} $m(\,\cdot\,,t_1,t_2)$ that evolves this representation between times $t_1$ and $t_2$, 
and (c) a \emph{decoder} $d(\,\cdot\,)$ that reconstructs the state from latent space.

\begin{enumerate}
    \item[(i)] \textbf{Well-posedness of the ODE.} 
    We adopt the standard Picard–Lindelöf framework \cite{hairer1993solving1} to guarantee well-posedness, formulated in a two-sided form to ensure existence and uniqueness both forward and backward in time.

    \begin{assumption}[ODE]\label{ass:twosided}
    Let $G:\mathbb{R}^{n_x}\times[0,T]\to\mathbb{R}^{n_x}$ be the governing vector field of
    \[
        x'(t)=G(x(t),t).
    \]
    Assume that for each $t\in[0,T]$, the map $G(\cdot,t)$ is locally Lipschitz in $x$, uniformly in $t$, and that $G$ satisfies an \emph{at-most linear growth bound}, i.e.,
    \[
        \|G(x,t)\| \;\le\; a(t) + b\|x\|, 
    \qquad \text{for all } (x,t)\in\mathbb{R}^{n_x}\times[0,T],
    \]
    for some Lebesgue measurable function $a:[0,T]\to\mathbb{R}_+$ and constant $b\ge 0$.

    Then for every $(t_0,x_0)\in[0,T]\times\mathbb{R}^{n_x}$ there exists a unique solution $x(\cdot)\in C([0,T];\mathbb{R}^{n_x})$, which is in fact absolutely continuous and satisfies $x'(t)=G(x(t),t)$ for almost every $t\in[0,T]$.  Consequently, the associated solution operator
    \[
    \Phi(t_2,t_1,x_1) = x(t_2)\quad\text{with }x(t_1)=x_1
    \]
    is well defined and jointly continuous in $(t_1,t_2,x_1)$ on compact subsets of $[0,T]\times[0,T]\times\mathbb{R}^{n_x}$.
    \end{assumption}

    \item[(ii)] \textbf{Approximation capacity of the Latent Twin.} 
    Since trajectories remain bounded on $[0,T]$ by \Cref{ass:twosided}, all states lie in some compact $X \subset \mathbb{R}^{n_x}$. Restricting to such compact sets is natural and provides the domains where universal approximation theorems apply.
    
    We take
    \[
    \mathcal{E}\subset C(X,Z), \quad
    \mathcal{D}\subset C(Z,X), \quad
    \mathcal{M}\subset C(Z\times[0,T]^2,Z),
    \]
    to be network classes dense in $C(K,\mathbb{R}^m)$ for compact $K$. Because $\Phi$ is continuous on compact sets by \Cref{ass:twosided}, these results apply directly. Examples include MLPs \cite{cybenko1989approximation,funahashi1989approximate,hornik1989multilayer,hornik1991approximation,leshno1993multilayer,pinkus1999approximation}, CNNs \cite{zhou2020universality,yarotsky2018invariant}, ResNets \cite{lu2017expressive,lin2017resnet,kidger2020universal}, and Transformers \cite{yun2020transformers,perez2019turing}.

\begin{enumerate}
    \item[(a)] \emph{Encoder.} 
    The encoder translates physical states into latent codes. 
    \begin{assumption}[Encoder]\label{ass:encoder}
    The encoder $e\in\mathcal{E}$ is Lipschitz continuous with constant $L_e>0$, i.e.,
    \[
    \|e(x)-e(y)\|\;\le\; L_e \|x-y\| \qquad \text{for all } x,y\in X.
    \]
    \end{assumption}
    
    \item[(b)] \emph{Latent map.} 
    The latent map governs how encoded states evolve between times $t_1$ and $t_2$. 
    \begin{assumption}[Latent map]\label{ass:latentmap} 
    The latent evolution $m\in\mathcal{M}$ is continuous and approximates the exact latent flow
    \[
    m^\ast(z_1,t_1,t_2) := e\!\left(\Phi(t_2,t_1,d(z_1))\right),
    \]
    where $\Phi$ is the solution operator of the ODE and $e,d$ are the encoder and decoder. 
    That is, for every $\varepsilon_{\mathrm{map}}>0$ there exists $m\in\mathcal{M}$ such that
    \[
    \sup_{(z_1,t_1,t_2)\in Z\times[0,T]^2}\;
    \|m(z_1,t_1,t_2)-m^\ast(z_1,t_1,t_2)\|\;\le\;\varepsilon_{\mathrm{map}}.
    \]
    In other words, $m^\ast$ represents the exact latent dynamics induced by the ODE flow and the encoder–decoder pair.
    \end{assumption}
    
    \item[(c)] \emph{Decoder.} 
    The decoder reconstructs physical states from latent codes. 
    \begin{assumption}[Decoder]\label{ass:decoder}
    The decoder $d\in\mathcal{D}$ is Lipschitz continuous with constant $L_d>0$, i.e.,
    \[
    \|d(z_1)-d(z_2)\|\;\le\; L_d \|z_1-z_2\| \qquad \text{for all } z_1,z_2\in Z,
    \]
    and the reconstruction error satisfies
    \[
    \sup_{x\in X}\;\|d(e(x))-x\|\;\le\;\varepsilon_{\mathrm{AE}},
    \]
    where $\varepsilon_{\mathrm{AE}}>0$ can be made arbitrarily small.
    \end{assumption}
\end{enumerate}
\end{enumerate}

Note, when the latent space dimension matches the state space ($n_z = n_x$),  the encoder and decoder need not introduce any representation error. In this case one may take $e=\operatorname{id}_X$ and $d=\operatorname{id}_Z$, or networks that approximate them arbitrarily well. Consequently, the autoencoder error $\varepsilon_{\mathrm{AE}}$ vanishes (or can be made negligible), and the approximation task reduces entirely to learning the latent dynamics $m$. This is especially relevant for low-dimensional ODE systems, where no benefit is gained from overparameterizing $e$ or $d$, and the expressive power can be focused exclusively on modeling temporal evolution.

Having laid out the necessary assumptions on the dynamical system, encoder, latent map, and decoder, we can now state our main result: the Latent Twin furnishes a controlled approximation of the true solution operator.

\begin{theorem}[Latent Twin Uniform Approximation]\label{thm:end-to-end}

Suppose \Cref{ass:twosided}--\Cref{ass:decoder} hold. In particular, the encoder–decoder pair introduces at most a reconstruction error $\varepsilon_{\mathrm{AE}}$, the latent map approximates $m^\ast$ with error $\varepsilon_{\mathrm{map}}$, and the ODE flow $\Phi$ is Lipschitz in its initial state with constant $L_G>0$. Then $\LT$ satisfies
\[
\sup_{(t_1,t_2,x_1)\in[0,T]^2\times X}
\big\|\LT[t_1][x_1](t_2)-\Phi(t_2,t_1,x_1)\big\|
\;\le\;
(1+\operatorname{e}^{\,L_G T})\,\varepsilon_{\mathrm{AE}}
\;+\;
L_d\,\varepsilon_{\mathrm{map}}.
\]
\end{theorem}

\begin{proof}
    Fix any $x_1\in X$. The Latent Twin acts as 
    \[
        \LT[t_1][x_1](t_2) \;=\; (d \circ m(\cdot,t_1,t_2)\circ e)(x_1),
    \]
    while the \emph{exact} latent evolution from \Cref{ass:latentmap} is
    \[
        m^\ast(z_1,t_1,t_2) \;=\; e\!\left(\Phi(t_2,t_1,d(z_1))\right).
    \]
    By the triangle inequality,
    \begin{multline}\label{eq:proofineq}
        \big\|\LT[t_1][x_1](t_2)-\Phi(t_2,t_1,x_1)\big\| \\
        \leq 
        \underbrace{\big\|d(m(e(x_1),t_1,t_2))-d(m^\ast(e(x_1),t_1,t_2))\big\|}_{\text{(a)}}
        + \underbrace{\big\|d(m^\ast(e(x_1),t_1,t_2))-\Phi(t_2,t_1,x_1)\big\|}_{\text{(b)}}.
    \end{multline}

    \begin{enumerate}
        \item[(a)] The only discrepancy in the first term is that $m$ replaces the exact latent evolution $m^\ast$. By Lipschitz continuity of $d$ (\Cref{ass:decoder}) and the uniform approximation property of $m$ (\Cref{ass:latentmap}), we obtain
        \[
            \big\|d(m(e(x_1),t_1,t_2))-d(m^\ast(e(x_1),t_1,t_2))\big\|
            \;\le\; L_d\,\varepsilon_{\mathrm{map}}.
        \]

        \item[(b)] The second term involves both the autoencoder and the flow. Using the definition of $m^\ast$, we write
        \[
            d(m^\ast(e(x_1),t_1,t_2))
            = d\!\left(e(\Phi(t_2,t_1,d(e(x_1))))\right).
        \]
        To compare with $\Phi(t_2,t_1,x_1)$, we insert the intermediate state $\Phi(t_2,t_1,d(e(x_1)))$. This separates the error into two contributions:
        \begin{align*}
            &\big\|d(m^\ast(e(x_1),t_1,t_2))-\Phi(t_2,t_1,x_1)\big\| \\
            &\quad\leq 
            \underbrace{\big\|d(e(\Phi(t_2,t_1,d(e(x_1))))) - \Phi(t_2,t_1,d(e(x_1)))\big\|}_{\text{(b.i)}}
            + \underbrace{\big\|\Phi(t_2,t_1,d(e(x_1)))-\Phi(t_2,t_1,x_1)\big\|}_{\text{(b.ii)}}.
        \end{align*}

        \begin{enumerate}
            \item[(b.i)] Since the autoencoder reconstructs any state with error at most $\varepsilon_{\mathrm{AE}}$, 
            \[
                \big\|d(e(\Phi(t_2,t_1,d(e(x_1))))) - \Phi(t_2,t_1,d(e(x_1)))\big\|
                \;\leq\; \varepsilon_{\mathrm{AE}}.
            \]

            \item[(b.ii)] By \Cref{ass:twosided}, the flow $\Phi$ is Lipschitz in its initial condition. Grönwall’s inequality 
            \cite{gronwall1919note} gives
            \[
                \|\Phi(t_2,t_1,x)-\Phi(t_2,t_1,y)\| 
                \;\le\; \operatorname{e}^{\,L_G|t_2-t_1|}\,\|x-y\|,
                \qquad \text{for all } x,y\in X.
            \]
            Since $\|d(e(x_1))-x_1\|\le \varepsilon_{\mathrm{AE}}$ and $|t_2-t_1|\leq T$, this implies
            \[
                \big\|\Phi(t_2,t_1,d(e(x_1)))-\Phi(t_2,t_1,x_1)\big\|
                \;\le\; \operatorname{e}^{\,L_G T}\,\varepsilon_{\mathrm{AE}}.
            \]
        \end{enumerate}

        Combining (b.i) and (b.ii), we conclude
        \[
            \big\|d(m^\ast(e(x_1),t_1,t_2))-\Phi(t_2,t_1,x_1)\big\|
            \;\le\; (1+\operatorname{e}^{\,L_G T})\,\varepsilon_{\mathrm{AE}}.
        \]
    \end{enumerate}

    Finally, combining (a) and (b) yields
    \[
        \big\|\LT[t_1][x_1](t_2)-\Phi(t_2,t_1,x_1)\big\|
        \;\le\; L_d\,\varepsilon_{\mathrm{map}} + (1+\operatorname{e}^{\,L_G T})\,\varepsilon_{\mathrm{AE}}.
    \]
    Taking the supremum over $(t_1,t_2,x_1)\in[0,T]^2\times X$ yields the stated bound.
\end{proof}

This result shows that the Latent Twin’s overall error is controlled by two 
structurally interpretable terms: the autoencoder error $\varepsilon_{\mathrm{AE}}$ 
and the latent flow approximation error $\varepsilon_{\mathrm{map}}$. We want to add a couple of remarks to highlight properties and consequences for the Latent Twin.

\paragraph{Remarks.} 
\begin{enumerate}
    \item The assumptions for \Cref{thm:end-to-end} align naturally with the architecture: the encoder–decoder pair must approximate the identity on $X$ with finite Lipschitz constants, while the latent map $m$ approximates the exact latent evolution $m^\ast$. If the latent space coincides with the state space ($Z=X$), one may choose the encoder and decoder to be the identity maps. In this case the reconstruction error vanishes ($\varepsilon_{\mathrm{AE}}=0$) and the bound in \Cref{thm:end-to-end} reduces to
    \[
    \sup_{(t_1,t_2,x_1)\in[0,T]^2\times X}
    \big\|\LT[t_1][x_1](t_2)-\Phi(t_2,t_1,x_1)\big\|
    \;\le\; L_d\,\varepsilon_{\mathrm{map}}.
    \]
    Thus, trivially, in low-dimensional problems the approximation error is determined entirely by the latent map, since the state representation introduces no additional error.

    \item Even when each time step introduces only a small local defect $\delta(\Delta t)$, the global error of time-marching schemes, such as Runge-Kutta, multistep or RNNs, LSTMs, typically grows \emph{exponentially} with the horizon $T$, 
    \[
    \|x_{\mathrm{num}}(T)-\Phi(T,t_0,x_0)\| \;\approx \; \operatorname{e}^{\,L_G T}\,\delta(\Delta t),
    \]
    where $L_G$ is a Lipschitz constant of the flow and $x_{\mathrm{num}}$ being the numerical solution of the marching scheme.
    
    \emph{By contrast}, the Latent Twin evaluates the state at $t_2$ in a single step, so this type exponential growth mechanism does not occur: its error remains uniformly bounded and depends only on approximation capacity, not on the number of steps taken.

    \item Restricting the Latent Twin to short time gaps $|t_2-t_1|$ yield sharper short-term accuracy: the exponential Lipschitz factor $\operatorname{e}^{L_G|t_2-t_1|}$ in \Cref{thm:end-to-end} remains small, 
    and the latent map $m$ only needs to approximate dynamics over limited horizons. Training on longer-range pairs instead emphasizes global dynamical structure, but at the expense of reduced pointwise fidelity. This trade-off is particularly visible in chaotic systems such as Lorenz--63, where local 
    training better reproduces short-term trajectories, while global training captures only coarse trends, see \Cref{subsec:odeexpt} below. 
\end{enumerate}

\paragraph{Linear ODEs.} \label{sec:linearODEtheory}
The linear ODE setting provides a transparent testbed for Latent Twins, showing how latent mappings $m$ can be designed to mirror true dynamics and clarifying their connection to classical projection--based model reduction. Consider the linear system
\[
x'(t)=Mx(t),
\]
whose solution is given by the flow operator
\[
\Phi(t_2,t_1,x_1)=\operatorname{e}^{(t_2-t_1)M}x_1.
\] 

To embed this system into the latent--twin framework, we introduce an encoder--decoder pair defined by an orthogonal matrix 
$U\in\mathbb{R}^{n\times r}$ with $U^\top U=I_r$:
\[
e(x)=U^\top x,\qquad d(z)=Uz .
\]
A natural and structurally faithful choice for the latent map is
\[
m(z_1,t_1, t_2)=\operatorname{e}^{(t_2-t_1)W}z_1,
\]
so that the Latent Twin evolves as
\begin{equation}\label{eq:ltODE}
\LT[t_1][x_1](t_2) = U\,\operatorname{e}^{(t_2-t_1)W}\,U^\top x_1,
\end{equation}
with trainable generator $W\in\mathbb{R}^{r\times r}$.  
This choice deliberately mimics the exact flow $\exp((t_2-t_1)M)$, endowing the Latent Twin with the correct semigroup structure and a physics--inspired, interpretable parameterization of dynamics.  
If $U=I$, this recovers the full--state case; if $r<n$, $U$ acts as a projection, directly 
linking the Latent Twin to POD reduction.  

Indeed, if $U$ contains the leading $r$ POD modes, then $W=U^\top M U$ is the Galerkin--projected operator, and the Latent Twin flow
\[
\LT(t_2,t_1,x_1)=U \operatorname{e}^{(t_2-t_1)W}U^\top x_1
\]
coincides exactly with the POD reduced flow.  Beyond this equivalence, Latent Twins also provide explicit error control: if training succeeds in driving $W$ close to $M$, then a perturbation bound for the matrix exponential \cite[Thm.~10.18]{higham2008functions} shows that, uniformly for $|h|\le H$,
\[
\sup_{\substack{|h|\le H \\ x_1\in X}}
\big\|\LT[t_1][x_1](t_1{+}h)-\Phi(t_1{+}h,t_1,x_1)\big\|
\;\le\; 
H\,\operatorname{e}^{\,H(\|M\|+\|W-M\|)}\,\|W-M\|\cdot\sup_{x_1\in X}\|x_1\|.
\]
Thus the operator error is governed directly by the generator mismatch $\|W-M\|$, with exact recovery when $W=M$.  

The key distinction from POD is that while POD fixes $U$ by snapshot SVD and projects the known operator $M$, Latent Twins may learn both the encoder--decoder and the latent dynamics directly from data.  In the linear case this recovers POD, but the architecture generalizes seamlessly: choices like $m(z_1,t_1,t_2)=\exp((t_2-t_1)W)z_1$ illustrate how embedding physical structure directly into $m$ yields interpretable and accurate models, providing a natural template for nonlinear extensions. An illustrative example of this setup is presented in the ODE experiments of \Cref{subsec:odeexpt}.

\subsubsection{PDEs}\label{sub:pdes}

For PDEs, the theoretical framework naturally extends the ODE analysis by introducing spatial discretization as an additional layer. We consider a PDE defined on the function space $\mathcal{X}$ and show that Latent Twins can approximate the continuous solution operator through a careful three-step process: spatial discretization, Latent Twin approximation on the discrete space, and reconstruction to the continuous space.

\begin{assumption}[PDE well-posedness \cite{robinson2001infinite}]\label{ass:pde-wellposed}
Let $(\mathcal{X}, \|\,\cdot\,\|_{\mathcal{X}})$ be a separable Banach space of functions. Consider the PDE written as an abstract evolution equation:
\[
\frac{\operatorname{\partial}u}{\operatorname{\partial}t} = \mathcal{L}(u, t), \quad u(0) = u_0 \in \mathcal{X}, \quad t \in [0,T],
\]
where the PDE operator $\mathcal{L}: \mathcal{X} \times [0,T] \to \mathcal{X}$ satisfies:
\begin{enumerate}
\item[(i)] For each $t \in [0,T]$, the map $\mathcal{L}(\cdot, t): \mathcal{X} \to \mathcal{X}$ is locally Lipschitz continuous uniformly in $t$.
\item[(ii)] $\mathcal{L}$ satisfies the linear growth bound
\[
\|\mathcal{L}(u, t)\|_{\mathcal{X}} \leq a(t) + b \|u\|_{\mathcal{X}}
\]
for some measurable $a: [0,T] \to \mathbb{R}_+$ and constant $b \geq 0$.
\end{enumerate}
Then there exists a compact invariant set $\mathcal{K} \subset \mathcal{X}$ (invariant under both forward and backward evolution) such that the solution operator $\Phi(t_2, t_1, \cdot): \mathcal{K} \to \mathcal{K}$ is well defined, unique, and jointly continuous in $(t_1,t_2,u_1)$ on $[0,T]^2 \times \mathcal{K}$.
Moreover, $\Phi$ is Lipschitz continuous with constant $L_{\mathcal{K}}>0$ in its initial condition:
\[
\|\Phi(t_2, t_1, u) - \Phi(t_2, t_1, v)\|_{\mathcal{X}} \leq \operatorname{e}^{L_{\mathcal{K}} |t_2 - t_1|} \|u - v\|_{\mathcal{X}}, \quad \text{for all } u, v \in \mathcal{K}.
\]
\end{assumption}

\begin{assumption}[Spatial discretization]\label{ass:spatial-disc}
Let $\{P_N: \mathcal{X} \to \mathbb{R}^{N}\}_{N \geq 1}$ and $\{R_N: \mathbb{R}^{N} \to \mathcal{X}\}_{N \geq 1}$ be spatial discretization and reconstruction operators (e.g., finite element projections \cite{ciarlet1978finite}) such that:
\begin{enumerate}
\item[(i)] \textbf{(Consistency)} $R_N \circ P_N \to I$ on $\mathcal{K}$ as $N \to \infty$, with discretization error
\[
\varepsilon_{\text{disc}}(N) := \sup_{u \in \mathcal{K}} \|R_N(P_N(u)) - u\|_{\mathcal{X}} \to 0.
\]
\item[(ii)] \textbf{(Stability)} The operators $P_N$ and $R_N$ are uniformly bounded: $\|R_N\|, \|P_N\| \leq C$.
\item[(iii)] \textbf{(Discrete flow consistency)} The discrete flow $\Phi_N(t_2, t_1, \cdot): \mathbb{R}^N \to \mathbb{R}^N$ satisfies
\[
\sup_{u \in \mathcal{K}, t_1, t_2 \in [0,T]} \|R_N(\Phi_N(t_2, t_1, P_N(u))) - \Phi(t_2, t_1, u)\|_{\mathcal{X}} \leq C \varepsilon_{\text{disc}}(N).
\]
\end{enumerate}
\end{assumption}

Given this setup, we construct a discrete Latent Twin $\LTN = d_N \circ m_N \circ e_N$ operating on the finite-dimensional space $\mathbb{R}^N$. Crucially, this discrete Latent Twin satisfies exactly the same assumptions as in the ODE case (\Cref{ass:encoder}--\Cref{ass:decoder}), but now applied to the compact domain $X_N := P_N(\mathcal{K}) \subset \mathbb{R}^N$:

\begin{enumerate}
\item The encoder $e_N: \mathbb{R}^N \to \mathbb{R}^{n_z}$ is Lipschitz on $X_N$ with constant $L_{e,N}$.
\item The decoder $d_N: \mathbb{R}^{n_z} \to \mathbb{R}^N$ is Lipschitz with constant $L_{d,N}$ and satisfies $\sup_{x \in X_N} \|d_N(e_N(x)) - x\| \leq \varepsilon_{\mathrm{AE}}^{(N)}$.
\item The latent map $m_N: \mathbb{R}^{n_z} \times [0,T]^2 \to \mathbb{R}^{n_z}$ approximates the exact discrete latent flow with error $\varepsilon_{\mathrm{map}}^{(N)}$.
\end{enumerate}
The continuous Latent Twin is then defined by lifting the discrete result back to the function space:
\[
\tildeLTN(t_2, t_1, u_1) := R_N(\LTN(t_2, t_1, P_N(u_1))).
\]
\begin{theorem}[Latent Twin Approximation for PDEs]\label{thm:pde}
Suppose \Cref{ass:pde-wellposed} and \Cref{ass:spatial-disc} hold and the discrete Latent Twin $\LTN$ satisfies the ODE assumptions (\Cref{ass:encoder}--\Cref{ass:decoder}) on $X_N = P_N(\mathcal{K})$ with errors $\varepsilon_{\mathrm{AE}}^{(N)}$ and $\varepsilon_{\mathrm{map}}^{(N)}$. Then the continuous Latent Twin satisfies
\[
\sup_{(t_1,t_2,u_1) \in [0,T]^2 \times \mathcal{K}} \|\tildeLTN(t_2, t_1, u_1) - \Phi(t_2, t_1, u_1)\|_{\mathcal{X}}
\leq \underbrace{C \varepsilon_{\text{disc}}(N)}_{\text{spatial discretization}} + \underbrace{C(1 + \operatorname{e}^{L_{\mathcal{K}} T}) \varepsilon_{\mathrm{AE}}^{(N)} + C L_{d,N} \varepsilon_{\mathrm{map}}^{(N)}}_{\text{discrete Latent Twin (ODE bound)}}.
\]
\end{theorem}

\begin{proof}
The proof follows a two-step decomposition. For any $(t_1, t_2, u_1) \in [0,T]^2 \times \mathcal{K}$:

\begin{equation}
\begin{aligned}
\|\tildeLTN(t_2, t_1, u_1) - \Phi(t_2, t_1, u_1)\|_{\mathcal{X}} 
&\leq \underbrace{\|R_N(\LTN(t_2, t_1, P_N(u_1))) - R_N(\Phi_N(t_2, t_1, P_N(u_1)))\|_{\mathcal{X}}}_{\text{discrete Latent Twin error}} \\
&\quad + \underbrace{\|R_N(\Phi_N(t_2, t_1, P_N(u_1))) - \Phi(t_2, t_1, u_1)\|_{\mathcal{X}}}_{\text{spatial discretization error}}.
\end{aligned}
\end{equation}

The second term is bounded by $C \varepsilon_{\text{disc}}(N)$ from \Cref{ass:spatial-disc}(iii).  For the first term, by stability of $R_N$ and since $\LTN$ operates on the finite-dimensional space $\mathbb{R}^N$, we obtain
\[
\|R_N(\LTN(t_2, t_1, P_N(u_1))) - R_N(\Phi_N(t_2, t_1, P_N(u_1)))\|_{\mathcal{X}} 
\leq C \|\LTN - \Phi_N\|_{\mathbb{R}^N}.
\]
Applying \Cref{thm:end-to-end} to the discrete system then yields
\[
\|R_N(\LTN(t_2, t_1, P_N(u_1))) - R_N(\Phi_N(t_2, t_1, P_N(u_1)))\|_{\mathcal{X}} 
\leq C\left((1 + \operatorname{e}^{L_{\mathcal{K}} T}) \varepsilon_{\mathrm{AE}}^{(N)} + L_{d,N} \varepsilon_{\mathrm{map}}^{(N)}\right).
\]
Combining the two terms gives the pointwise estimate above; taking the supremum over 
$(t_1,t_2,u_1)\in[0,T]^2\times \mathcal{K}$ then yields the stated bound.
\end{proof}

The above theorem shows that as $N \to \infty$ and the network capacity increases, the errors vanish, i.e., $\varepsilon_{\text{disc}}(N)\to 0$, $\varepsilon_{\mathrm{AE}}^{(N)}\to 0$, and $\varepsilon_{\mathrm{map}}^{(N)}\to 0$, so that $\tildeLTN \to \Phi$ uniformly on $[0,T]^2 \times \mathcal{K}$.

\paragraph{Remarks.}
\begin{enumerate}
    \item The PDE error bound extends the ODE result (\Cref{thm:end-to-end}) by incorporating an additional spatial discretization error term $\varepsilon_{\text{disc}}(N)$. Once the PDE is discretized in space, the resulting finite-dimensional system falls within the ODE setting where the Latent Twin theory applies directly. In particular, when the latent dimension equals the discretized state dimension ($n_z = N$), the autoencoder reconstruction error vanishes ($\varepsilon_{\mathrm{AE}}^{(N)}=0$), so that the bound reduces to spatial discretization plus latent dynamics errors, exactly mirroring the ODE case.

    \item Traditional time-marching schemes for PDEs suffer from error amplification due to spatial discretization that is then propagated exponentially in time, while temporal truncation errors also accumulate exponentially. \emph{By contrast}, Latent Twins avoid temporal error accumulation entirely, with the total error bounded by the sum of spatial discretization and network approximation errors.

    \item The invariance requirement in \Cref{ass:pde-wellposed}, which demands a compact set to be invariant under both forward and backward evolution, can be relaxed in settings where only one temporal direction is of interest. In such cases, invariance under the forward (or backward) evolution semigroup suffices. Moreover, in applications where global invariance is too restrictive, one may instead consider local invariance over finite time horizons. Such relaxations preserve the applicability of the Latent Twin framework while broadening its relevance to practical scenarios where only one-sided or local well-posedness can be guaranteed.
\end{enumerate}

\section{Numerical Experiments} \label{sec:numerics}
A GitHub repository containing illustrative code, example scripts, and instructions for reproducing numerical experiments presented in this work and is available\footnote{The repository will be made fully available upon acceptance of the manuscript.}:
\url{https://github.com/matthiaschung/latent-twins}.

\subsection{ODE Experiments}\label{subsec:odeexpt}

To illustrate the flexibility of the Latent Twin framework, we apply it to four canonical low-dimensional ODE systems spanning periodic, dissipative, and chaotic dynamics: a harmonic oscillator, the SIR epidemiological model, the predator--prey (Lotka--Volterra) equations, and the Lorenz--63 system. These provide a representative benchmark set for testing performance across qualitatively different dynamical regimes. Using standard notation, we consider the following benchmark systems:

\begin{description}
    \item[Harmonic Oscillator:] $x' = v$, $v' = -\omega_0^2 x$ with $\omega_0 = 2$, initial state $(x, v) = (1, 0)$, simulated over $t \in [0, 10]$.
    \item[SIR Model:] $S' = -\beta S I$, $I' = \beta S I - \gamma I$, $R' = \gamma I$ with $\beta = 3/10$, $\gamma = 1/10$, initial state $(S, I, R) = (99/100, 1/100, 0)$, simulated over $t \in [0, 60]$.
    \item[Lotka--Volterra:] $u' = \alpha u - \beta u v$, $v' = \delta u v - \gamma v$ with $\alpha = 1$, $\beta = 1/10$, $\delta = 75/1000$, $\gamma = 3/2$, initial state $(u, v) = (10, 5)$, simulated over $t \in [0, 30]$.
    \item[Lorenz--63:] $x' = \sigma (y - x)$, $y' = x (\rho - z) - y$, $z' = xy - \beta z$ with $\sigma = 10$, $\rho = 28$, $\beta = 8/3$, initial state $(x, y, z) = (1, 1, 1)$, simulated over $t \in [0, 40]$.
\end{description}

All systems are integrated with the default Python \texttt{scipy} ODE solver in high-precision mode ($\mathtt{atol} = 10^{-12}$, $\mathtt{rtol} = 10^{-10}$). The solutions are shown in \Cref{fig:odes}.

In this section our aim is not to tune hyperparameters for peak accuracy, but to demonstrate that Latent Twins perform well under a simple, uniform setup. Since these systems are low-dimensional, we set $z(t) \equiv x(t)$ so that the only learned component is the latent mapping $m$.  All models share the same architecture: a three-layer MLP with input dimension $n_x + 2$ (state plus $t_1$ and $t_2$) and output dimension $n_x$. The encoder has layer widths $16 \to 8 \to n_z$ (with $n_z = 4$) using \texttt{softmax} activations at each layer; the decoder mirrors this structure with a linear output layer. This results in 246, 246, 267, and 267 trainable network parameters for the Harmonic Oscillator, Lotka–Volterra, SIR, and Lorenz–63 systems, respectively.

We generate $J = 2^{15}$ random pairs $\{(t_1^j, x_1^j), (t_2^j, x_2^j)\}_{j=1}^J$ with $t_1$ and $t_2$ drawn uniformly from simulation intervals to ensure diverse temporal separations. Data is normalized by subtracting the mean and dividing by standard deviation, then split into  training/testing with a 80--20 split. Training uses Adam optimizer with learning rate $10^{-3}$ for $\num{1000}$ epochs and batch size 256.

We randomly select a single $(t_1, x_1)$ to predict full trajectories using the Latent Twin $\LT$. As shown in \Cref{fig:odes}, predictions are highly accurate for the harmonic oscillator and SIR model, while the Lotka--Volterra system exhibits some divergence indicating incomplete convergence. The Lorenz--63 system, being chaotic, is inherently challenging but the Latent Twin successfully captures overall trends. Corresponding prediction errors are shown in \Cref{fig:odeerror}.

A key advantage is direct evaluation capability: $\LT(t)$ maps any $(t_1, x_1)$ directly to $x(t)$ for arbitrary $t$, avoiding sequential stepping entirely and enabling trivial parallelization. In contrast, classical schemes (Euler, Runge--Kutta \cite{butcher2008numerical}, multistep methods \cite{hairer1993solving1}) and learning-based predictors \cite{rowley2004model,hochreiter1997long} advance solutions by fixed recursion. Unlike traditional marching methods where errors accumulate exponentially over time, Latent Twin errors remain uniformly distributed across time horizons. To investigate error accumulation, we perform recursive evaluations of $\LT$, analogous to classical time-marching schemes. In this recursive mode, the Latent Twin propagates the state step by step according to
\[
x_{k+1} = \LT[t_k][x_k](t_k+h), \qquad  \text{with constant $h>0$ and } t_{k+1} = t_k+h \qquad \text{for } k = 0, \ldots .
\]
As with classical methods, this approach is inherently sequential and accumulates stepwise error. We observe that MSE increases slightly when using the recursive form compared to direct evaluation. Combined-state MSE comparisons (direct vs. recursive): Harmonic oscillator ($4.64 \times 10^{-4}$ vs. $5.13 \times 10^{-4}$), SIR ($5.76 \times 10^{-8}$ vs. $6.45 \times 10^{-8}$), Predator--prey ($7.02$ vs. $8.80$), and Lorenz--63 ($5.38 \times 10^{1}$ for both), reflects global rather than local dynamics capture.

\begin{figure}
    \centering
    \begin{tabular}{cc}
        \includegraphics[width=0.4\linewidth]{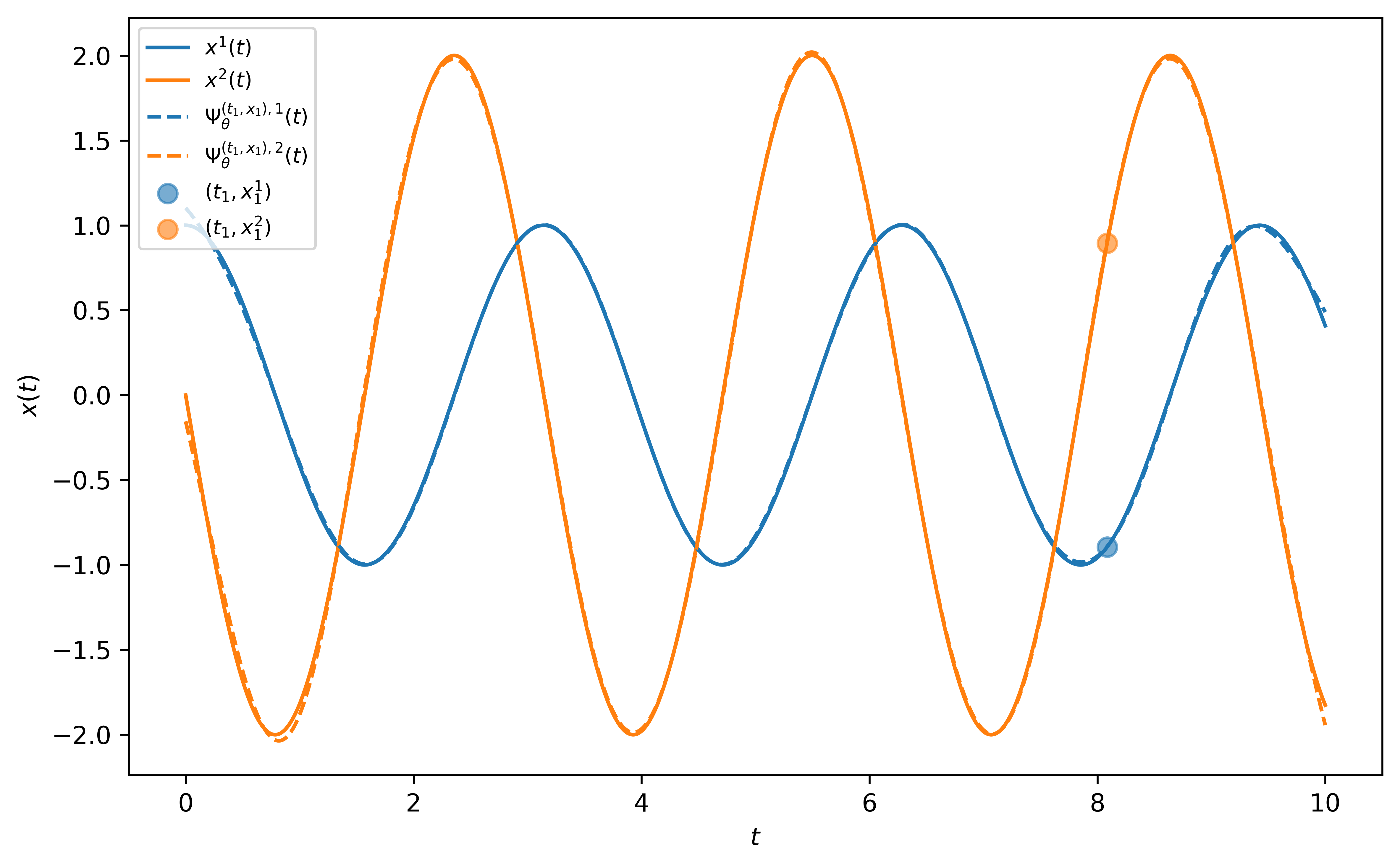} &
        \includegraphics[width=0.4\linewidth]{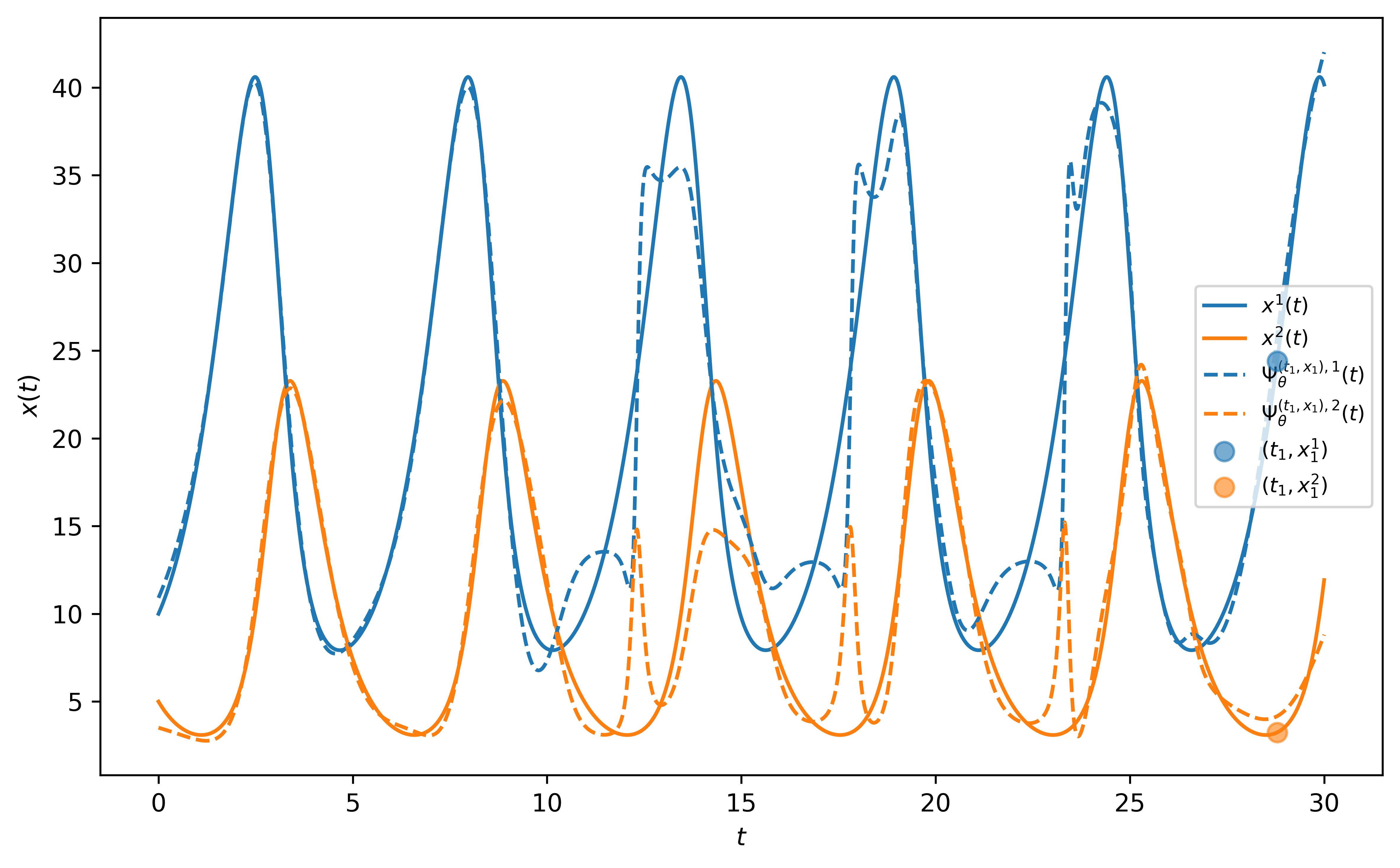}\\
        \includegraphics[width=0.4\linewidth]{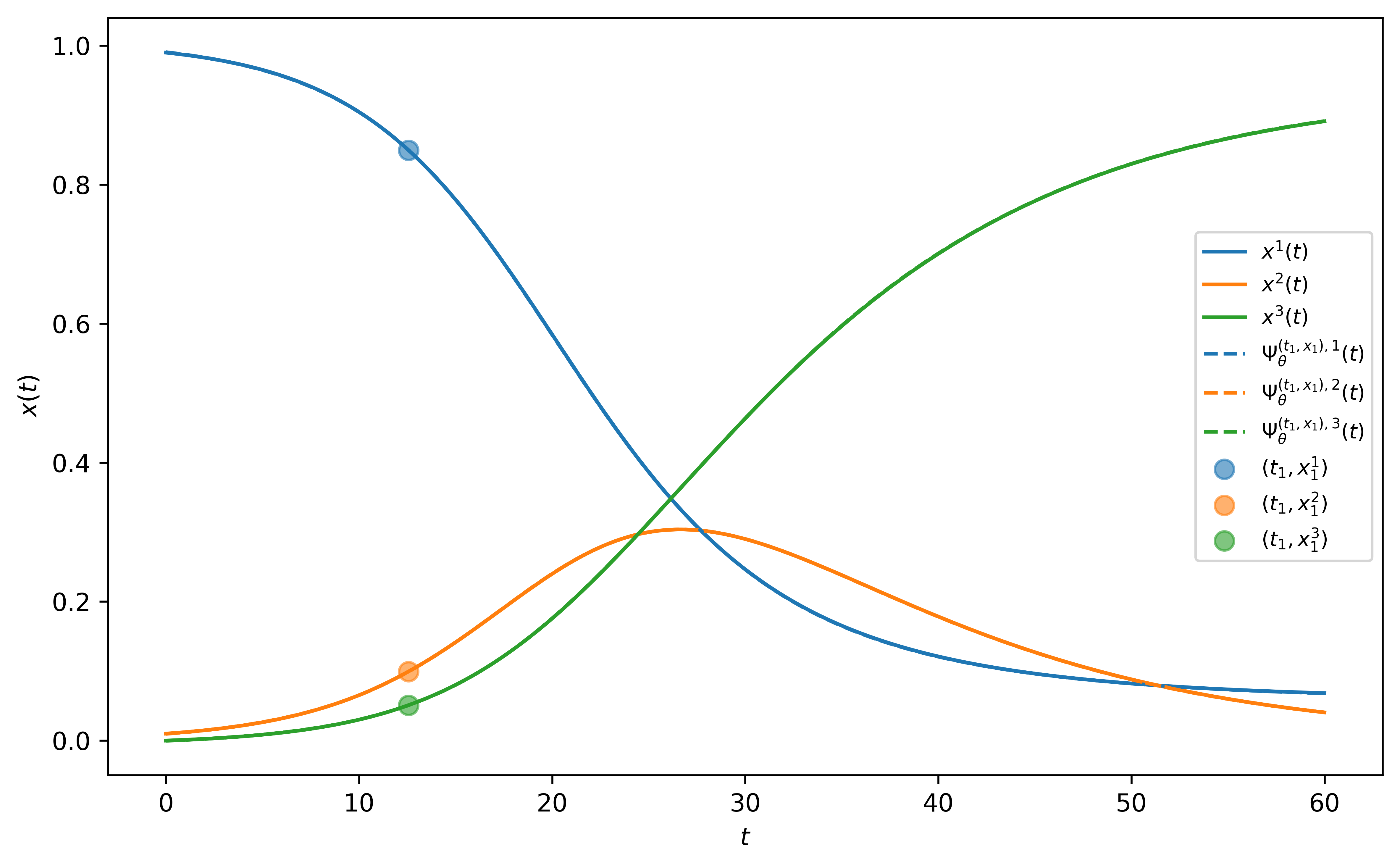} & 
        \includegraphics[width=0.4\linewidth]{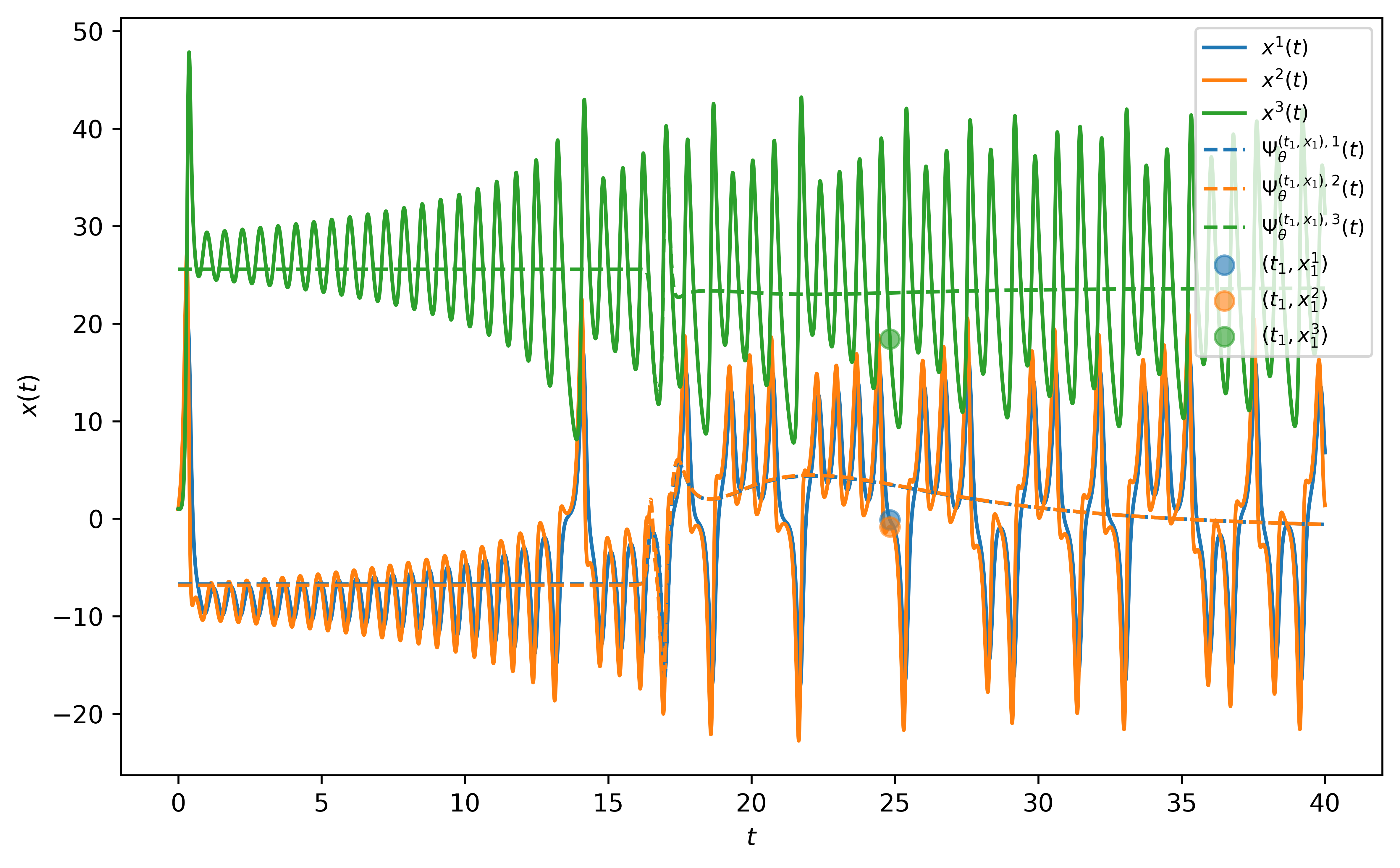}
    \end{tabular}
    \caption{The trajectories for the four ODE systems---harmonic oscillator, predator--prey, SIR, and Lorenz--63 (arranged top-left to bottom-right in reading order)---are shown, with superscript indices denoting the associated state variables.}
    \label{fig:odes}
\end{figure}

\begin{figure}
    \centering
    \begin{tabular}{cc}
        \includegraphics[width=0.5\linewidth]{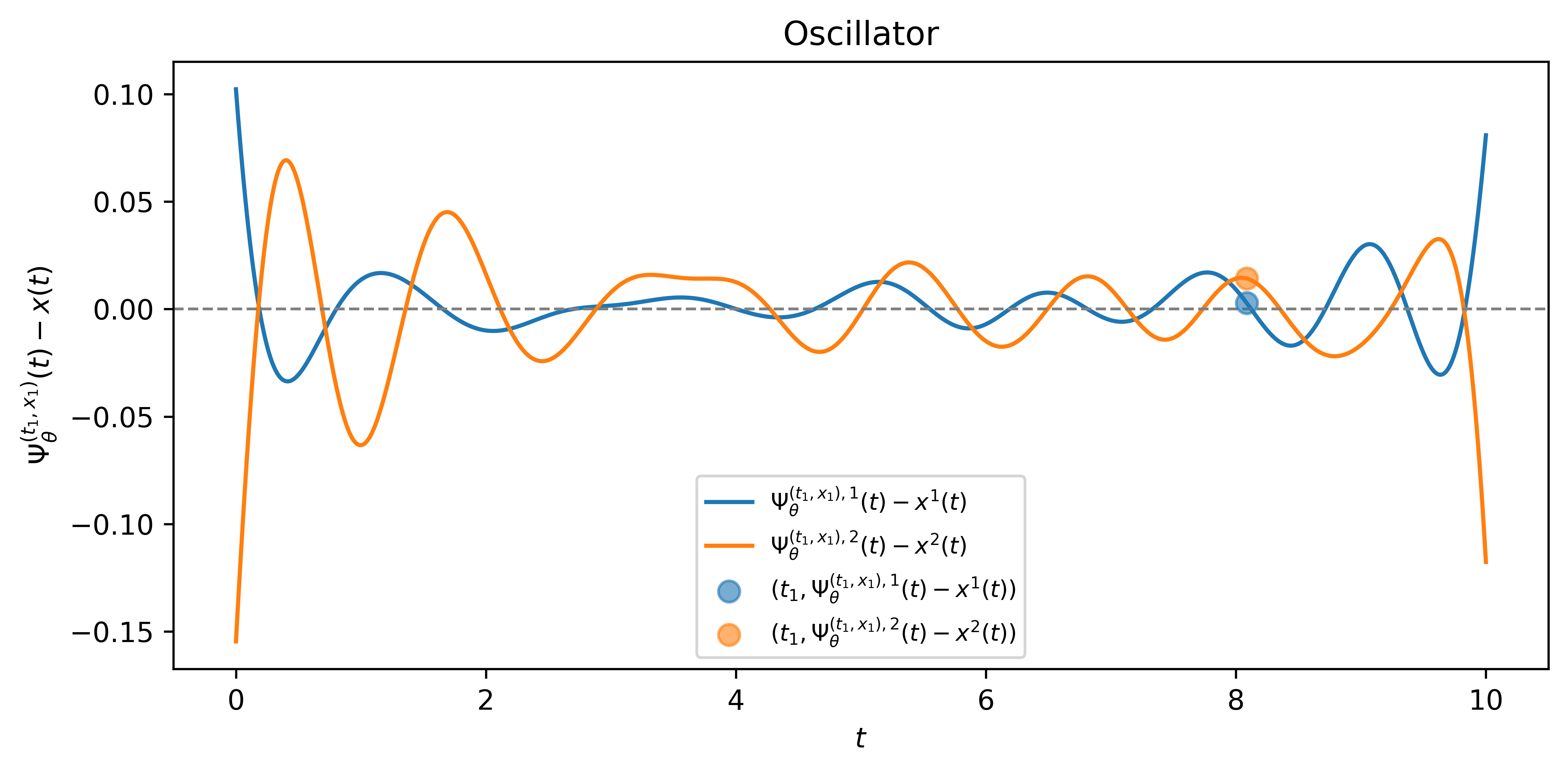} &
        \includegraphics[width=0.5\linewidth]{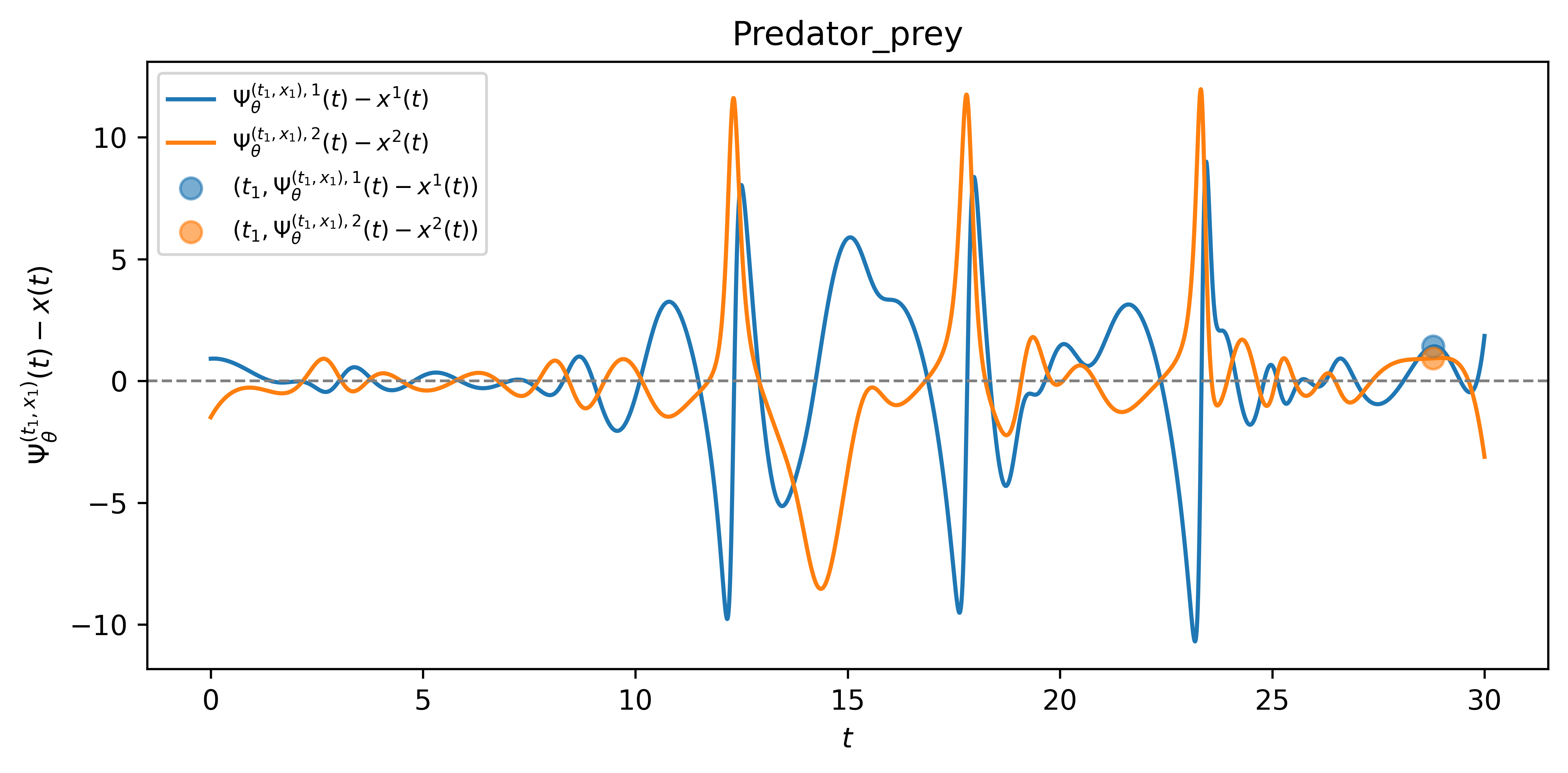}\\
        \includegraphics[width=0.5\linewidth]{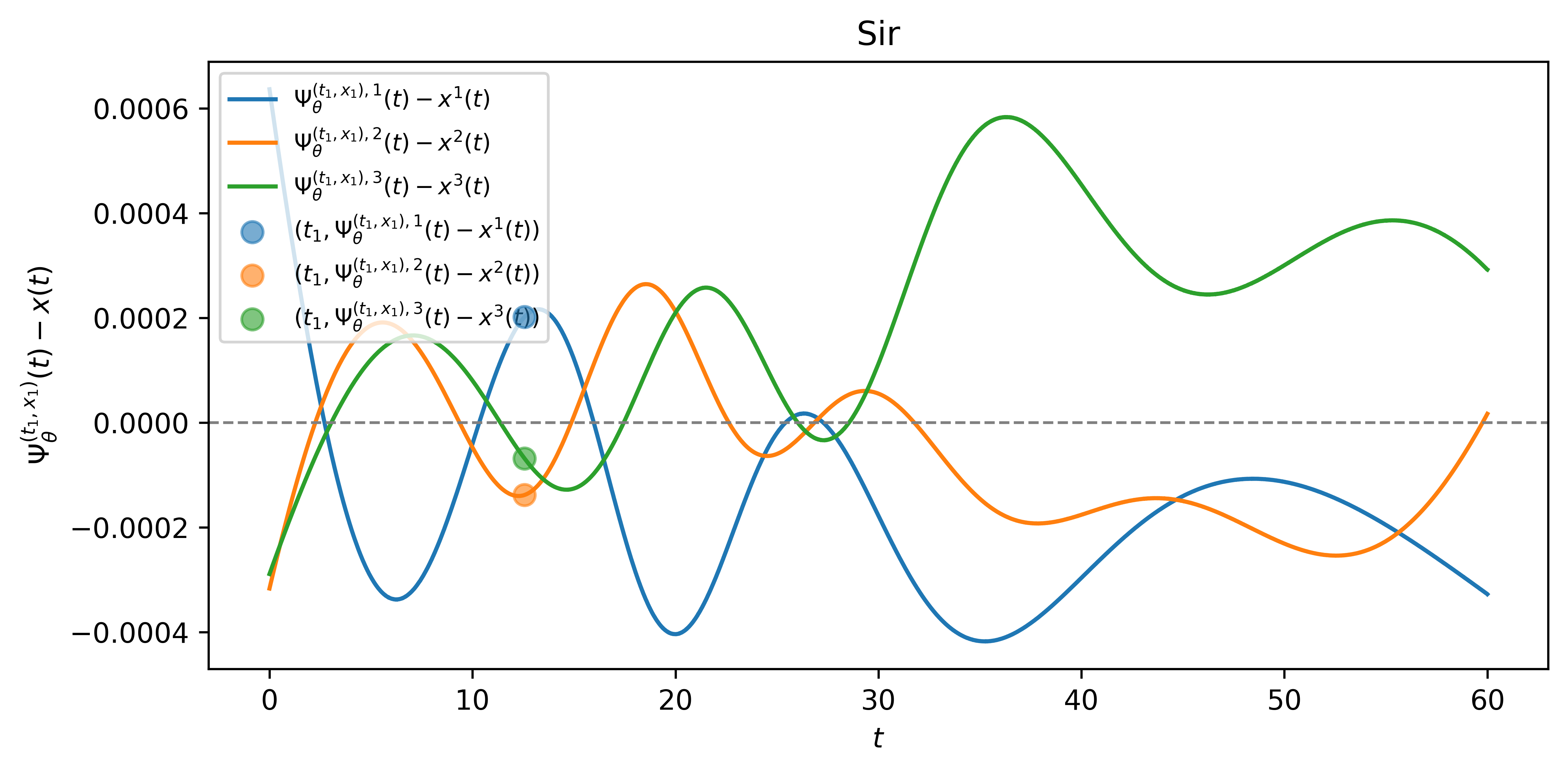} & 
        \includegraphics[width=0.5\linewidth]{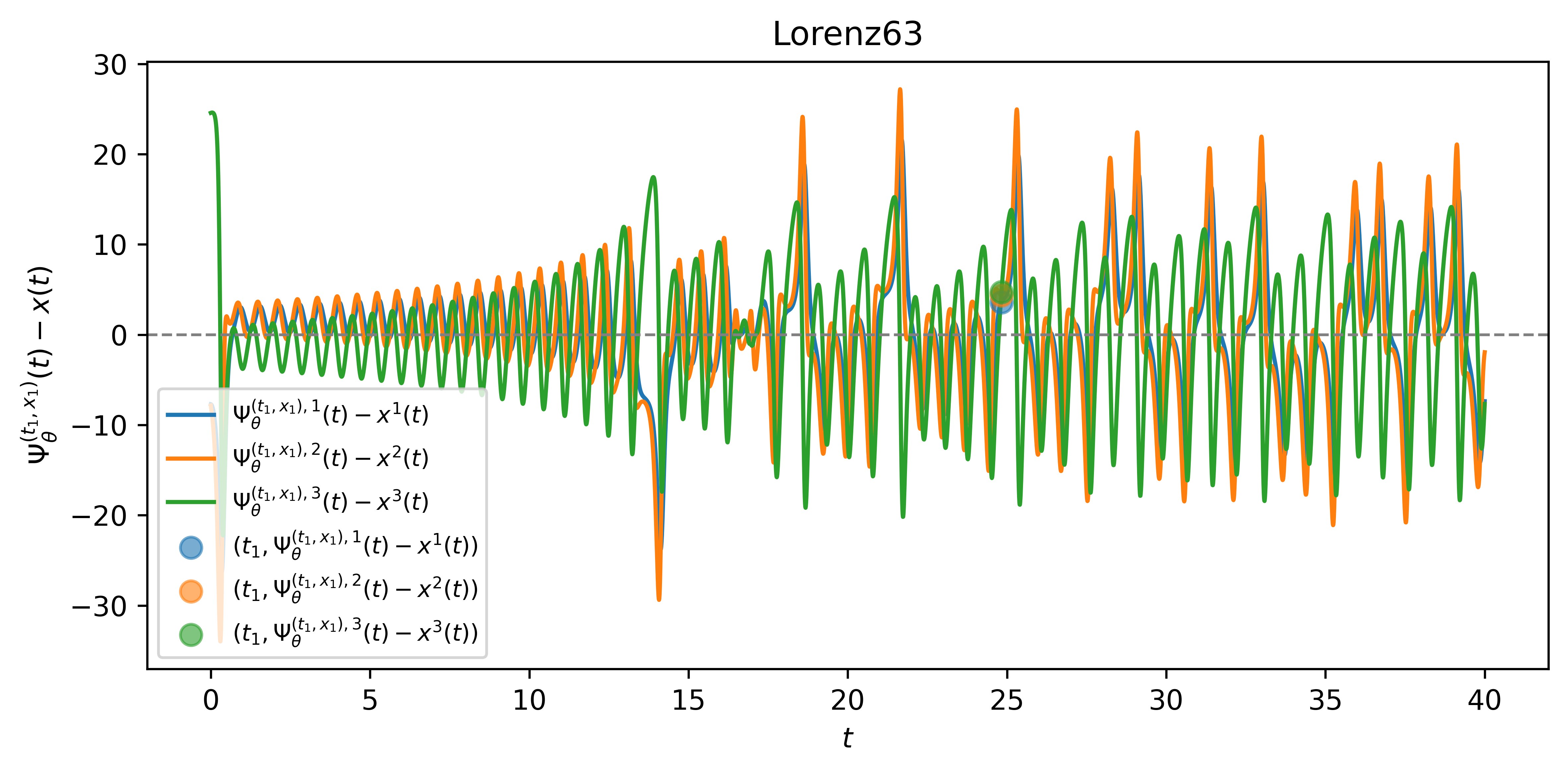}
    \end{tabular}
    \caption{
  Prediction errors for the Latent Twin model. Superscript indices label the state variables to which the errors correspond.
    }
    \label{fig:odeerror}
\end{figure}

Chaotic dynamics are fundamentally difficult to approximate with purely global models. For Lorenz-63, sensitive dependence causes trajectories to diverge quickly, so global Latent Twins trained on arbitrarily distant time pairs capture only coarse trends rather than locally accurate dynamics. To improve local fidelity, we restrict training data to temporally close pairs ($|t_2 - t_1| \leq 0.4$). This biases the Latent Twin toward short-term dynamics rather than global evolution. As shown in \Cref{fig:localLorenz}, reconstructions are considerably more accurate near the initial state with visibly reduced local error. However, this gain comes at the expense of global coverage: the network no longer reproduces full attractor geometry but focuses on faithful short-term evolution.

\begin{figure}
    \centering
    \includegraphics[width=0.48\linewidth]{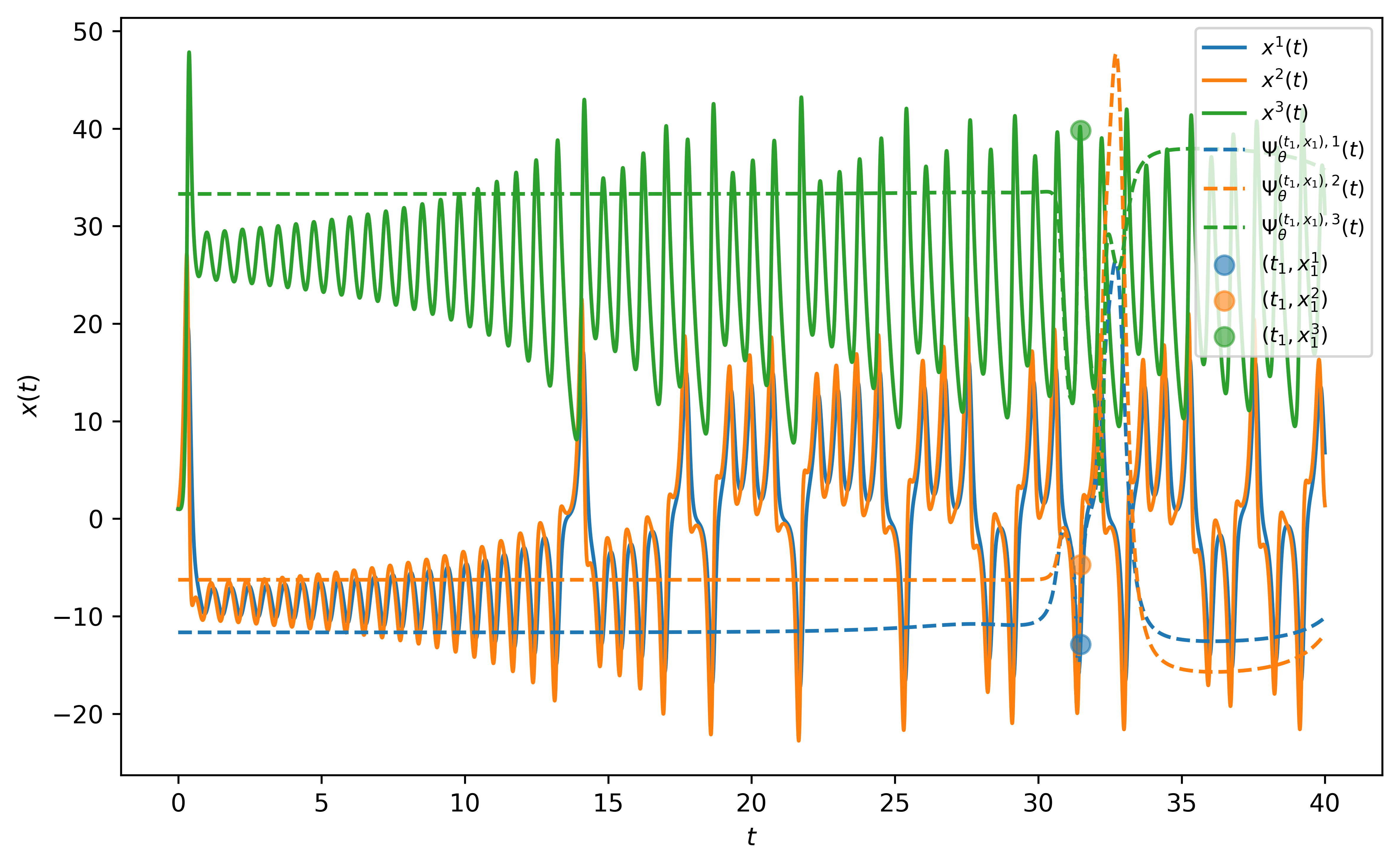}
    \includegraphics[width=0.48\linewidth]{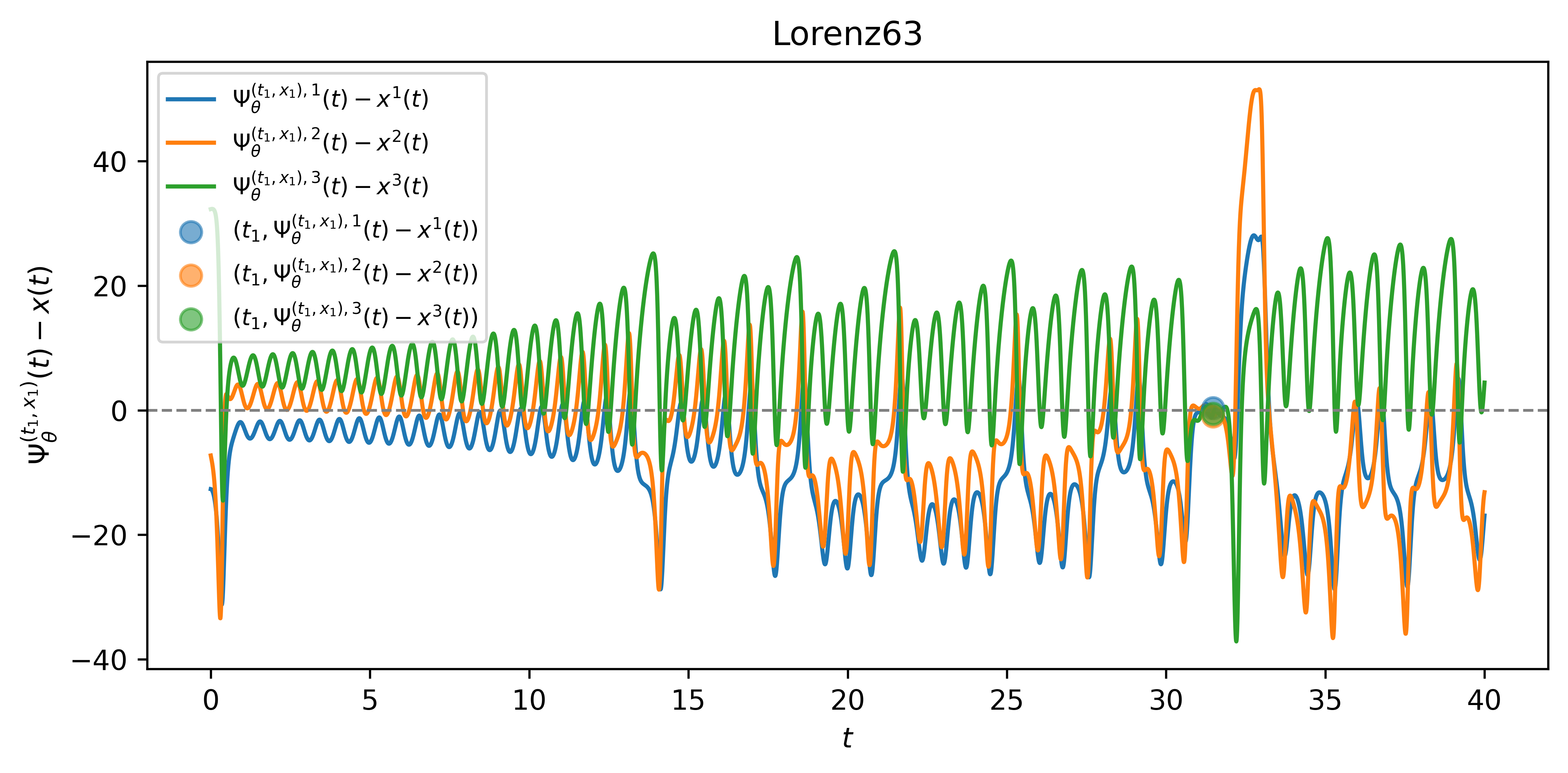}    
    \caption{
    Lorenz--63 predictions when training is restricted to temporally close pairs ($|t_2 - t_1| \leq 0.4$). 
    Left: reconstructed trajectory compared to ground truth. 
    Right: corresponding prediction errors, showing improved local accuracy near the initial state but reduced ability to capture global attractor dynamics.
    }
    \label{fig:localLorenz}
\end{figure}

\emph{Comparison with LSTMs.} Recurrent sequence models such as LSTMs \cite{hochreiter1997long} are widely used for learning temporal dynamics and can achieve strong predictive performance when provided with suitable capacity and data. We focus on LSTMs rather than simpler RNNs \cite{elman1990finding}, as they represent the stronger and more widely adopted standard for sequential modeling.  

In our experiments, we train compact LSTM baselines in a many-to-one (10-to-1) configuration, mapping $(x(t_{k-9}),\dots,x(t_k)) \mapsto x(t_{k+1})$. This setup enables the LSTM to infer dynamics directly from state transitions without explicit time encoding, and recursive application allows for forecasting over longer horizons. To ensure comparability, we restrict the models to the same low-complexity regime as the Latent Twin. Specifically, the LSTMs use 582 parameters for 2D systems and 855 parameters for 3D systems—slightly more than the 246 and 267 parameters of the corresponding Latent Twins. Training data is generated on the same time intervals as in the Latent Twin experiments, with a reduced $\delta t$ so that a $10$-step history captures sufficient dynamical variation. The resulting models achieve good short-term accuracy, particularly in oscillatory systems, the harmonic oscillator and Lotka–Volterra. Over longer horizons, recursive prediction introduces phase drift and error growth, which becomes most pronounced in chaotic systems, Lorenz--63. We present the details and visual comparisons in  \Cref{app:lstm,fig:lstm-dynamics,fig:lstm-errors}.

This comparison is illustrative rather than competitive. LSTMs and Latent Twins pursue different goals: LSTMs provide strong baselines for sequential prediction, while Latent Twins approximate temporal operators in a single evaluation. Together, these perspectives highlight complementary approaches to learning dynamics.  

\emph{Physics-Aware Network Design.}  The experiments above use uniform, uninformed network design across all systems. However, our theoretical analysis in \Cref{sub:odes} shows that one may embed structure directly into Latent Twins for linear ODEs. The construction in \Cref{eq:ltODE} mirrors the exact exponential flow, providing a physics-inspired latent map.

For the harmonic oscillator, we implement the Latent Twin with identity encoder/decoder and structured latent map $m(z,t_1,t_2) = \exp((t_2-t_1)W)z$, embedding the exponential solution operator directly. Importantly, for this experiment, we train the Latent Twin on the denormalized data. This is simply to allow the learned generator to match the true system matrix. The network parameters reduce to $W$, and converge faster than the setups above. 
The learned generator converges to the true system matrix with Frobenius error $\|W-M\|_{\mathrm{F}} \leq 9.5 \times 10^{-5}$, yielding trajectory errors below $4.3 \times 10^{-5}$ and visually indistinguishable results (\Cref{fig:oscillator-err}).

\begin{figure}[t]
  \centering
  \includegraphics[width=0.48\textwidth]{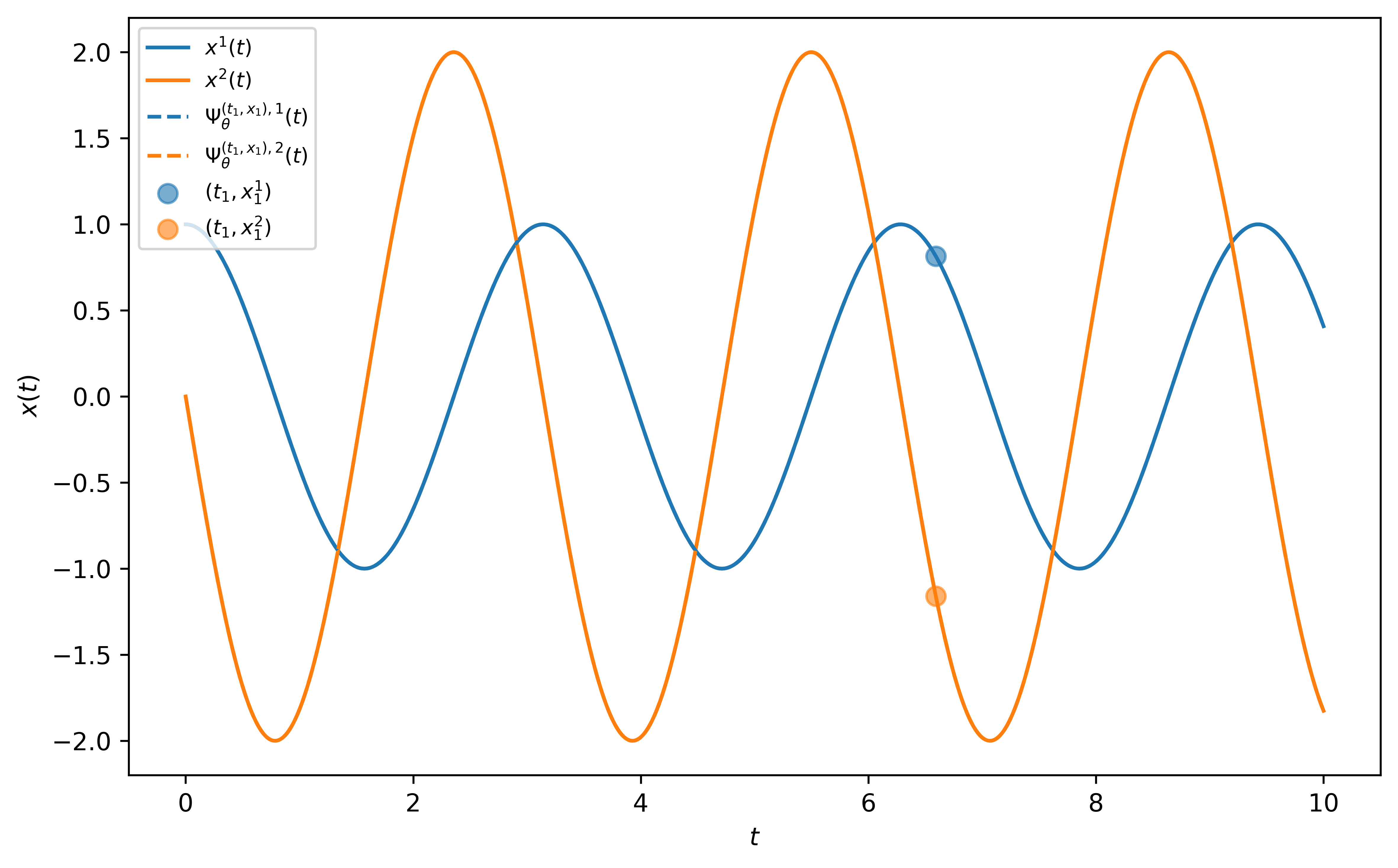}\hfill
  \includegraphics[width=0.48\textwidth]{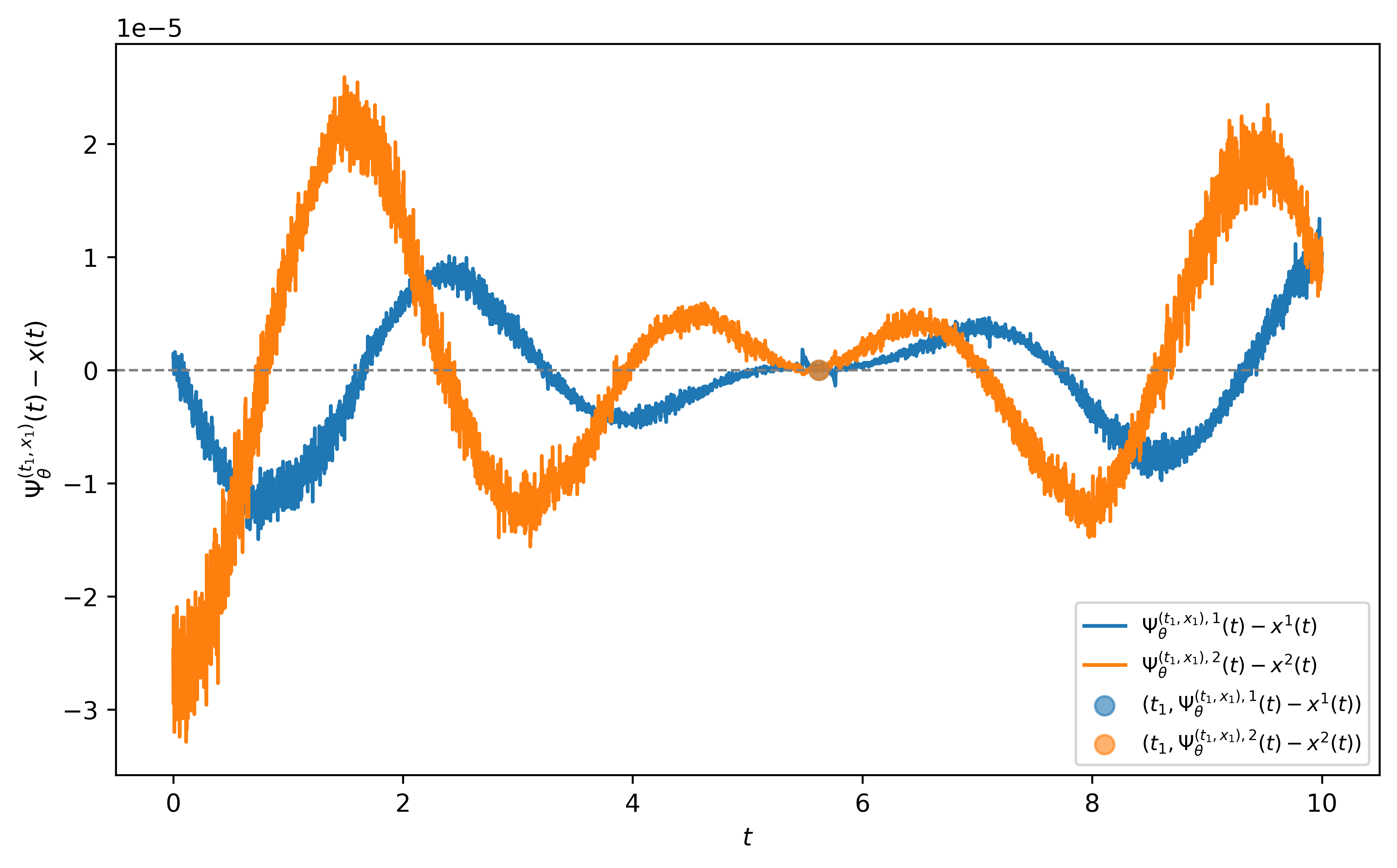}
  \caption{Harmonic oscillator with structured latent map $m(z,h)=\exp(hW)z$. Left: overlapping trajectories. Right: near machine-precision errors.}\label{fig:oscillator-err}
\end{figure}

For nonlinear Lorenz--63, we train a local Latent Twin with generator $W(t_1,x_1,t_2)$ such that $x_2 \approx \exp^{(t_2-t_1)W(t_1,x_1,t_2)}x_1$. Here $W$ is parameterized by a feedforward network with two hidden Tanh layers. As shown in \Cref{fig:lorenz-local}, this captures short-horizon dynamics well but errors grow with $|t_2-t_1|$, reflecting the impossibility of global linear surrogates for chaotic behavior. This construction can also be interpreted through the lens of \emph{hypernetworks}~\cite{Ha2017HyperNetworks}, in which the latent time-step operator is dynamically generated by a compact network rather than stored explicitly.

\begin{figure}[t]
  \centering
  \includegraphics[width=0.62\textwidth]{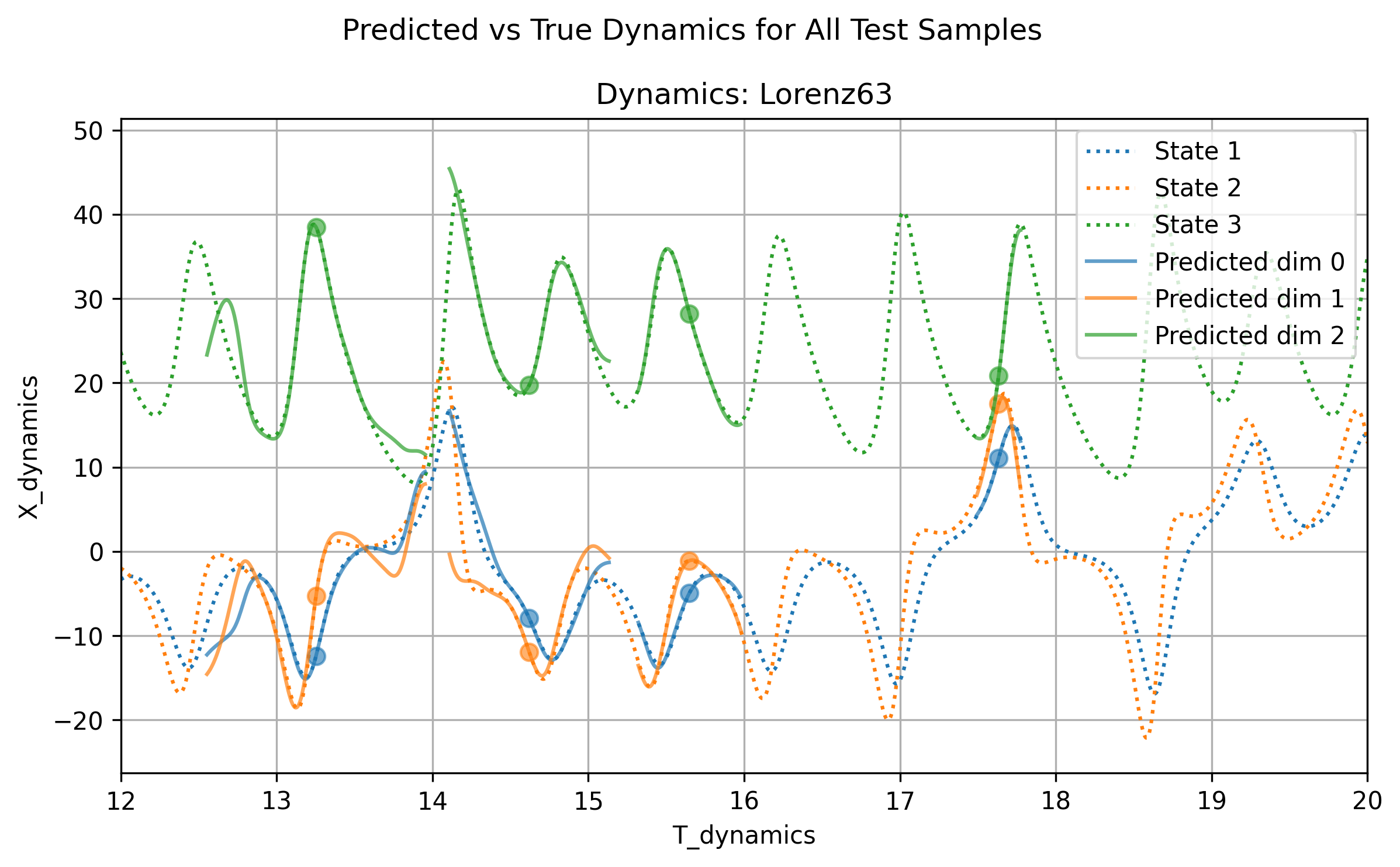}
  \caption{Lorenz--63 local Latent Twin $x_2 \approx \exp((t_2-t_1)W(t_1,x_1,t_2))x_1$. Accurate locally but errors grow with horizon, consistent with chaotic sensitivity.}
  \label{fig:lorenz-local}
\end{figure}

Together, these experiments demonstrate the power of theory-aware architectures: exact recovery for linear systems and effective local linearization for chaotic dynamics, showing how architectural design fundamentally shapes approximation quality in the Latent Twin framework.

\subsection{PDE Experiment: Shallow Water Equations}\label{sub:swe}

The shallow water equations (SWE) provide a classical model for geophysical fluid dynamics in regimes where horizontal scales dominate the vertical. They capture essential phenomena such as wave propagation, tides, and large-scale circulation, and thus serve as an indispensable testbed for surrogate modeling approaches \cite{vreugdenhil2013numerical,temam2024navierstokes,vallis2017atmospheric}. Despite their relative simplicity compared to the full Navier-Stokes equations, SWE simulations remain computationally intensive at high resolution, motivating the search for efficient data-driven surrogates.

In this experiment, we investigate the application of the Latent Twin framework to the SWE, with the goal of learning low-dimensional surrogates that capture the system’s temporal evolution. In this section, we present a streamlined, high-level overview of our numerical investigations to highlight the central story. Full details of the experimental setup—including PDE formulation, parameter choices, boundary conditions, discretization scheme, data generation protocol, network architecture, and learning regime—are provided in \Cref{app:swe-details} to ensure completeness and reproducibility.

The governing equations of the SWE describe the coupled dynamics of surface elevation $\eta$ and horizontal velocities $(u,v)$. We adopt a geophysical parameter regime, with mean depth $H=100$~m, gravitational acceleration $g=9.81$~m/s$^2$, and Coriolis forcing under a $\beta$-plane approximation. Reflective boundaries prevent inflow and outflow. To generate training and testing data, we initialize each simulation with a \emph{randomized Gaussian perturbation} of the free surface elevation, while setting the velocity field to zero. These stochastic perturbations, with centers drawn uniformly across the domain, produce a wide variety of wave patterns and flow regimes, ensuring that the resulting dataset captures a diverse range of dynamical behaviors.

To numerically solve the SWE, we build on a publicly available Python implementation \cite{braendshoi2019shallow}, which provides a structured finite-difference solver on Cartesian grids. This ensures that our experiments are reproducible and based on a widely accessible code base. The model is discretized on a $64 \times 64$ grid with a stable time step $\Delta t \approx 51$~s, and trajectories are simulated up to $T \approx 3 \times 10^4$~s. From these simulations, we extract random pairs of states $(x(t_1), x(t_2))$, resulting in a dataset of $J=32{,}768$ paired time-state samples. Each state is a tensor of size $3 \times 64 \times 64$ representing $(\eta,u,v)$ at the given time. \Cref{fig:swe_states} presents three representative testing states (not seen by the Latent Twin), where only the surface height is shown while the velocity field is omitted for compactness. For neural network learning, data are standardized to zero mean and unit variance.

\begin{figure}
    \centering
  \includegraphics[width=1\linewidth,
    trim=0 375 0 0,   
    clip]{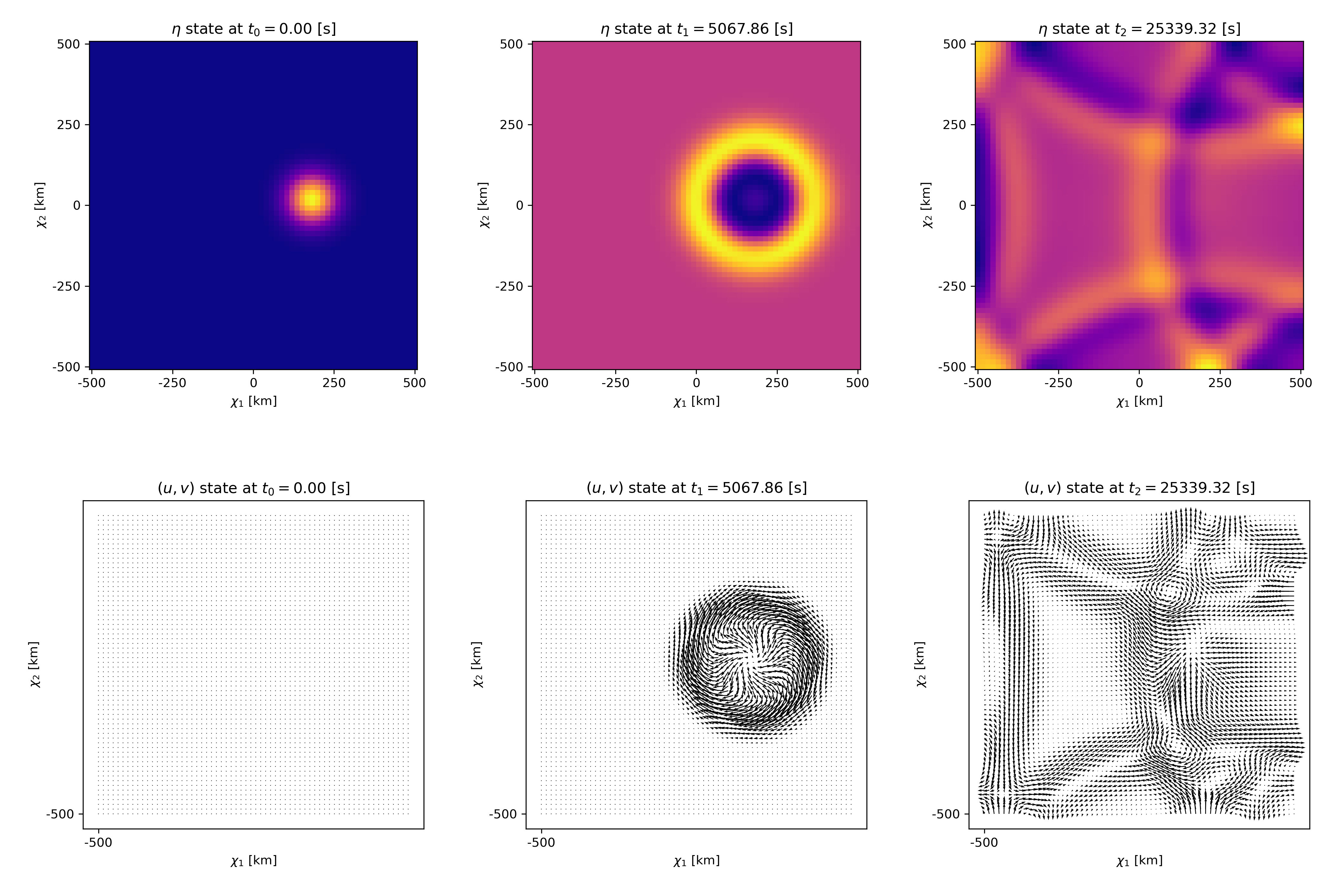}%
    \caption{\emph{Sample data of surface elevations $\eta$:} The leftmost image shows the randomly selected initial condition height, generated according to \Cref{eq:swe_initial}. The subsequent columns display the surface elevations of the shallow water equations at two representatively chosen time points, $t_1$ (middle) and $t_2$ (right). 
    }
    \label{fig:swe_states}
\end{figure}

The Latent Twin network architecture consists of a simple multilayer perceptron autoencoder that compresses each state into a latent vector of dimension 128, and a latent mapping  $m$ that advances latent states across arbitrary time pairs. The architecture is intentionally simple to emphasize the framework rather than the network design—three layers with ReLU activations for the autoencoder, and a single affine linear layer for the latent map. Training follows the empirical risk formulation \Cref{eq:empirical}, balancing reconstruction loss and temporal prediction loss. Optimization is performed using Adam for 1,000 epochs. The 
reconstructions $\tilde x$ are shown in \Cref{fig:swe_autoencoded}.

\begin{figure}
    \centering
  \includegraphics[width=1\linewidth,
    trim=0 375 0 0,   
    clip]{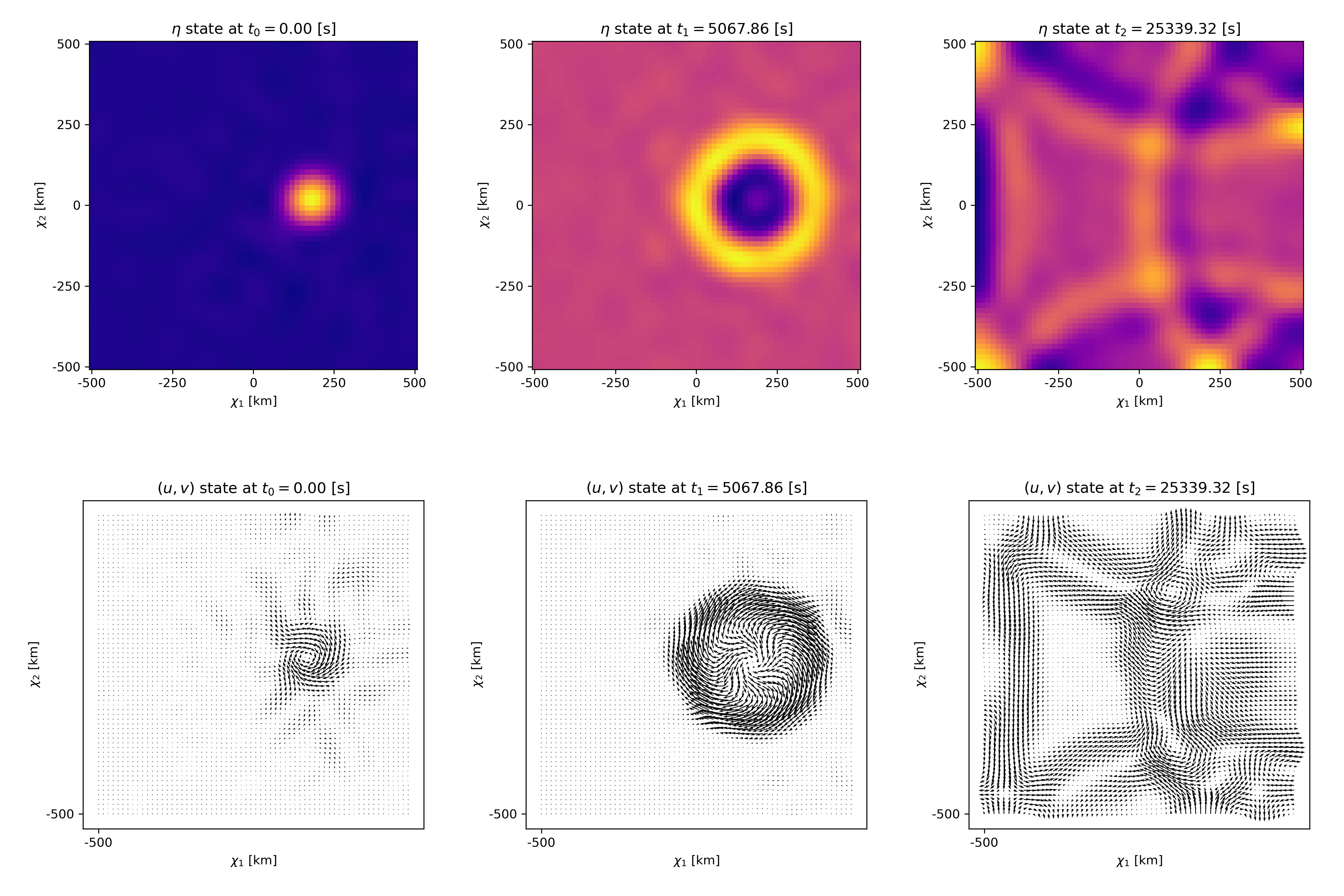}%
    \caption{\emph{Autoencoded data:} The three images display the autoencoded reconstructions $(d\circ e) (x(t))$ of the SWE' surface elevations $\eta$ at three time points $t_0$ (left), $t_1$ (middle) and $t_2$ (right)}
    \label{fig:swe_autoencoded}
\end{figure}

The trained Latent Twin achieves strong performance across the test set. Representative Latent Twin predictions are shown in \Cref{fig:swe_pred1}, where the third and first image correspond to forward and backward in time predictions at times $t_0$ and $t_2$, respectively, given a randomly chosen state from the test set at time $t_1$. The model attains relative reconstruction errors of approximately $2.3 \times 10^{-2}$ and prediction errors of $4.3 \times 10^{-2}$, demonstrating that the learned latent representation successfully captures both spatial structure and temporal dynamics.

\begin{figure}
    \centering
  \includegraphics[width=1\linewidth,
    trim=0 375 0 0,   
    clip]{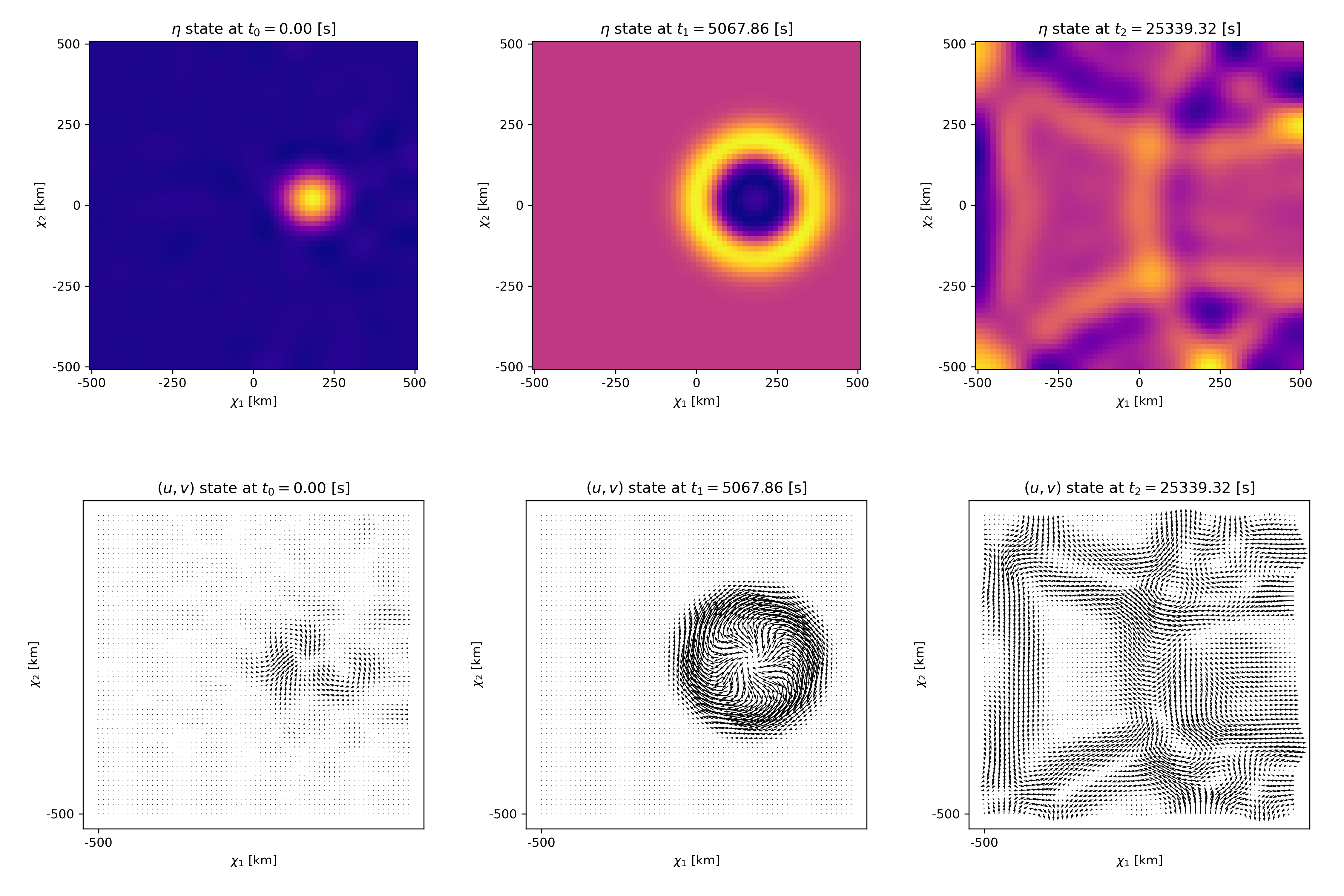}%
    \caption{\emph{Predictions of surface elevations $\eta$:} The middle image shows the current state at $t_1$, while the first and third image displays the Latent Twin predicted state at $t_0$ and $t_2$, resp., that is a forward and backward in time prediction.}
    \label{fig:swe_pred1}
\end{figure}

Alternatively, fully data-driven dynamics can be generated using operator-learning approaches such as DeepONet \cite{lu2021learning} and the Fourier Neural Operator (FNO) \cite{li2020fourier}. These methods typically focus on learning mappings from PDE settings (e.g., boundary or forcing function) directly to the solution field, whereas Latent Twin learns a surrogate solution operator in the latent-space that propagates between states. To explore this connection, we train a DeepONet to learn the operator from the boundary condition $(\eta_0, u_0, v_0)$ to the full solution $(\eta, u, v)$; see \Cref{app:swe-details} for implementation details. Results are reported in \Cref{fig:deepONet}. On the test case, the DeepONet attains a relative reconstruction error at $t_2$ of $6.2\times 10^{-1}$. See \Cref{fig:SWErelErrorDeepOnet} for a comparison of relative errors between DeepONet and the Latent Twin, where the Latent Twin consistently achieves lower errors. In this example, the Latent Twin not only produces visibly sharper and more accurate reconstructions than the DeepONet (compare \Cref{fig:swe_pred12} and \Cref{fig:deepONet}), but also demonstrates the advantage of embedding the dynamics into a latent operator framework that preserves structure. Operator-learning models naturally provide evaluations at arbitrary points in the domain, while Latent Twins here produce reconstructions at a prescribed resolution tailored to their latent representation. Although the present implementation of Latent Twin is mesh-dependent, the framework itself is not restricted to fixed discretizations and can generalize across meshes, for example by incorporating mesh-location inputs into the latent map. More broadly, compared to established operator-learning methods, Latent Twins offer a flexible, theory-backed approach that combines interpretability with strong predictive accuracy, underscoring their promise as a complementary direction for scientific machine learning. We leave the implementation of a mesh-free Latent Twin as well as an exploration of the connection between the Latent Twin and common Neural Operator approaches as a direction for future work.

\begin{figure}
    \centering
    \includegraphics[width=0.9\linewidth]{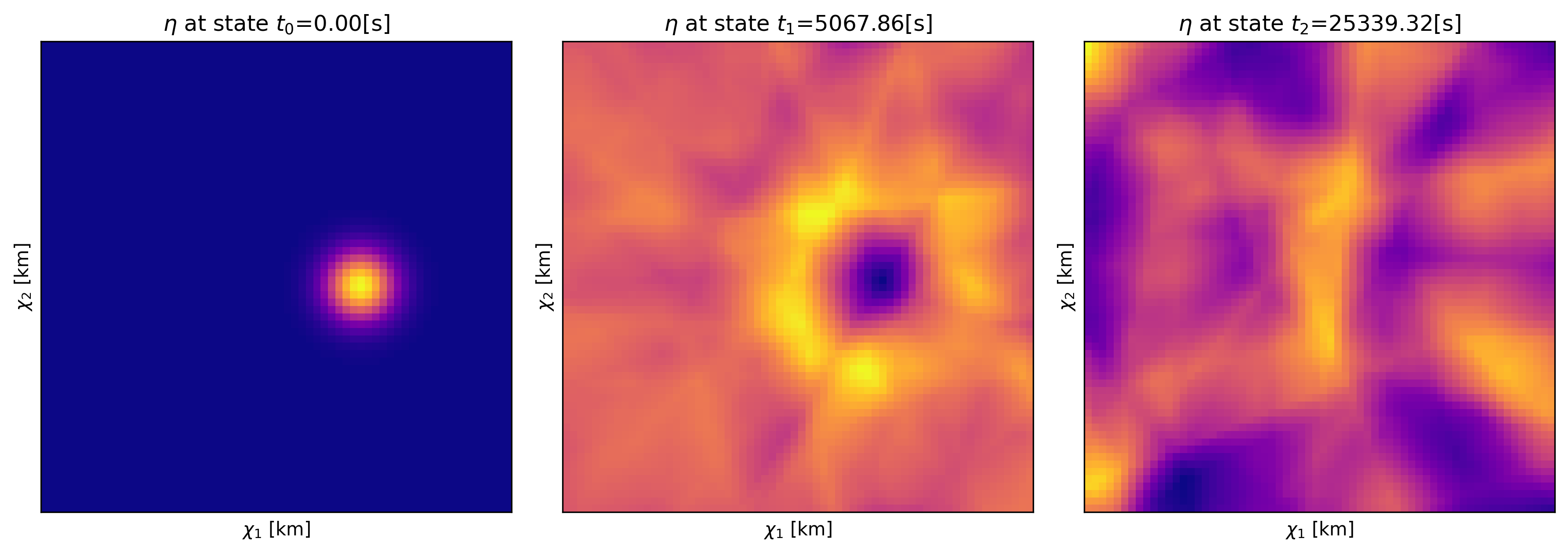}
    \caption{State prediction of surface elevation $\eta$ from initial condition using Deep Operator Network. We plot the initial height $\eta$ on the left, the prediction of $\eta$ at $t_1 = 5067.86$\.s in the center, and the prediction of $\eta$ at $t_2 = 25339.32$\.s on the right}
    \label{fig:deepONet}
\end{figure}

\begin{figure}
    \centering
    \includegraphics[width=0.75\linewidth]{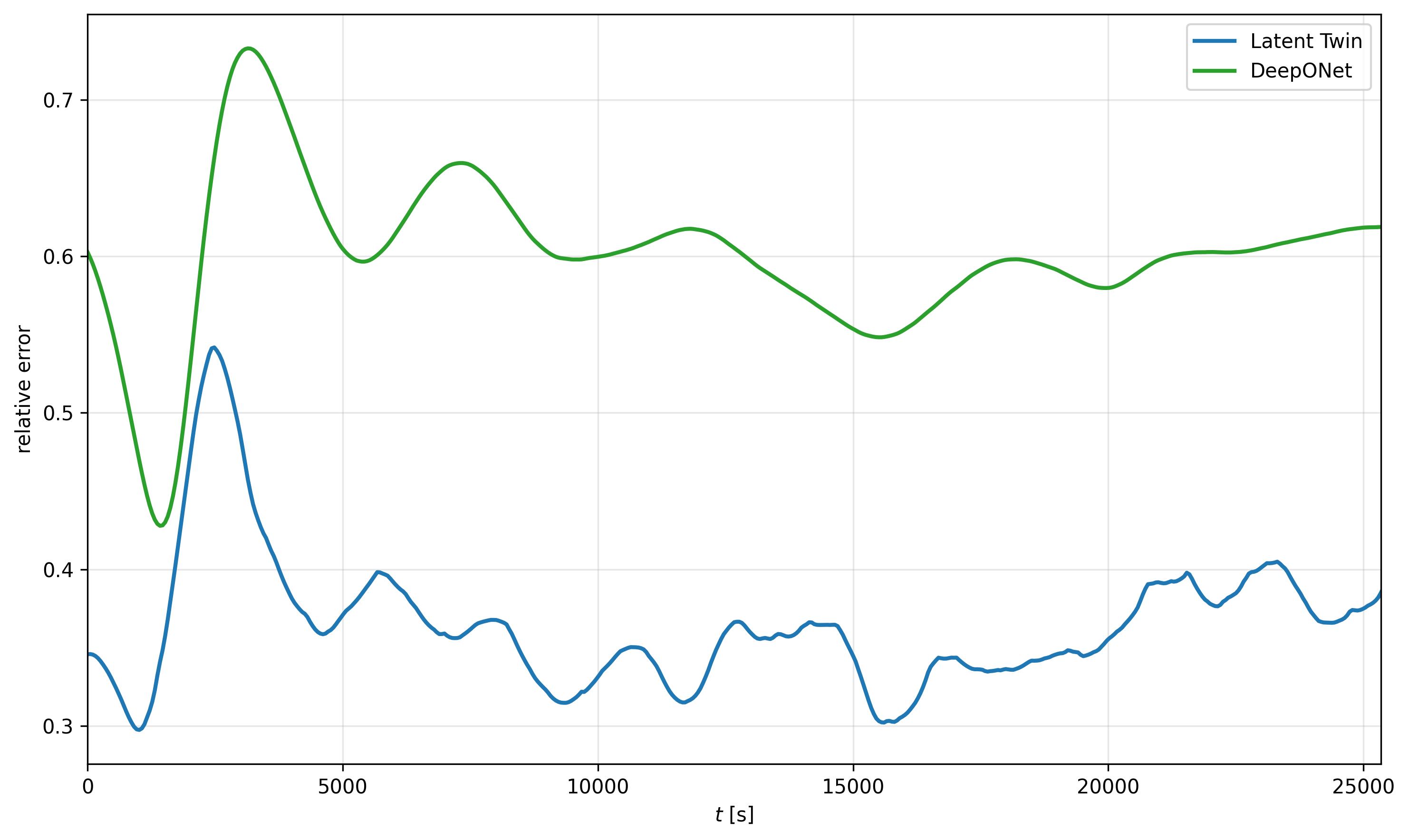}
    \caption{Comparison of relative reconstruction errors of the states across the time interval. The Latent Twin consistently outperforms DeepONet accross all time points.}
    \label{fig:SWErelErrorDeepOnet}
\end{figure}

In realistic scenarios, one rarely has access to ``full'' states. To emulate this, we downsample observations by factors of $2,4,8$ and add Gaussian noise (see \Cref{app:swe-details} for details), see \Cref{fig:swe_noisyobservation}. We then perform latent-space inference: given a noisy observation $y_{\text{obs}}(t_1) = P(x(t_1))+\epsilon$, we solve
\[
    \hat z(t_1) \in \arg\min_{z} \ \| (P \circ d)(z) - y_{\text{obs}}(t_1) \|^2,
\]
to recover a latent representation consistent with the observation. Once $\hat z(t_1)$ is inferred, the Latent Twin can propagate it to predict states at other times. 

\begin{figure}
    \centering
  \includegraphics[width=1\linewidth,
    trim=0 390 0 0,   
    clip]{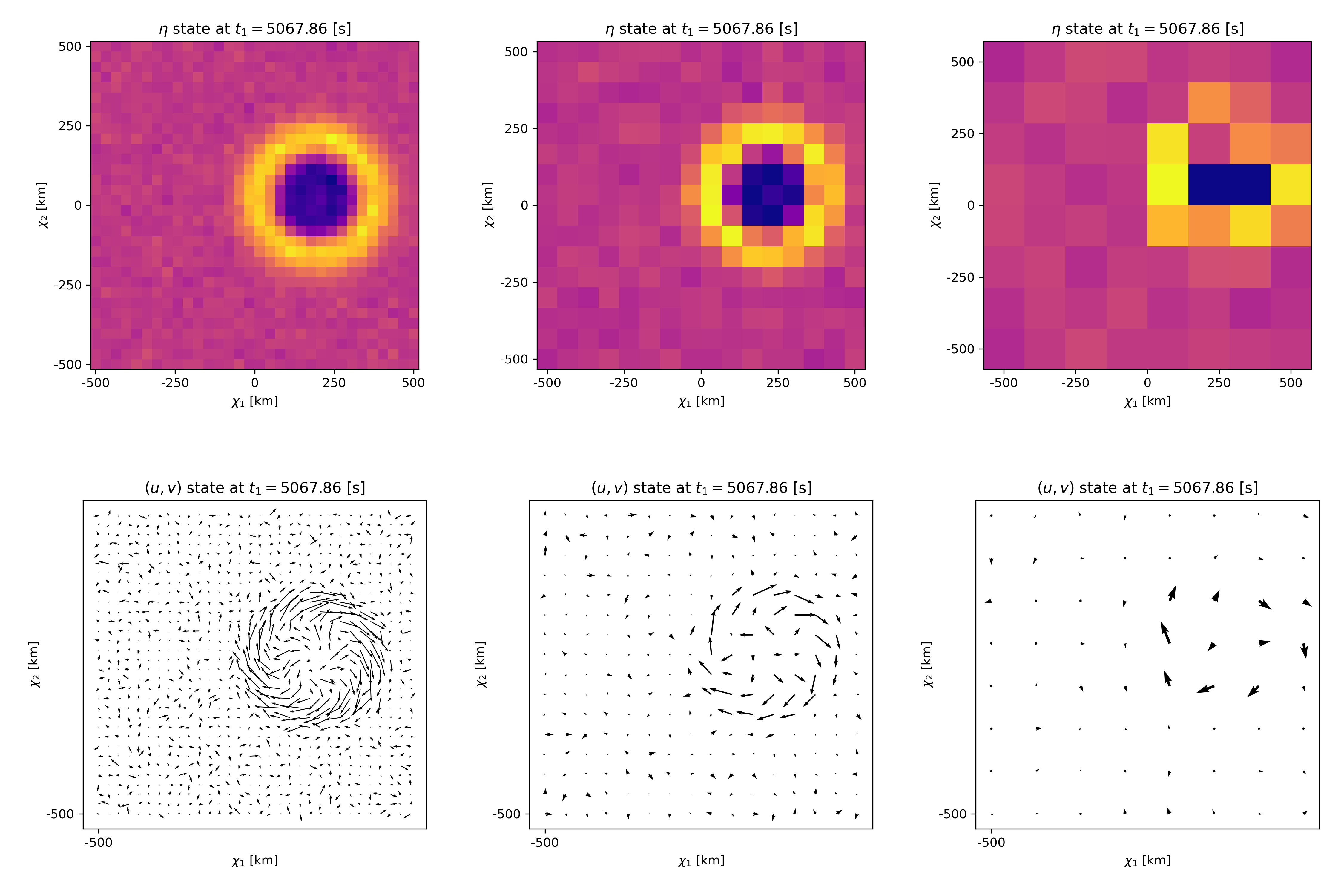}%
    \caption{\emph{Observations:} From left to right, the figure displays varying levels of observational resolution (only surface elevations $\eta$ shown), ranging from $3 \times 32 \times 32$ to $3 \times 16 \times 16$ and down to $3 \times 8 \times 8$ observations. Each observation is corrupted by additive noise of the same noise level.}
    \label{fig:swe_noisyobservation}
\end{figure}

\Cref{fig:swe_noisy} illustrates the results for the most challenging case with $3 \times 8\times 8$ observations. Despite the severe degradation, the framework reconstructs plausible high-resolution states and predicts their evolution forward and backward in time. This demonstrates the strong potential of Latent Twins for data assimilation and forecasting from sparse, noisy measurements.

\medskip
\noindent

The Latent Twin framework thus provides a compact and flexible surrogate for the SWE: it not only reconstructs and forecasts from clean data but also enables inference and prediction from sparse and noisy observations, making it a promising candidate for real-time geophysical forecasting tasks where classical solvers remain computationally prohibitive.

To place these results in the context of established data assimilation (DA) practice, we also compare against the strong-constraint four-dimensional variational method (4D-Var) \cite{LeDimet_1986,asch2016book,evensen2022data}. In our setup, 4D-Var seeks a state estimate $x_2$ that minimizes a quadratic cost functional balancing a background prior and observation misfit over an assimilation window, subject to the SWE dynamics (see \Cref{app:swe-details} for the precise formulation in \Cref{eq:swe-4dvar} and implementation details). We employ the same observation operator, namely the downsampling operator $P$ used in the latent-twin noisy-observation experiments above, and draw observation noise from the identical distribution to ensure a fair comparison. \Cref{fig:swe_4dvar} illustrates a representative case with a single coarse, noisy observation at $t_1=\num{5067.87}$\,s (resolution $3\times 8\times 8$). Starting from a background at $t_1$ (left), 4D-Var produces an analysis at $t_1$ (middle) that recovers the dominant spatial structure consistent with the observation, which is then forecast to $t_2$ (right).

Comparing \Cref{fig:swe_noisy} and \Cref{fig:swe_4dvar}, the 4D-Var analysis at the observation time is reasonable, but its forecast degrades rapidly once outside the assimilation window, whereas the Latent Twin—trained purely from data—reconstructs coherent states from sparse, noisy observations and propagates them forward with higher fidelity, capturing later-time dynamics more accurately. See \Cref{fig:SWErelErrorLatentTwin4Dvar} for a comparison of the relative reconstruction errors of the states $x$ over the considered time interval. Notably, the Latent Twin used here is fully data-driven and evaluated by a lightweight surrogate, while 4D-Var requires both a forward model for propagation and an adjoint for gradient-based estimation, which introduces substantial computational complexity and implementation effort. This contrast reflects different information budgets rather than a direct competition. 

\begin{figure}
    \centering
    \includegraphics[width=0.75\linewidth]{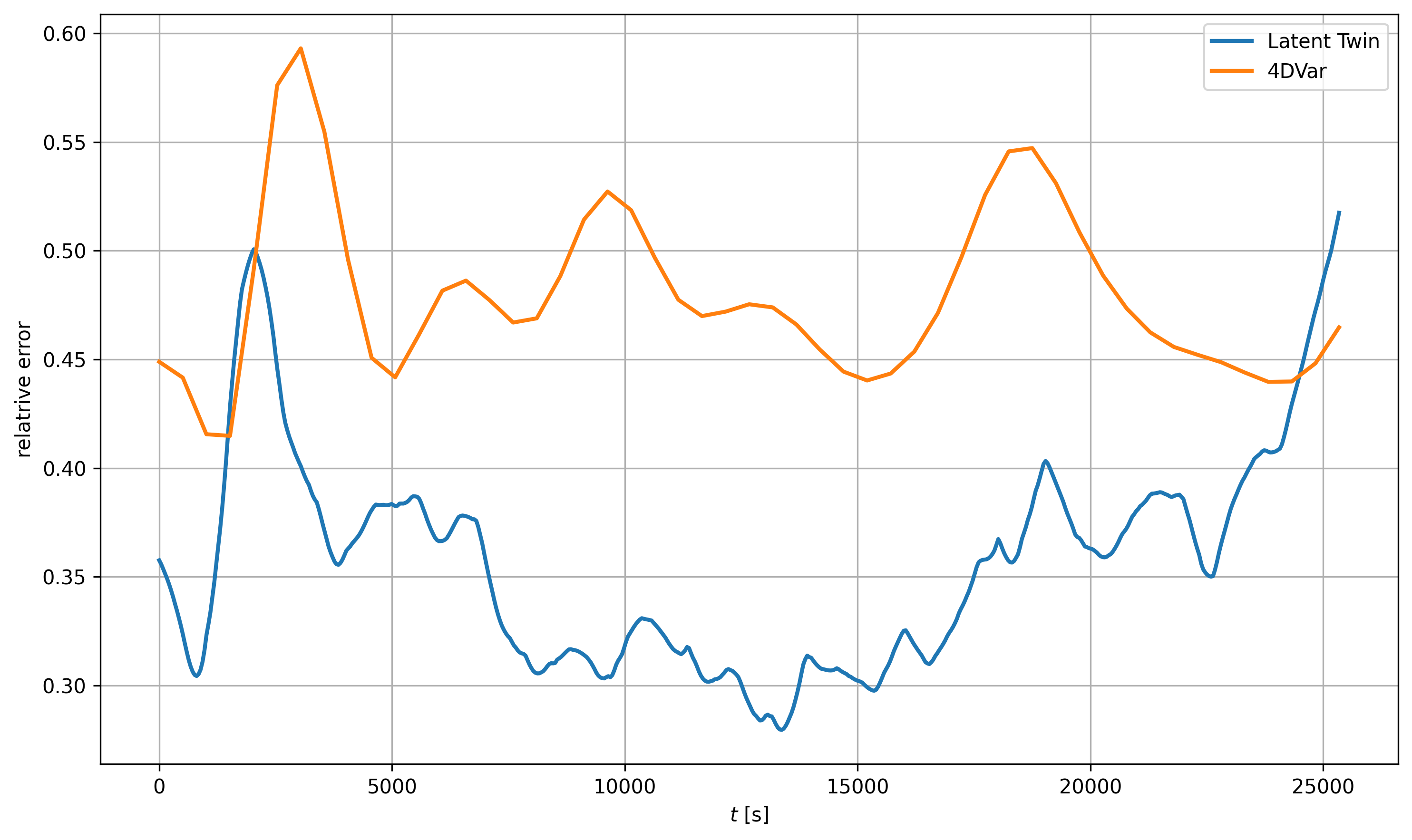}
    \caption{Comparison of relative reconstruction errors of the states across the time interval. The Latent Twin consistently outperforms 4D-Var, with only a few time points where 4D-Var achieves slightly better reconstructions.}
    \label{fig:SWErelErrorLatentTwin4Dvar}
\end{figure}

The Latent Twin is trained offline on a large dataset with time stamps spanning $[0,\num{25339.32}]$\,s and learns a global surrogate of the flow, whereas 4D-Var exploits detailed local model and adjoint information within a single assimilation window. The two methods are therefore complementary: 4D-Var incorporates physics and principled priors through covariance models, while Latent Twins leverage datasets to build fast surrogates that not only reconstruct but also robustly forecast the system beyond the assimilation window.

In summary, Latent Twins demonstrate a stronger ability than 4D-Var to generalize beyond the assimilation window in this experiment, underscoring their promise as compact and flexible surrogates for real-time geophysical prediction.

\begin{figure}
    \centering
  \includegraphics[width=1\linewidth,
    trim=0 375 0 0,   
    clip]{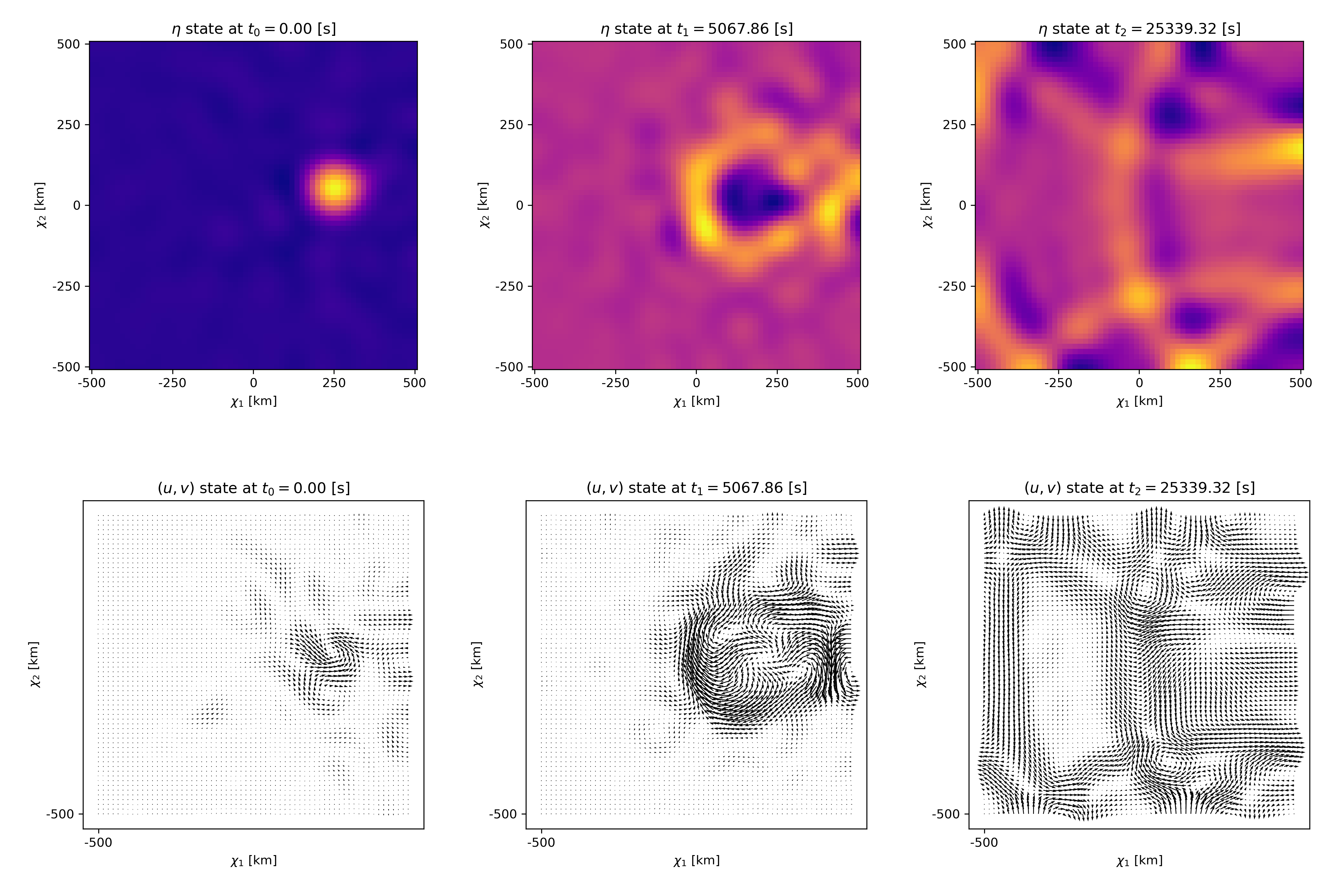}%
    \caption{\emph{State prediction of surface elevation $\eta$ from noisy, downsampled observations.} Middle image: Latent Twin reconstruction at $t_1$ from a coarse $3 \times 8\times 8$ noisy observation at $t_1$ seen in \Cref{fig:swe_noisyobservation}. Left: reconstructed state at $t_0$. Right: predicted state at $t_2$.}
    \label{fig:swe_noisy}
\end{figure}

\begin{figure}
    \centering
    \includegraphics[width=\linewidth]{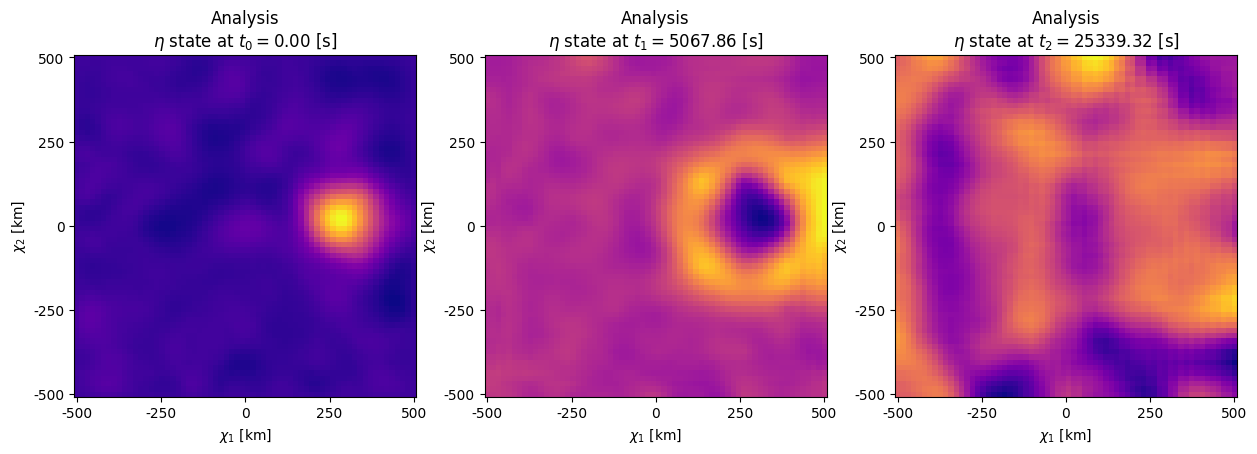}
    \caption{\emph{State prediction of surface elevation $\eta$ from coarse, noisy observations using 4D-Var.} 
    Left: analysis state at $t_0$. 
    Middle: 4D-Var analysis at $t_0$ inferred from a $3 \times 8 \times 8$ noisy observation at $t_1=5067.87$\,s. 
    Right: forecasted state at $t_2$ obtained by propagating the analysis forward with the SWE model.}
    \label{fig:swe_4dvar}
\end{figure}

\subsection{Experiments with Geopotential Height}
The NCEP/NCAR Reanalysis dataset provides a long-term, high-resolution global record of atmospheric conditions by combining observations with numerical weather prediction models. It spans from 1948 to the present and includes variables such as geopotential height, temperature, wind, humidity, and surface pressure \cite{kalnay2018ncep}. Among these, geopotential height is a central diagnostic in atmospheric sciences: it indicates the altitude of a given pressure surface and reflects the large-scale circulation patterns that influence regional and global weather.

Traditional forecasting models for such fields rely on solving the Navier-Stokes equations (often in the form of the primitive equations) within numerical weather prediction dynamical cores \cite{kalnay2018ncep,Muller2015NUMA}. While physically faithful, these solvers become computationally prohibitive at high resolution and long time horizons, especially when ensembles or reanalysis runs are required. This motivates the search for efficient data-driven surrogates. In this context, the Latent Twin framework offers a promising alternative: by learning reduced latent representations and direct evolution operators from historical reanalysis data, it can provide fast and accurate forecasts without the full expense of numerical solvers.

We focus on six-hourly geopotential height fields from the NCEP/NCAR Reanalysis spanning the years 1948--1954. The raw data are provided on a $73 \times 144$ latitude–longitude grid, which we subsample to $64 \times 128$ for computational tractability while retaining large-scale spatial structure. From this period we extract $J=\num{60000}$ paired samples $\{(t_i,x_i),(t_j,x_j)\}$, where $x_i$ denotes the geopotential height field at time $t_i$. To ensure temporal locality while still capturing meaningful dynamics, pairs are restricted by $|i-j|\leq 40$, corresponding to a gap of at most ten days. This setup produces a rich dataset spanning multiple seasons and circulation regimes. Prior to training, both geopotential height fields and their associated time indices are standardized to zero mean and unit variance.

Again, for simplicity we employ a fully connected spatiotemporal autoencoder to compress the high-dimensional geopotential height fields and model their temporal evolution. The encoder $e$ flattens each $64\times 128$ field and maps it to a latent vector of dimension $n_z=128$, capturing the dominant spatial structures. The decoder $d$ mirrors this mapping to reconstruct the original grid from the latent code. Temporal evolution is modeled using the same latent mapping $m$ introduced earlier: given a latent code $z(t_1)$ at time $t_1$, the map produces a transformed latent state $m(z(t_1),t_1,t_2)$, which the decoder then reconstructs as the predicted field at $t_2$. In this way, the model again combines spatial compression with temporal propagation in latent space.

Training is performed for \num{1000} epochs using the Adam optimizer with an initial learning rate of $10^{-3}$ and batch size 32. A \texttt{ReduceLROnPlateau} scheduler decreases the learning rate by a factor of 0.7 if the validation loss stagnates for 10 epochs. Of the $J=\num{60000}$ samples, 80\% are used for training and 20\% for testing. 

\Cref{fig:geopotential_dynamics1} provide comparisons between the reconstruction, forecast, and 4D-Var solution against the true values. In this figure, the first row illustrates the true data, the middle row displays the reconstruction using the latent-twin, and the third row presents both the forecast and the 4D-Var solution. Notably, the third row's first column represents the 4D-Var solution at time $t=\num{801}$, constructed using observations from $t = \num{837}$. Similarly, the second column depicts the forecast for $t=\num{837}$ using $t = \num{801}$ as the initial conditions. The reconstructions provide reliable predictions across the spatial and temporal domain, even under the challenges of noise and ill-posedness. It is noteworthy that while the predictions appear smoothed in the southern region, the latent-twin effectively captures sharp features in the northern region. We specifically highlight two regions in the figure marked by a black box and a red box. The features inside the black box in the truth is similar across the two times. Whereas the feature in the red box evolves, with the low-pressure region shrinking in size at time $t = 837$. The latent-twin captures both these features accurately. Furthermore, we observe that the forecast errors of the Latent Twin increase gradually and almost linearly as the lead time grows (see Figure~\ref{fig:error_trend}). The error trend is close to linear, showing that the model maintains a stable and predictable degradation rather than exhibiting abrupt divergence. Importantly, even at long horizons of up to 10 days, the errors remain at a manageable scale, underscoring the consistency and robustness of the Latent Twin’s predictions under challenging forecasting conditions. This behavior highlights the framework’s ability to provide reliable estimates over extended time windows, where traditional approaches often struggle with compounding instabilities.

\begin{figure}[htbp]
    \centering
    \includegraphics[width=0.95\linewidth, trim={1mm 1mm 1mm 1mm},clip]{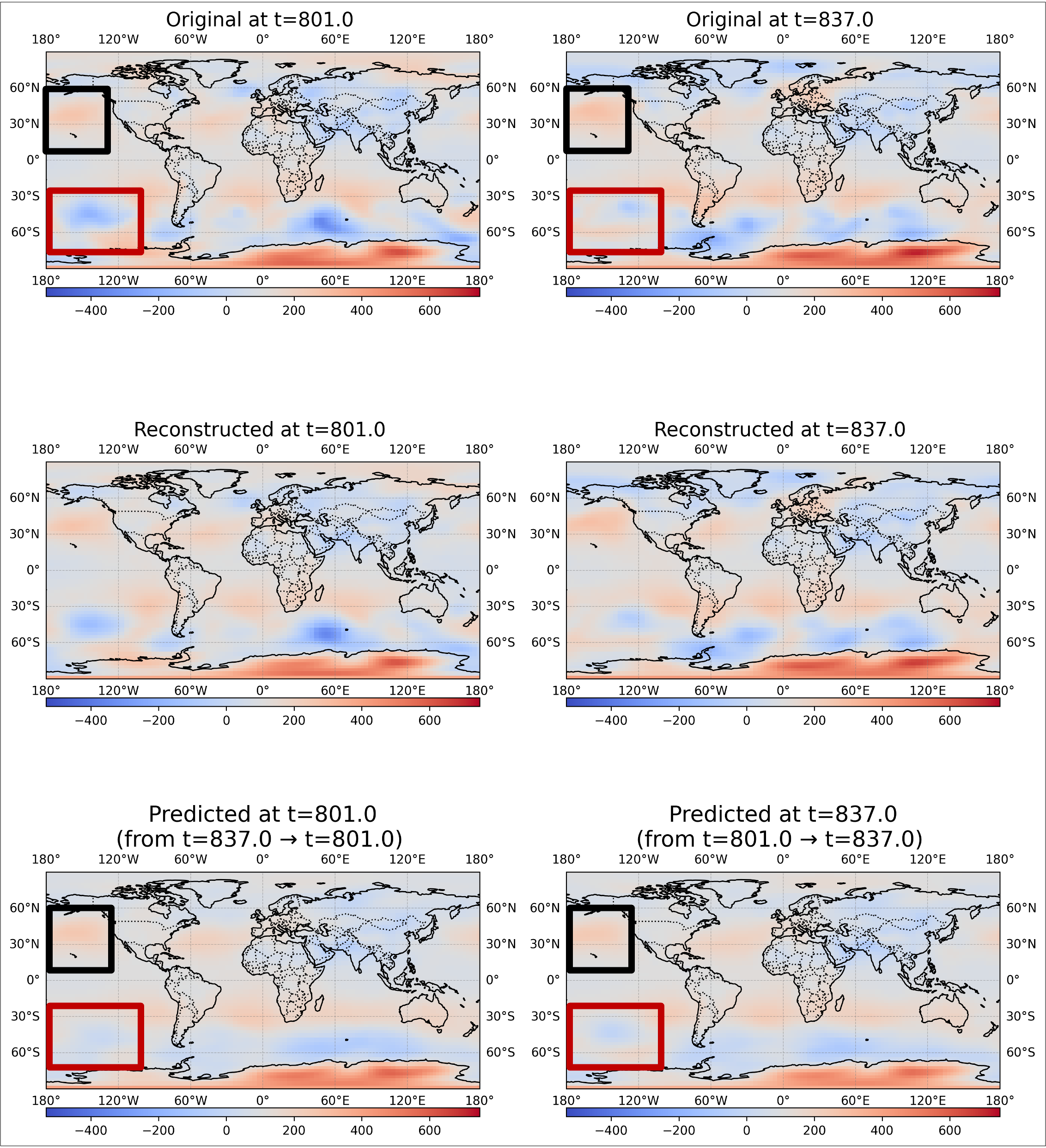}
    \caption{Reconstruction and Predictions with the testing data. The Reconstructions are compared to the truth at $t= 801$ and $t = 837$ with the truth. The predictions from $t=801$ to $t=837$ and vice-versa are compared to the truth. Note that the latent-twin successfully captures features that evolve with time (red box) and those that are relatively unchanged (black box). }
    \label{fig:geopotential_dynamics1}
\end{figure}

\begin{figure}[htbp]
    \centering
    \includegraphics[width=0.45\linewidth]{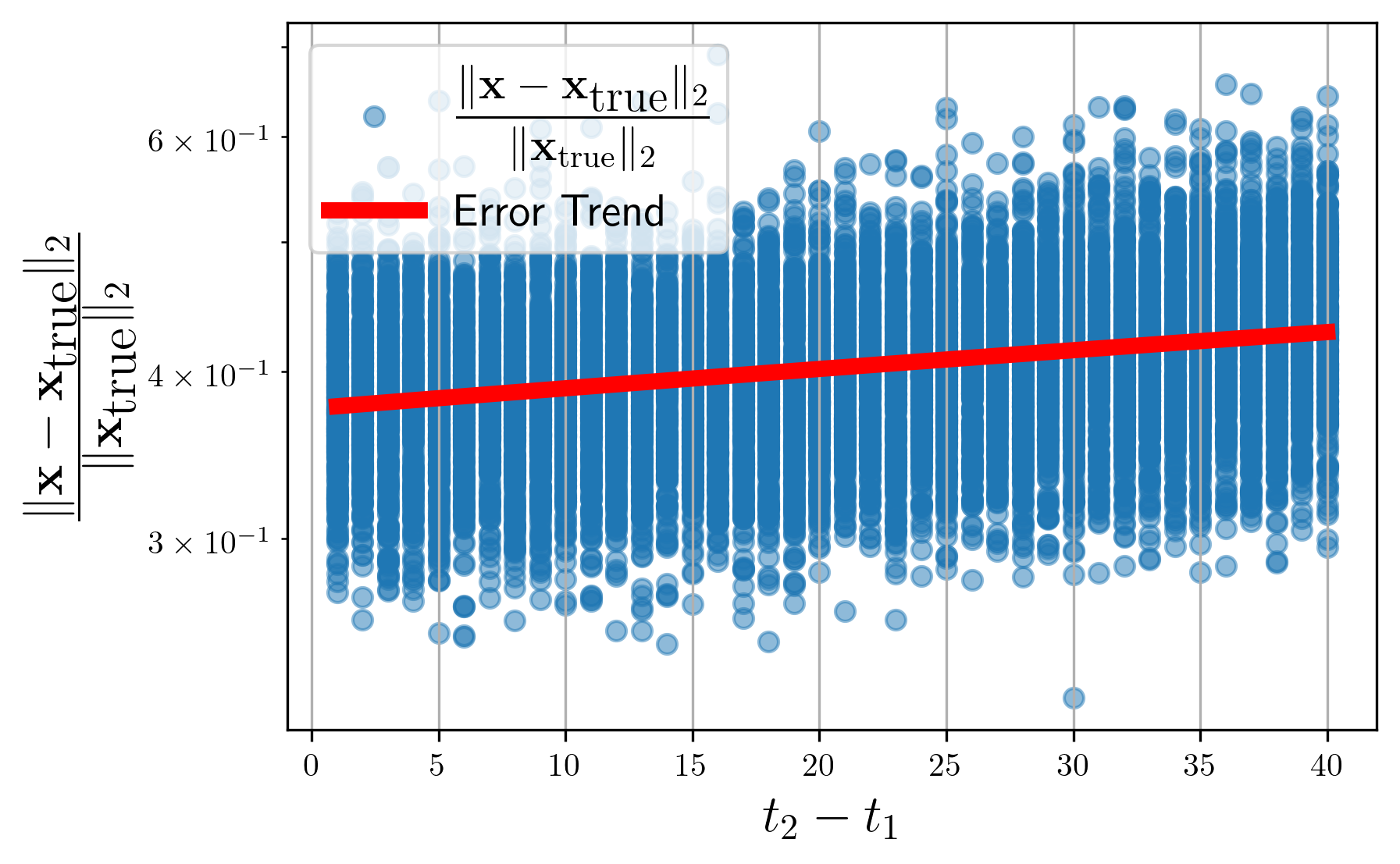}
    \caption{Forecast error distribution across increasing lead times. Errors grow gradually and almost linearly with horizon length, demonstrating that the Latent Twin provides consistent and stable predictions even up to 10 days ahead.}
    \label{fig:error_trend}
\end{figure}

\section{Conclusions \& Outlook} \label{sec:conclusion}
We introduced \emph{Latent Twins}, a representation-driven framework that learns solution operators in a task-adaptive latent space. Theoretically, we established uniform approximation guarantees for ODE/PDE flow-maps and separated spatial discretization from network approximation. Numerically, we showed that Latent Twins (i) capture canonical ODE dynamics in a single-shot across arbitrary time gaps, thus avoiding stepwise error accumulation; (ii) provide efficient PDE surrogates competitive with 4D-Var forecasts on the shallow-water equations; and (iii) reconstruct and forecast geopotential heights from sparse, noisy observations.

Latent twins open a broad spectrum of opportunities that reach far beyond the scope of this paper. One promising direction is to embed classical data assimilation techniques, such as 3D/4D-Var or ensemble filters, directly into the $(e,d,m)$ architecture, enabling fully differentiable pipelines in which background covariances and flow-dependent priors are learned alongside the latent dynamics. The latent map $m$ itself need not remain a generic neural block: it can be endowed with structure that reflects the underlying dynamics. Examples include encoding semigroup properties through exponential maps, enforcing symmetries or conservation laws, or ensuring contractivity for stability—design choices that could dramatically improve robustness, sample efficiency, and maintain consistency with physical models. Extending Latent Twins into the probabilistic realm offers another compelling avenue, by introducing distributions on $(e,d,m)$ or the latent states $z$, yielding Bayesian surrogates, posterior-predictive flow-maps, and score-based priors capable of reconstructing trajectories from sparse or noisy observations. These developments would connect naturally with the rapidly growing operator-learning literature, with Latent Twins serving as an integrative framework that complements existing approaches such as DeepONets and Fourier Neural Operators, and clarifies how representation, decoder expressivity, or time conditioning shape operator identifiability and performance. Latent twins could also be explored in control and inverse design: once $m$ is treated as a differentiable surrogate, it can be embedded into PDE-constrained optimization or closed-loop control, enabling safe and efficient decision-making informed by latent dynamics. Ultimately, advancing both scalable architectures and rigorous theory will be essential for elevating Latent Twins from proof-of-concept models to a foundational tool for large-scale, reliable scientific computing. Taken together, these avenues illustrate that Latent Twins are not merely another surrogate modeling tool, but a unifying paradigm that brings together representation learning and operator approximation with rigorous theory and numerical experimentation into a single coherent framework with wide-ranging potential. Just as Feynman described nature’s tapestry, Latent Twins offer a way to weave data and equations into a common fabric, where each latent thread reflects structure, representation, and dynamics, revealing the deeper organization shared across scientific computing.

\section{Acknowledgements}
This research used resources of the Argonne Leadership Computing Facility, a U.S. Department of Energy (DOE) Office of Science user facility at Argonne National Laboratory and is based on research supported by the U.S. DOE Office of Science-Advanced Scientific Computing Research Program, under Contract No. DE-AC02-06CH11357

\printbibliography

\begin{appendix}
\numberwithin{equation}{section}
\renewcommand{\theequation}{A.\arabic{equation}}

\section{Details on the LSTM comparison}\label{app:lstm}

We provide additional comparisons between LSTM baselines and ground-truth trajectories across all four benchmark systems: \textit{Harmonic Oscillator}, \textit{Lotka--Volterra}, \textit{SIR}, and \textit{Lorenz--63}. Each LSTM is trained in a fixed-window many-to-one configuration (here 10-to-1), mapping $\big(x(t_{k-9}), \dots, x(t_k)\big) \mapsto x(t_{k+1})$, using a history of 10 timesteps. No explicit time encoding is supplied; the network must infer dynamics purely from state transitions. Forecasts over long horizons are obtained by recursively applying the model, which inevitably compounds errors. To ensure a ``capacity-constrained'' baseline, the LSTM hidden dimensions are chosen such that the parameter counts match the low-complexity regime similar to that of the Latent Twin. Specifically, we obtain 582 parameters for all 2D systems and 855 parameters for 3D systems. These counts are still \emph{larger} than those of the corresponding Latent Twin architectures, which use only 246 (2D) and 267 (3D) parameters.

Training data is generated over the same time intervals as in the Latent Twin experiments. However, to expose the LSTM to sufficient variation within its 10-step input window, we increase the timestep $\delta t$ and generate \num{1000} samples. From the resulting trajectory, overlapping 10-step windows are extracted, with an 80/20 split between training and validation. Models are trained for \num{1000} epochs using the Adam optimizer with learning rate $10^{-3}$.

During inference, we perform a full-horizon rollout by recursively applying the trained LSTM, seeded with the first 10 ground-truth states. This setup allows us to assess both short-term accuracy and long-term error growth. The LSTMs achieve good short-term accuracy, particularly in oscillatory systems such as the harmonic oscillator and Lotka--Volterra. However, recursive prediction leads to phase drift, amplitude distortion, and error accumulation over longer horizons—effects that are especially pronounced in chaotic Lorenz--63 dynamics and in the sensitivity of the SIR model.

By contrast, the Latent Twin framework directly approximates the solution operator in a single forward pass. This design avoids recursive stepping and yields more robust long-term predictions with fewer parameters. These comparisons are intended as illustrative rather than competitive: LSTMs and Latent Twins embody fundamentally different modeling philosophies—sequential transition versus operator learning.

\begin{figure}[!htb]
    \centering
    \includegraphics[width=0.48\textwidth]{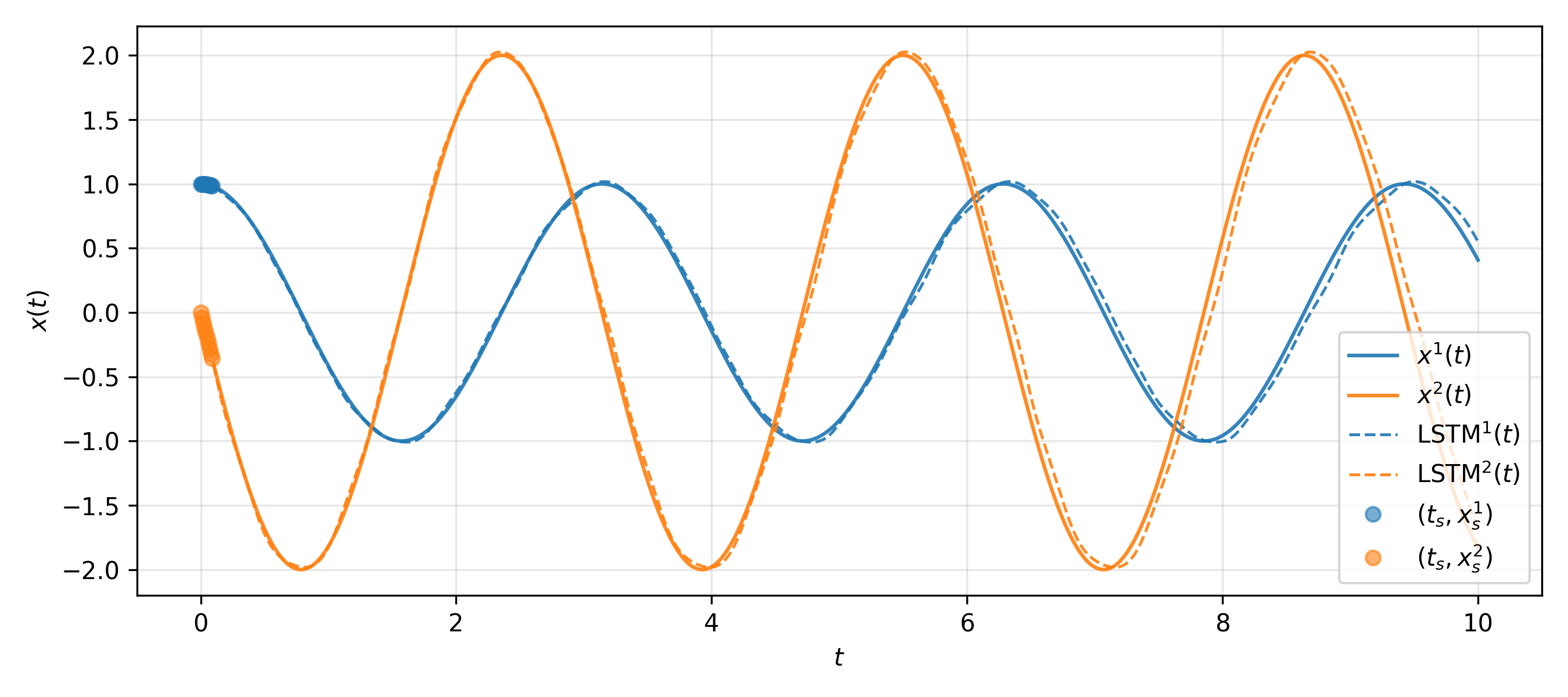}    
    \includegraphics[width=0.48\textwidth]{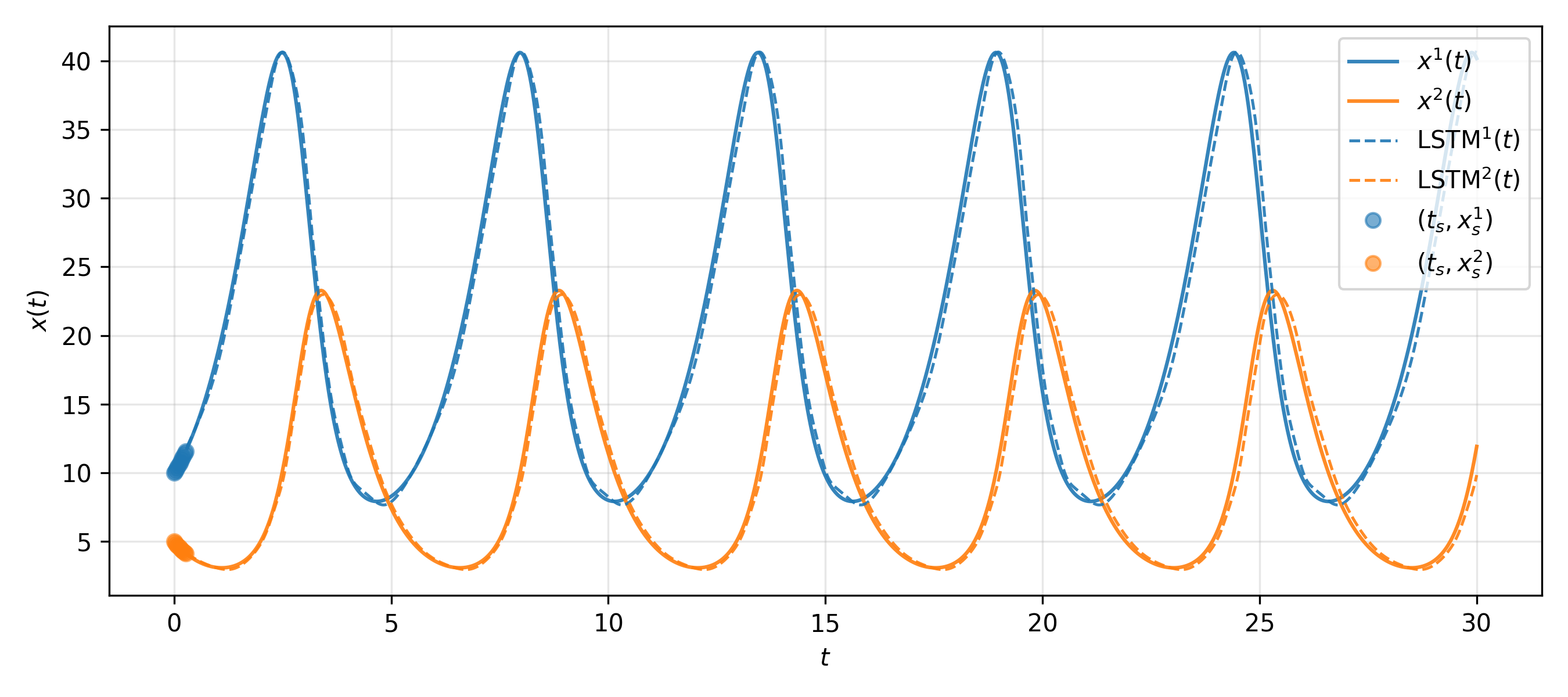}    
    \includegraphics[width=0.48\textwidth]{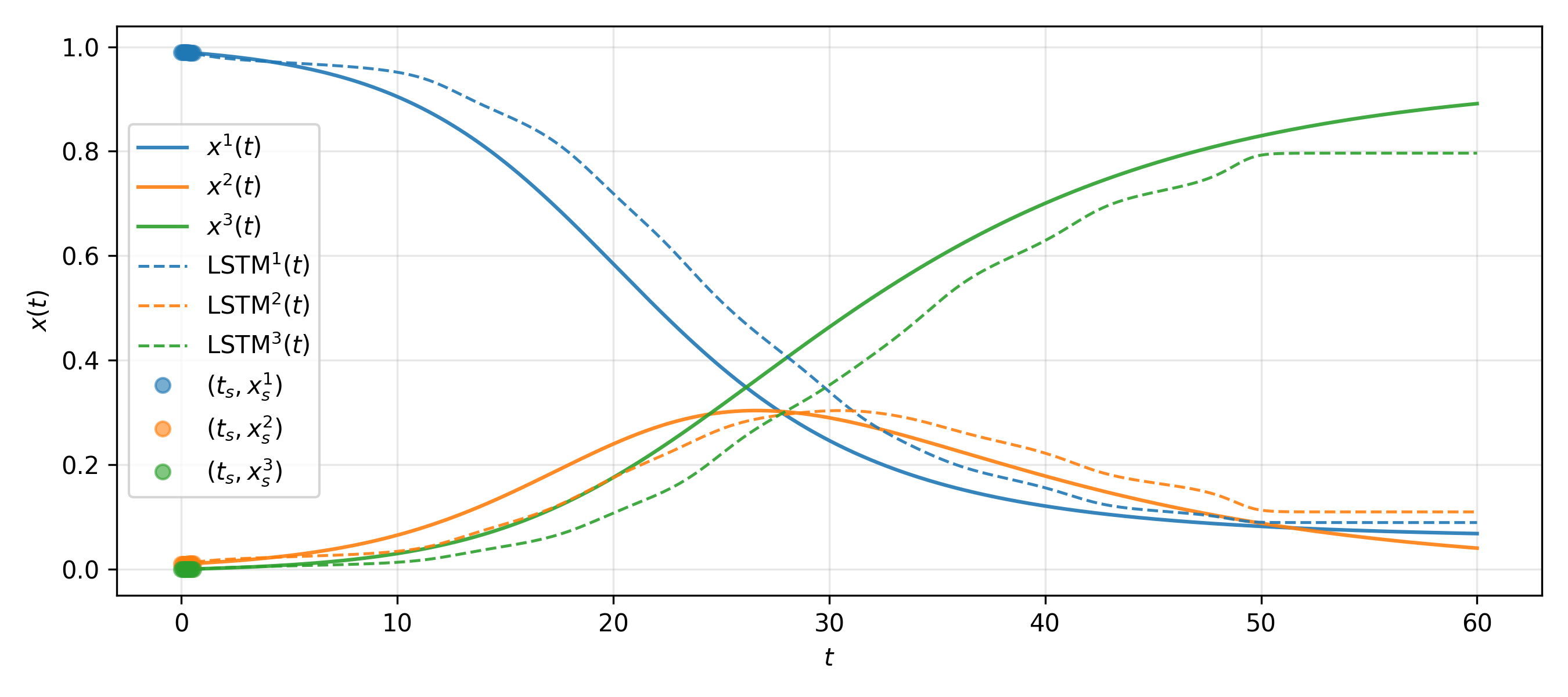}
    \includegraphics[width=0.48\textwidth]{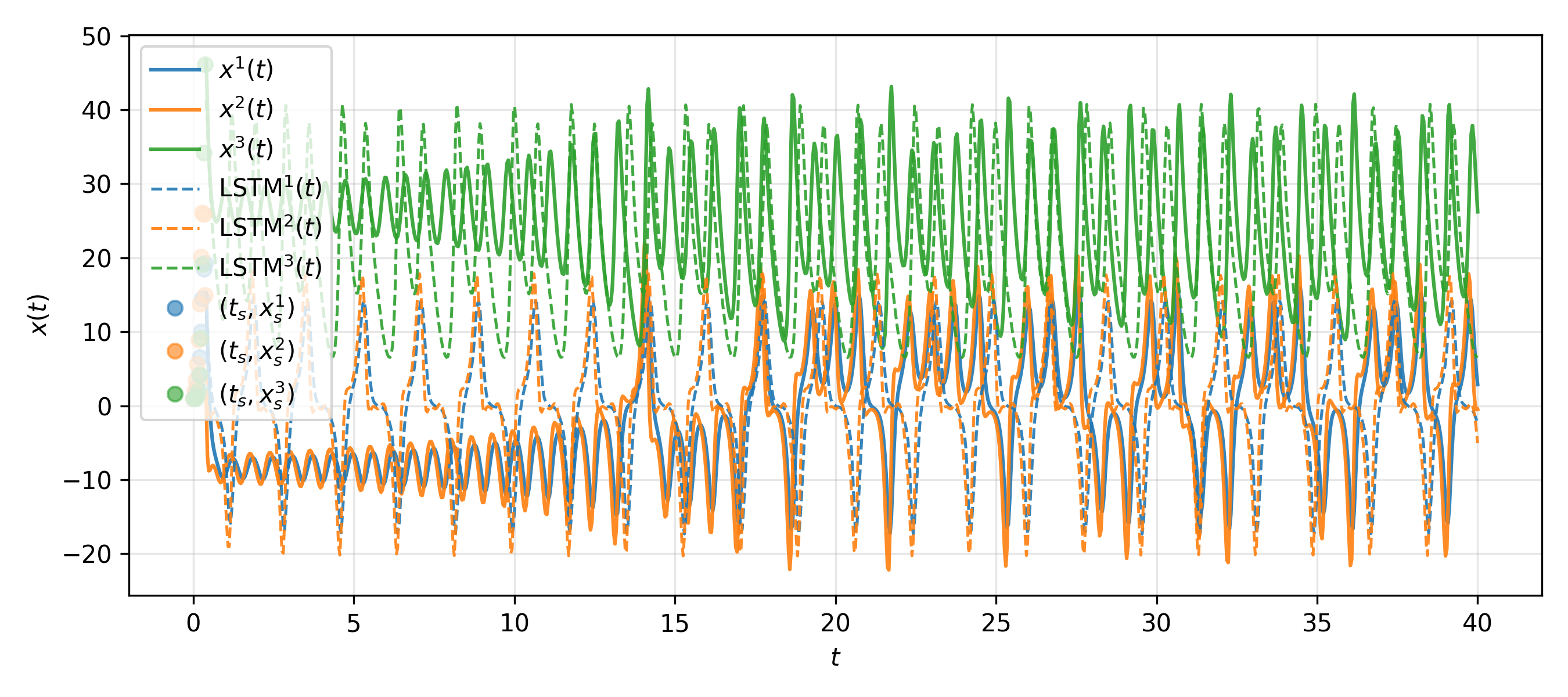}
    \caption{LSTM predictions (dashed) compared to ground truth trajectories (solid) for all systems under a 10-to-1 training regime. While short-term accuracy is reasonable under limited resources, recursive prediction leads to increasing phase and amplitude distortion over time.}
    \label{fig:lstm-dynamics}
\end{figure}

\begin{figure}[!htb]
    \centering
    \includegraphics[width=0.48\textwidth]{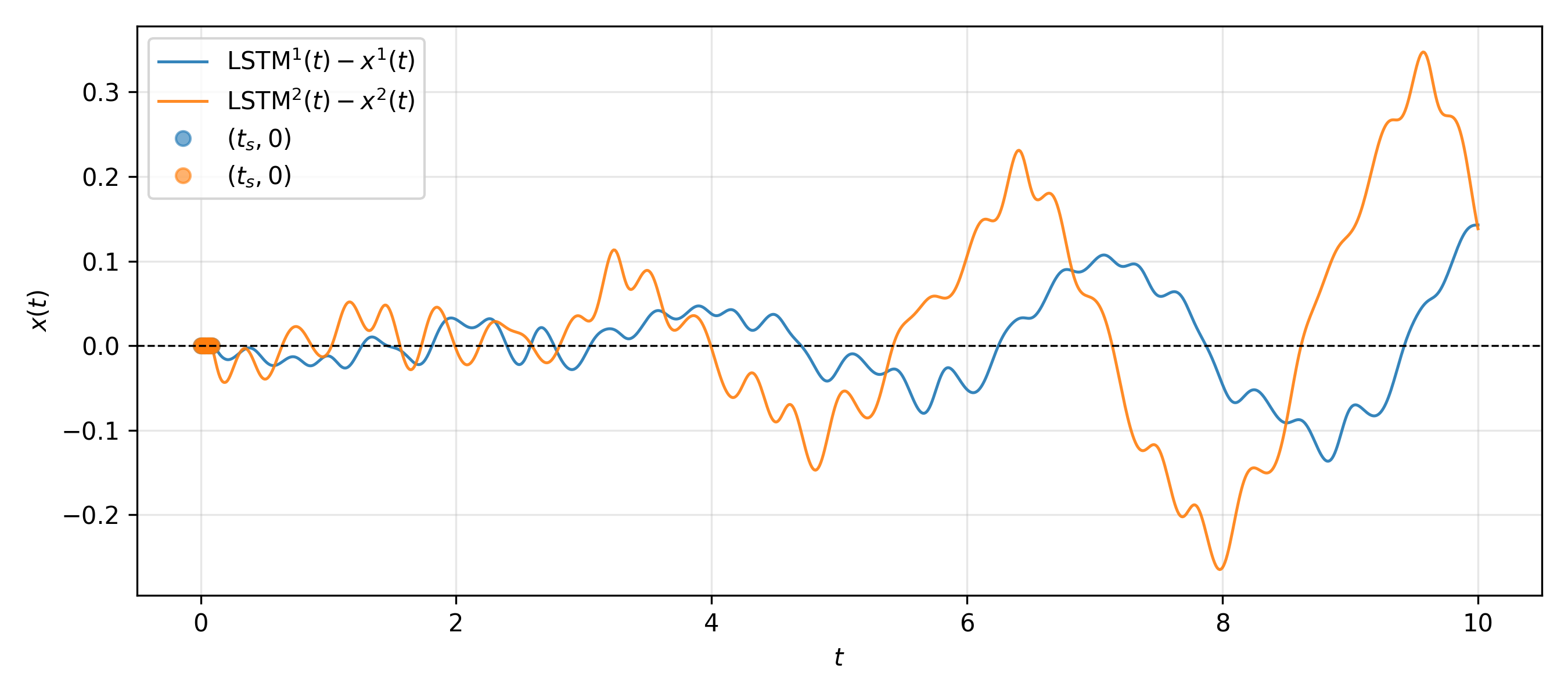}    
    \includegraphics[width=0.48\textwidth]{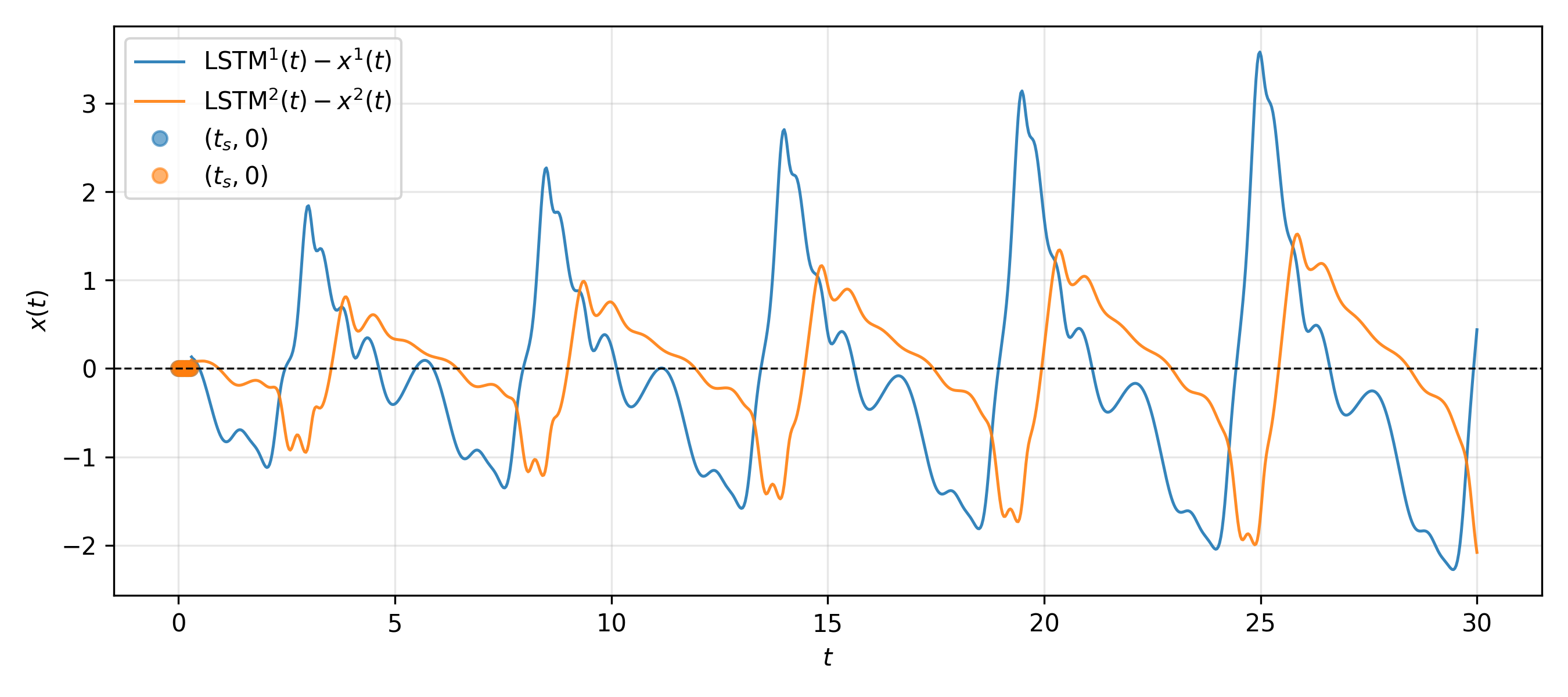}    
    \includegraphics[width=0.48\textwidth]{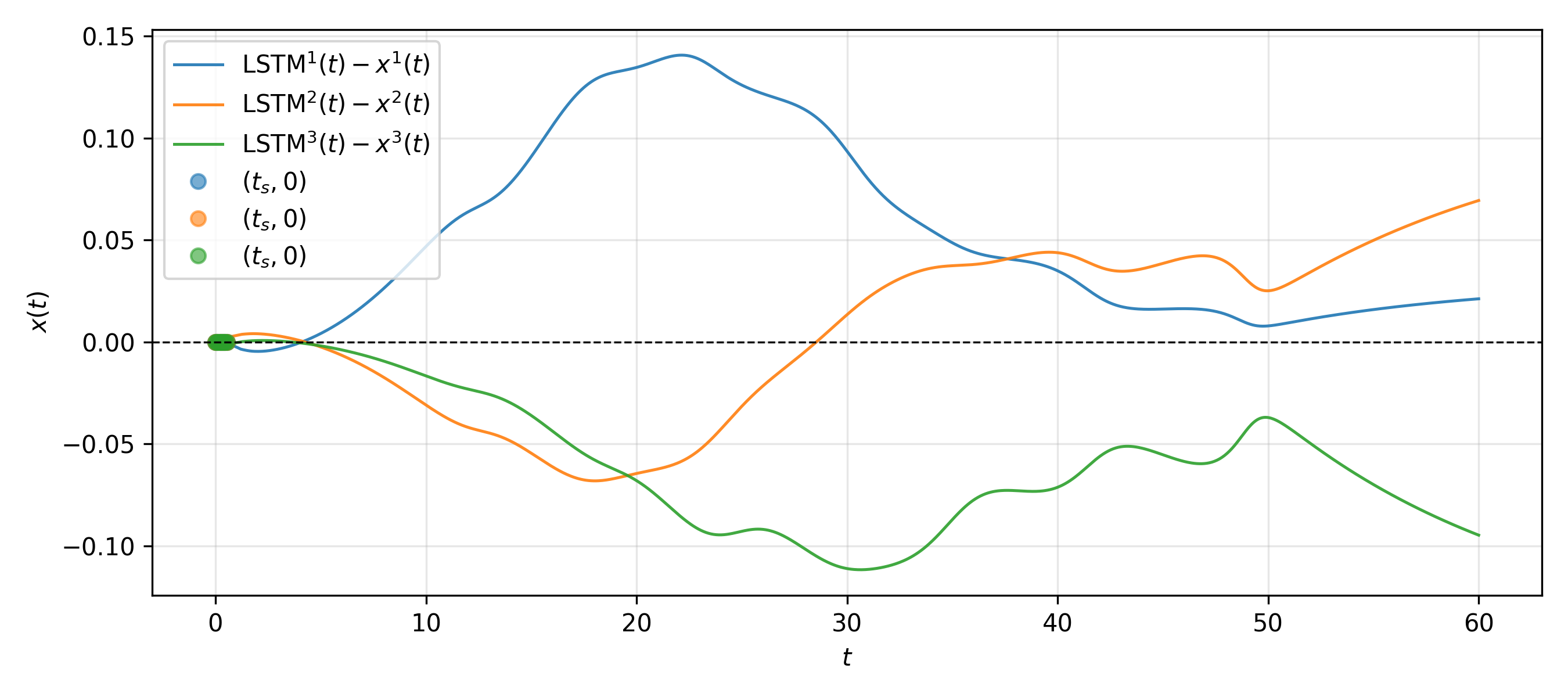}    
    \includegraphics[width=0.48\textwidth]{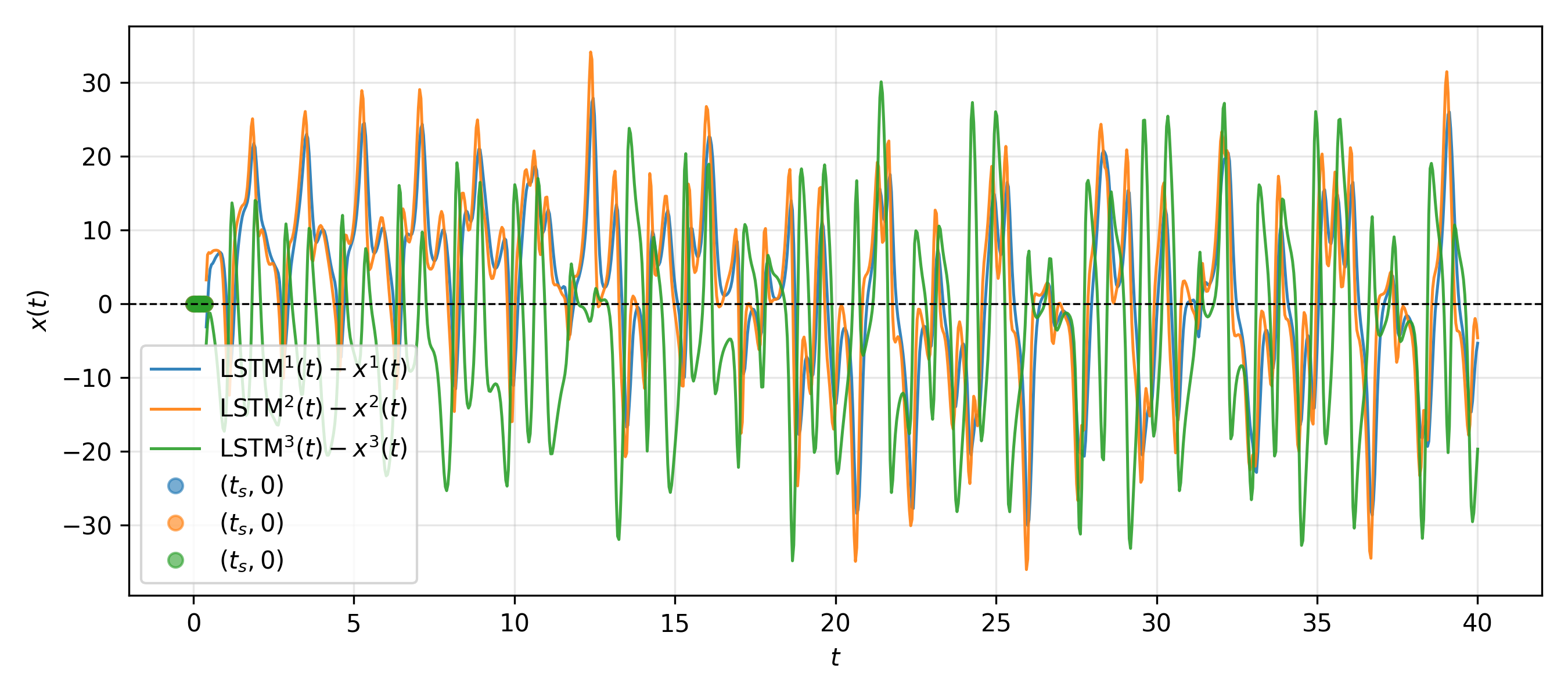}
    \caption{Component-wise prediction errors (LSTM vs.~ground truth) over time. All models are compact (582–855 parameters), but error accumulation increases with horizon.}
    \label{fig:lstm-errors}
\end{figure}    

\section{Details of the Shallow Water Experiments}\label{app:swe-details}
\numberwithin{equation}{section} 
\renewcommand{\theequation}{B.\arabic{equation}}

The shallow water equations (SWE) describe the depth-averaged dynamics of fluid motion and serve as a canonical testbed for geophysical flows. In our experiments we adopt the partially nonlinear formulation
\begin{equation}\label{eq:swe_initial}
\begin{aligned}
    \eta_t &+ \partial_{\chi_1}\big((\eta + H) u\big) + \partial_{\chi_2}\big((\eta + H) v\big) = 0, \\
    u_t &- f v = - g \,\partial_{\chi_1}\eta, \\
    v_t &+ f u = - g \,\partial_{\chi_2}\eta,
\end{aligned}
\end{equation}
where $\eta(\chi_1,\chi_2,t)$ denotes the free surface elevation, $u(\chi_1,\chi_2,t)$ and $v(\chi_1,\chi_2,t)$ are the horizontal velocities, and $h = \eta + H$ is the total column depth. The physical parameters are gravitational acceleration $g = 9.81$~m/s$^2$, mean depth $H = 100$~m, and Coriolis forcing under the $\beta$-plane approximation
\[
    f(\chi_2) = f_0 + \beta \chi_2,
    \quad f_0 = 10^{-4}\,\text{s}^{-1}, 
    \quad \beta = 2\times 10^{-11}\,\text{m}^{-1}\text{s}^{-1}.
\]

The spatial domain is chosen as $\Omega = [-5\!\times\!10^5, 5\!\times\!10^5]^2$ meters squared. All four boundaries are reflective, with the normal velocity set to zero and $\eta$ satisfying homogeneous Neumann conditions. These choices enforce impermeable walls and conserve mass within the domain. Initial conditions are constructed as randomized Gaussian perturbations of the free surface,
\[
    \eta(\chi_1,\chi_2,0) = \exp\!\Big( -\tfrac{(\chi_1-\mu_{\chi_1})^2 + (\chi_2-\mu_{\chi_2})^2}{2\sigma^2} \Big),
\]
with centers $(\mu_{\chi_1},\mu_{\chi_2})$ drawn uniformly from $\Omega$ and standard deviation $\sigma = 5\times 10^4$ m, while the initial velocity field is set to zero. This stochastic initialization produces a wide variety of wave patterns, ensuring that the dataset spans diverse dynamical behaviors.

Numerical integration is carried out using an open-source Python implementation~\cite{braendshoi2019shallow}. The solver employs a finite-difference scheme on a $64\times 64$ Cartesian grid, corresponding to spatial resolution $\Delta x \approx 1.6\times 10^4$ m. The time step $\Delta t \approx 51$~s satisfies the CFL condition, and trajectories are simulated up to $T \approx 3\times 10^4$~s (600 steps). From these trajectories, we construct a dataset by sampling random time pairs $t_1,t_2 \in [0,T]$ and recording the corresponding states
\[
    x(t) = [\eta(\cdot,t), u(\cdot,t), v(\cdot,t)] \in \R^{3\times 64\times 64}.
\]
This procedure yields $J=32{,}768$ paired samples, split into $29{,}491$ for training and $3{,}277$ for testing. All channels are standardized to zero mean and unit variance.

The Latent Twin model comprises a multilayer perceptron autoencoder and a simple latent mapping. The autoencoder takes flattened input vectors in $\R^{12{,}288}$, compresses them into a 128-dimensional latent space via fully connected layers $12{,}288 \to \num{1024} \to 256 \to 128$ with ReLU activations, and reconstructs through a mirrored decoder $128 \to 256 \to \num{1024} \to 12{,}288$. The latent mapping $m_\theta(z,t_1,t_2)$ is implemented as a single affine linear layer conditioned on $(t_1,t_2)$, mapping $\R^{128}\times\R^2 \to \R^{128}$. Training minimizes the empirical risk in \Cref{eq:empirical}, balancing reconstruction and temporal prediction losses. Optimization uses Adam with learning rate $10^{-3}$, batch size 64, and 1,000 epochs, with the rate halved every 250 epochs. The training loss is shown in \Cref{fig:swe_loss_appendix}, where sharper decreases coincide with scheduled learning-rate reductions.

Latent twin experiments were performed on a 2020 MacBook Pro (Apple M1 chip, 16 GB unified memory) using PyTorch with the Metal Performance Shaders (MPS) backend. Training was conducted in single precision. All runs used a fixed random seed (42) for data splitting and weight initialization. The autoencoder has approximately $ 2.63\times 10^7$ parameters and the latent map about $1.7\times 10^4$. The model trains efficiently on modest hardware: each epoch required on average $\sim 32$ seconds, so 1,000 epochs completed in under 9 hours. This demonstrates that the framework is computationally lightweight and does not rely on specialized high-performance resources.

\begin{figure}[ht]
    \centering
    \includegraphics[width=0.55\linewidth]{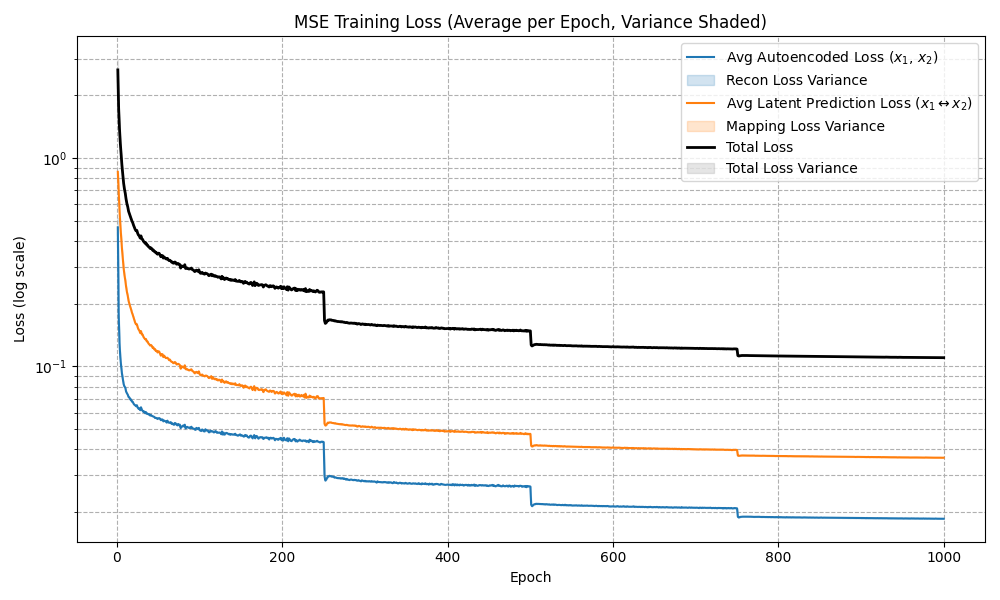}
    \caption{\emph{Training curves for the SWE Latent Twin:} empirical risk decreases steadily, with noticeable drops at each scheduled learning-rate reduction.}
    \label{fig:swe_loss_appendix}
\end{figure}

\begin{figure}
    \centering
  \includegraphics[width=1\linewidth,
    trim=0 375 0 0,   
    clip]{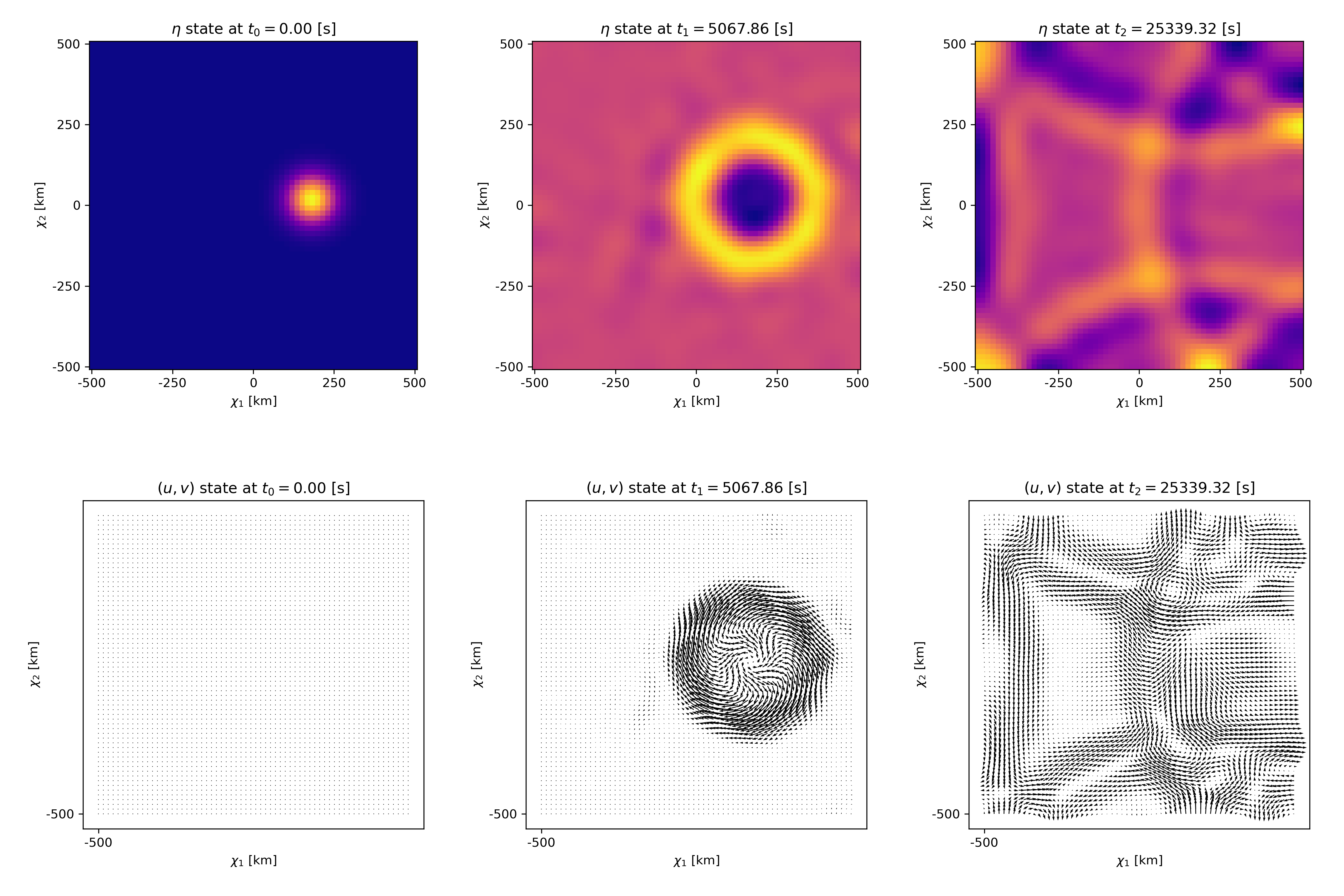}%
    \caption{\emph{Predictions of surface elevations $\eta$:} The first image shows the current state at $t_0$, while the middle image displays the Latent Twin predicted state at $t1$ while the third image displays Latent Twin predicted state at $t_2$. The forward predictions may be directly compared to the DeepONet predictions shown in \Cref{fig:deepONet}, where the Latent Twin exhibits superior visual accuracy.}
    \label{fig:swe_pred12}
\end{figure}

To emulate realistic observations, we apply spatial downsampling by factors of $2,4,8$ and add Gaussian noise $\epsilon \sim \mathcal{N}(0,0.01I)$, producing observations
\[
    y_{\text{obs}}(t_1) = P(x(t_1)) + \epsilon.
\]
A latent state $\hat z(t_1)$ consistent with $y_{\text{obs}}(t_1)$ is then recovered by solving
\[
    \hat z(t_1) = \arg\min_z \| P(d_\theta(z)) - y_{\text{obs}}(t_1) \|^2,
\]
initialized by encoding an interpolated version of the coarse observation. Once $\hat z(t_1)$ is inferred, the Latent Twin propagates it to arbitrary times $t_2$, enabling reconstruction and prediction from sparse and noisy observations.

\emph{4D-Var.} For completeness, we summarize the implementation of the strong-constraint four-dimensional variational (4D-Var) data assimilation method used in \Cref{sub:swe}. The goal is to estimate the state $\hat x_2$ that minimizes a quadratic cost functional balancing background and observation information, subject to the SWE dynamics. Specifically, the optimization problem reads
\begin{equation}\label{eq:swe-4dvar}
    \hat x_2
    = \argmin_{{x}_0 \in \R^n} 
      \tfrac{1}{2} (x_0 - x_0^b)^\top \Gamma^{-1} (x_0 - x_0^b)
      + \tfrac{1}{2} \sum_{i=1}^{n_t} 
        \big(P(\Phi(t_i,t_0,x_0)) -y^{i}_{\text{obs}}\big)^\top 
        \Omega_i^{-1}
        \big(P(\Phi(t_i,t_0,x_0)) - y^{i}_{\text{obs}} \big).
\end{equation}
Here $x_0^b$ is the background state, $\Gamma$ is the background error covariance, $\Omega_i$ are observation error covariances, $P$is the observation operator (chosen as the same downsampling $P$ used in the noisy-observation experiments), and $\Phi$ is the SWE solution operator. Observation errors were drawn from the same distribution as in the latent-twin experiments to ensure comparability.

The SWE were discretized using second-order centered finite differences in space and integrated in time with the third-order TVD Runge--Kutta (TVD-RK3) scheme of \cite{GottliebShu1998}. The discrete tangent–linear and adjoint models were constructed following the methodology of \cite{Sandu2006}, enabling gradient computation with respect to the initial condition.

The cost functional \Cref{eq:swe-4dvar} was minimized using L-BFGS, with termination once the relative gradient norm dropped below $10^{-3}$. The background covariance $\Gamma$ was chosen as a scaled inverse Laplacian, block-diagonal across $(\eta,u,v)$, to encode spatial correlation structure. 

This 4D-Var setup represents a standard baseline in geophysical data assimilation. It relies only on local model dynamics and a single assimilation window, in contrast to the Latent Twin which leverages a large offline dataset over the full horizon. The comparison in \Cref{fig:swe_4dvar} is therefore illustrative of complementary strengths: 4D-Var integrates physics and covariance priors in a principled optimization framework, while Latent Twins provide fast, global surrogates that remain effective for reconstruction and forecasting beyond the assimilation window. In practice, the two approaches can be combined: a latent-twin forecast can serve as a flow-dependent background for 4D-Var or ensemble filters, while DA provides principled uncertainty quantification for latent-space inferences.

\emph{DeepONet.} For the deep operator setup, we numerically solve the SWE in the same manner as for Latent Twin. We simulate \num{4096} trajectories of the SWE where the initial condition is sampled as in the Latent Twin experiment. We create a train and test dataset by selecting $95\%$ of the trajectories for training and reserving the rest for testing. From these trajectories, we create a dataset by pairing initial conditions together with the discretized evaluation of full solution. That is, our dataset consists of pairs $(x_0, x_k)$ where
\[
x_0 = x(0) = \big[\eta(\,\cdot\,,0),\, u(\,\cdot\,,0),\, v(\,\cdot\,,0)\big] 
   \;\in\; \mathbb{R}^{3\times 64\times 64},
\]
and the trajectory is
\[
x_k = \{x(t_k)\}_{k=0}^{600}
  = \big\{ \big(\eta(\,\cdot\,,t_k),\, u(\,\cdot\,,t_k),\, v(\,\cdot\,,t_k)\big) \big\}_{k=0}^{600} 
  \;\in\; \mathbb{R}^{601 \times 3 \times 64 \times 64},
\]
with $t_k = k\Delta t$.
Due to memory constraints, at training time, we randomly sample $x$ at 10,000 points to approximate the loss between the model and the full solution of the SWE subject to the sampled initial condition. Randomly sampling evaluation points has been proposed as a method of reducing memory requirements while achieving lower generalization error \cite{karumuri2024efficient}. To construct the DeepONet, we use a feed-forward neural network for both the branch and trunk networks. The \emph{branch network} accepts the input vector and compresses it into a 128-dimensional vector with layers $\num{12288}\to 512\to 512\to 256 \to 128$ using a ReLU activation function. The \emph{trunk network} accepts the 3-dimensional input $(\chi_1, \chi_2, t)$ and encodes it into a 128-dimensional vector via layers $3\to 64 \to 128 \to 128$ with ReLU nonlinearities.

In total, the model contains approximately $6\times 10^6$ parameters. We train the DeepONet for 500 epochs with batch sizes of 16. Optimization is performed using Adam with a learning rate of $10^{-3}$. Each batch takes roughly 75 seconds to compute, thus, training for 500 epochs takes roughly 9 hours.

\section{Table of Mathematical Notation}
\label{sec:notation}
\numberwithin{equation}{section} 
\renewcommand{\theequation}{C.\arabic{equation}}

These tables summarizes the key mathematical symbols and notation used throughout this paper.

\begin{center}
\small
\begin{longtable}{p{0.19\textwidth} p{0.42\textwidth} p{0.29\textwidth}}
\toprule
\textbf{Symbol} & \textbf{Description} & \textbf{Context} \\
\midrule
\endfirsthead

\multicolumn{3}{@{}l}{\small\textit{Notation — continued from previous page}}\\
\toprule
\textbf{Symbol} & \textbf{Description} & \textbf{Context} \\
\midrule
\endhead

\midrule
\multicolumn{3}{r}{\small\textit{Continued on next page}} \\
\endfoot

\bottomrule
\endlastfoot
\multicolumn{3}{l}{\textit{\textbf{Spaces, Sets, and Domains}}} \\
$X, Y$ & Normed vector spaces (input / output) & General framework \\
$Z, Z_x, Z_y$ & Latent spaces & Latent Twins \\
$\mathcal{X}$ & Function space for PDEs & PDE theory \\
$\mathcal{K}$ & Compact (forward/backward) invariant set in $\mathcal{X}$ & PDE well-posedness \\
$X_N := P_N(\mathcal{K})$ & Discrete state set after spatial discretization & PDE discretization \\
$[0,T],\, \mathcal{T}$ & Time interval / domain & Dynamics \\
$\mathcal{C}$ & Constraint set & Projected gradient descent \\[0.35em]

\multicolumn{3}{l}{\textit{\textbf{Core Latent Twin Components}}} \\
$F^{\to},\, F^{\gets}$ & True forward / inverse operators of interest & General definition \\
$f^{\to},\, f^{\gets}$ & Latent Twin surrogates of $F^{\to}, F^{\gets}$ & General definition \\
$e_x, e_y$ & Encoders $X\!\to\! Z_x$, $Y\!\to\! Z_y$ & Paired spaces \\
$d_x, d_y$ & Decoders $Z_x\!\to\! X$, $Z_y\!\to\! Y$ & Reconstruction \\
$e, d$ & Single-space encoder/decoder $X\!\leftrightarrow\! Z$ & ODE/PDE dynamics \\
$m^{\to}, m^{\gets}$ & Latent maps $Z_x\!\to\! Z_y$, $Z_y\!\to\! Z_x$ & Paired operators \\
$m(z,t_1,t_2)$ & Latent evolution map between times & Dynamics \\
$m^\ast(z,t_1,t_2)$ & Exact latent flow & Theory (Assumption 3) \\
$\Lambda_\theta = d \circ e$ & Autoencoder (parameters $\theta$) & Reconstruction \\
$\LT[t_1][x_1](t_2)$ & Latent Twin evaluation $(d\!\circ\! m(\cdot,t_1,t_2)\!\circ\! e)(x_1)$ & Dynamics \\
$\LTN$ & Discrete Latent Twin on $\mathbb{R}^{N}$ & PDE discretization \\
$\tildeLTN$ & Latent Twin lifted back to $\mathcal{X}$ via $R_N$ & PDE discretization \\
$\Phi(t_2,t_1,x_1)$ & True flow (solution) operator & Dynamics \\
$\Pi_{\mathcal{C}}$ & Projection operator onto constraint set & Optimization \\[0.35em]

\multicolumn{3}{l}{\textit{\textbf{Time and States}}} \\
$x(t),\, x_1, x_2$ & States at $t, t_1, t_2$ & Dynamics \\
$\tilde{x}, \tilde{z}_x, \tilde{z}_y$ & Reconstructed/predicted states & Figures/experiments \\
$z(t),\, z_1, z_2$ & Latent states & Encoded dynamics \\
$t_1, t_2$ & Start/target times & Dynamics \\
$n_x, n_z, N$ & State dim., latent dim., spatial d.o.f. & Dimensions \\[0.35em]

\multicolumn{3}{l}{\textit{\textbf{Training and Parameters}}} \\
$\theta$ & Network parameters of $(e,d,m)$ & Learning \\
$\mathcal{L}_{\text{loss}}$ & Loss function (reconstruction / prediction) & Risk in \Cref{eq:bayesrisk}, \Cref{eq:empirical} \\
$\mathbb{E}$ & Expectation (Bayes risk) & Training objective \\
$J$ & Number of training pairs & Dataset size \\[0.35em]

\multicolumn{3}{l}{\textit{\textbf{Errors and Constants}}} \\
$\varepsilon_{\mathrm{AE}}$ & Autoencoder reconstruction error & Theory \\
$\varepsilon_{\mathrm{map}}$ & Latent map approximation error & Theory \\
$\varepsilon_{\text{disc}}(N)$ & Spatial discretization error & PDE theory \\
$L_e,\, L_d$ & Lipschitz constants of $e$ and $d$ & Theory \\
$L_G$ & Flow Lipschitz constant (ODE) & Theory \\
$L_{\mathcal{K}}$ & Flow Lipschitz constant on $\mathcal{K}$ (PDE) & Theory \\
$C$ & Bound/stability constant & PDE assumptions \\
$a(t), b$ & Growth bound functions/constants & Well-posedness \\
$\operatorname{id}$ & Identity map & Utilities \\[0.35em]



\multicolumn{3}{l}{\textit{\textbf{ODE Systems}}} \\
$G(x,t)$ & Vector field in $x'(t)=G(x(t),t)$ & ODE definition \\
$\Phi(t_2,t_1,x_1)$ & Flow: solution at $t_2$ given $(t_1,x_1)$ & ODE theory \\[0.35em]

\multicolumn{3}{l}{\textit{\textbf{Benchmark ODE Parameters}}} \\
$\omega_0$ & Frequency parameter & Harmonic oscillator \\
$S, I, R$ & Susceptible, Infected, Recovered populations & SIR model \\
$\alpha, \beta, \gamma, \delta$ & Rate parameters & Lotka-Volterra \\
$\sigma, \rho, \beta$ & Parameters & Lorenz-63 system \\[0.35em]

\multicolumn{3}{l}{\textit{\textbf{Linear ODE Reduction}}} \\
$M$ & System matrix in $x' = Mx$ & Linear dynamics \\
$U \in \mathbb{R}^{n\times r}$ & Orthonormal basis ($U^\top U = I_r$) & Projection / POD \\
$W \in \mathbb{R}^{r\times r}$ & Learned latent generator & Structured $m$ \\
$\exp\!\big((t_2{-}t_1)W\big)$ & Latent semigroup used in $m$ & Physics-informed map \\[0.35em]

\multicolumn{3}{l}{\textit{\textbf{Koopman Operators}}} \\
$\mathscr{K}$ & Koopman operator & Nonlinear dynamics \\
$\mathfrak{K}$ & Linear latent evolution operator & Koopman autoencoders \\
$\Phi(x_t)$ & Discrete-time dynamical system map & Koopman theory \\
$g: \mathbb{R}^n \to \mathbb{C}$ & Observable function & Koopman analysis \\[0.35em]

\multicolumn{3}{l}{\textit{\textbf{PDE Theory and Discretization}}} \\
$u,\, u_0$ & PDE solution and initial state & Function space \\
$\mathcal{L}_{\text{PDE}}(u,t)$ & Abstract PDE operator & Evolution eq. \\
$P_N,\, R_N$ & Discretization / reconstruction ops. & Finite elements / FD \\
$\Phi_N$ & Discrete flow on $\mathbb{R}^{N}$ & Semi-discrete dynamics \\
$\LTN,\, \tildeLTN$ & Discrete and lifted Latent Twins & PDE surrogate \\
$\Delta t, \Delta x$ & Time and spatial discretization steps & Numerical methods \\[0.35em]

\multicolumn{3}{l}{\textit{\textbf{Observations \& Data Assimilation}}} \\
$P$ & Observation/downsampling operator & SWE \& DA experiments \\
$y_{\text{obs}}$ & Observed data $P(x)+\epsilon$ & Noisy/coarse obs. \\
$\epsilon$ & Observation noise & Experiments \\
$x_0^b$ & Background (prior) state & 4D-Var \\
$\Gamma$ & Background covariance & 4D-Var \\
$\Omega_i$ & Observation covariance at time $t_i$ & 4D-Var \\[0.35em]

\multicolumn{3}{l}{\textit{\textbf{Shallow Water Equations (SWE)}}} \\
$\eta(\chi_1,\chi_2,t)$ & Free surface elevation & SWE state \\
$u(\chi_1,\chi_2,t),\, v(\chi_1,\chi_2,t)$ & Horizontal velocities & SWE state \\
$H,\, h=\eta+H$ & Mean depth / total depth & Parameters / state \\
$g$ & Gravitational acceleration & Parameter \\
$f(\chi_2)=f_0+\beta \chi_2$ & Coriolis parameter ($\beta$-plane) & Forcing \\
$f_0, \beta$ & Coriolis constants & SWE parameters \\
$\chi_1,\chi_2$ & Spatial coordinates & Domain \\
$\mu_{\chi_1}, \mu_{\chi_2}, \sigma$ & Gaussian perturbation parameters & Initial conditions \\[0.35em]

\multicolumn{3}{l}{\textit{\textbf{Mathematical Utilities}}} \\
$\exp(hW), \mathrm{e}^{(\cdot)}$ & Matrix exponential & Linear flows \\
$\|\cdot\|,\ \|\cdot\|_{\mathrm{F}}$ & Vector / Frobenius norms & Errors \& matrices \\
$\argmin$ & Argument of the minimum & Optimization \\
$C(K,\mathbb{R}^m)$ & Continuous functions on compact $K$ & Approximation classes \\
$\mathcal{E},\,\mathcal{D},\,\mathcal{M}$ & Function classes for $e,d,m$ & UAT assumptions \\
$\mu$ & Parameter in Poisson equation & Example PDE 
\\
\end{longtable}
\end{center}

\end{appendix}

\end{document}